\documentclass[conference]{IEEEtran}
\IEEEoverridecommandlockouts

\usepackage{cite}
\usepackage{amsmath,amssymb,amsfonts}
\usepackage{textcomp}
\usepackage{xcolor}
\def\BibTeX{{\rm B\kern-.05em{\sc i\kern-.025em b}\kern-.08em
    T\kern-.1667em\lower.7ex\hbox{E}\kern-.125emX}}

\usepackage{subcaption}
\usepackage{url}
\usepackage[hidelinks]{hyperref}
\usepackage{booktabs}
\usepackage{placeins}

\usepackage{times}
\usepackage{soul}
\usepackage{url}
\usepackage[utf8]{inputenc}
\usepackage{graphicx}
\usepackage{amsmath}
\usepackage{amsthm}
\usepackage{booktabs}
\usepackage{makecell}
\usepackage{algorithm}
\usepackage[switch]{lineno}
\usepackage{mathtools}

\usepackage{latexsym}

\newtheorem{theorem}{Theorem}
\newtheorem{lemma}{Lemma}

\newtheorem{definition}{Definition}

\usepackage{amssymb}
\usepackage{tabularx}
\usepackage{multirow}
\usepackage{subcaption}
\usepackage{nicefrac}
\usepackage{enumitem}
\usepackage{algpseudocode}
\usepackage{xcolor}
\usepackage[most]{tcolorbox}
\usepackage{bbm}

\usepackage{pifont}

\newcommand{\xmark}{\ding{55}}

\newcommand{\NN}{\mathbb{N}}
\newcommand{\RR}{\mathbb{R}}

\newcommand{\EE}{\mathbb{E}}
\newcommand{\VV}{\mathbb{V}}
\newcommand{\1}{\mathbbm{1}}
\newcommand{\norm}[1]{\left\|#1 \right\|}

\newcommand{\ndata}{ n }
\newcommand{\nsamples}{ m_n }
\newcommand{\datasample}{V}

\begin{document}

\makeatletter

\newcommand{\linebreakand}{%
  \end{@IEEEauthorhalign}
  \hfill\mbox{}\par
  \mbox{}\hfill\begin{@IEEEauthorhalign}
}

\title{Probabilistic Label Spreading:\\ Efficient and Consistent Estimation of Soft Labels with Epistemic Uncertainty on Graphs
}

\author{
  \IEEEauthorblockN{1\textsuperscript{st} Jonathan Klees}
  \IEEEauthorblockA{\textit{Institute of Computer Science} \\
    \textit{Osnabrück University}\\
    jonathan.klees@uni-osnabrueck.de}
  \and
  \IEEEauthorblockN{2\textsuperscript{nd} Tobias Riedlinger}
  \IEEEauthorblockA{\textit{Department of Mathematics} \\
  \textit{Technical University of Berlin} \\
    riedlinger@tu-berlin.de}
  \and 
  \IEEEauthorblockN{3\textsuperscript{rd} Peter Stehr}
  \IEEEauthorblockA{\textit{Department of Mathematics} \\
  \textit{University of Wuppertal} \\
    stehr@uni-wuppertal.de}
\linebreakand 
  \IEEEauthorblockN{4\textsuperscript{th} Bennet Böddecker}
  \IEEEauthorblockA{\textit{Department of Mathematics} \\
    \textit{University of Wuppertal}\\
    bennet.boeddecker@uni-wuppertal.de}
    \and 
  \IEEEauthorblockN{5\textsuperscript{th} Daniel Kondermann}
  \IEEEauthorblockA{\textit{Quality Match GmbH} \\
    Heidelberg, Germany \\
    dk@quality-match.com}
  \and
  \IEEEauthorblockN{6\textsuperscript{th} Matthias Rottmann}
  \IEEEauthorblockA{\textit{Institute of Computer Science} \\
    \textit{Osnabrück University}\\
    matthias.rottmann@uni-osnabrueck.de}
}

\maketitle

\setlength{\parindent}{0pt}
\setlength{\parskip}{0.5em}

\begin{abstract}
Safe artificial intelligence for perception tasks remains a major challenge, partly due to the lack of data with high-quality labels. Annotations themselves are subject to aleatoric and epistemic uncertainty, which is typically ignored during annotation and evaluation. While crowdsourcing enables collecting multiple annotations per image to estimate these uncertainties, this approach is impractical at scale due to the required annotation effort. We introduce a probabilistic label spreading method that provides reliable estimates of aleatoric and epistemic uncertainty of labels. Assuming label smoothness over the feature space, we propagate single annotations using a graph-based diffusion method. We prove that label spreading yields consistent probability estimators even when the number of annotations per data point converges to zero. We present and analyze a scalable implementation of our method. Experimental results indicate that, compared to baselines, our approach substantially reduces the annotation budget required to achieve a desired label quality on common image datasets and achieves a new state of the art on the Data-Centric Image Classification benchmark.
\end{abstract}

\section{Introduction}

Accurate and reliable labels are the foundation of well-performing machine learning models and of particular importance in real-world applications. 
With the new paradigm of data-centric AI, the quality and reliability of data have come into focus as key ingredients for accurate and robust models~\cite{data_centric_AI}. 
Meanwhile, generating large-scale image datasets with high-quality labels is tedious and highly error-prone due to human factors. 
In fact, crowd-sourced labels are often imperfect, and their quality can even fall short of the zero-shot performance of modern large language models~\cite{crowdsourcing_vs_gpt4}. 
Even datasets like MNIST~\cite{lecun1998Gradientbased} contain samples with ambiguous class affiliations despite deterministic label assignment~\cite{ambiguousMNIST}. 
The impact of erroneous annotations can be mitigated by collecting multiple annotations per image, but this comes at substantial annotation effort and is therefore often infeasible for large-scale datasets.
This work is guided by the following question: \emph{How can we achieve reliable soft labels, including estimates of their epistemic uncertainty, based on as few human annotations as possible?}

\begin{figure}[t]
    \centering
    \begin{subfigure}[t]{0.48\linewidth}
        \includegraphics[trim=15 15 15 15, clip, width=\textwidth]{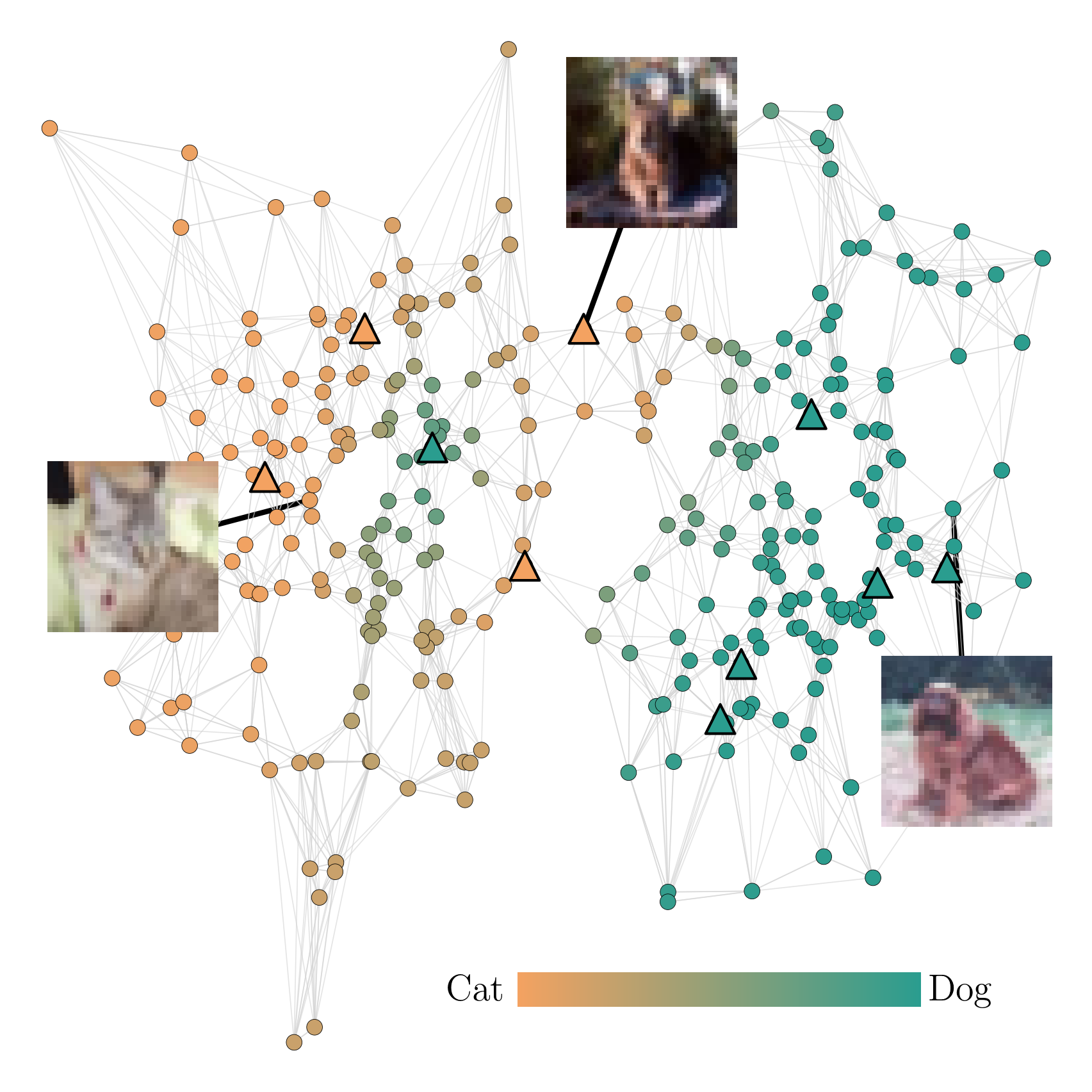}
        \caption{\centering Estimated soft labels}
    \end{subfigure}
    \begin{subfigure}[t]{0.48\linewidth}
        \includegraphics[trim=15 15 15 15, clip, width=\textwidth]{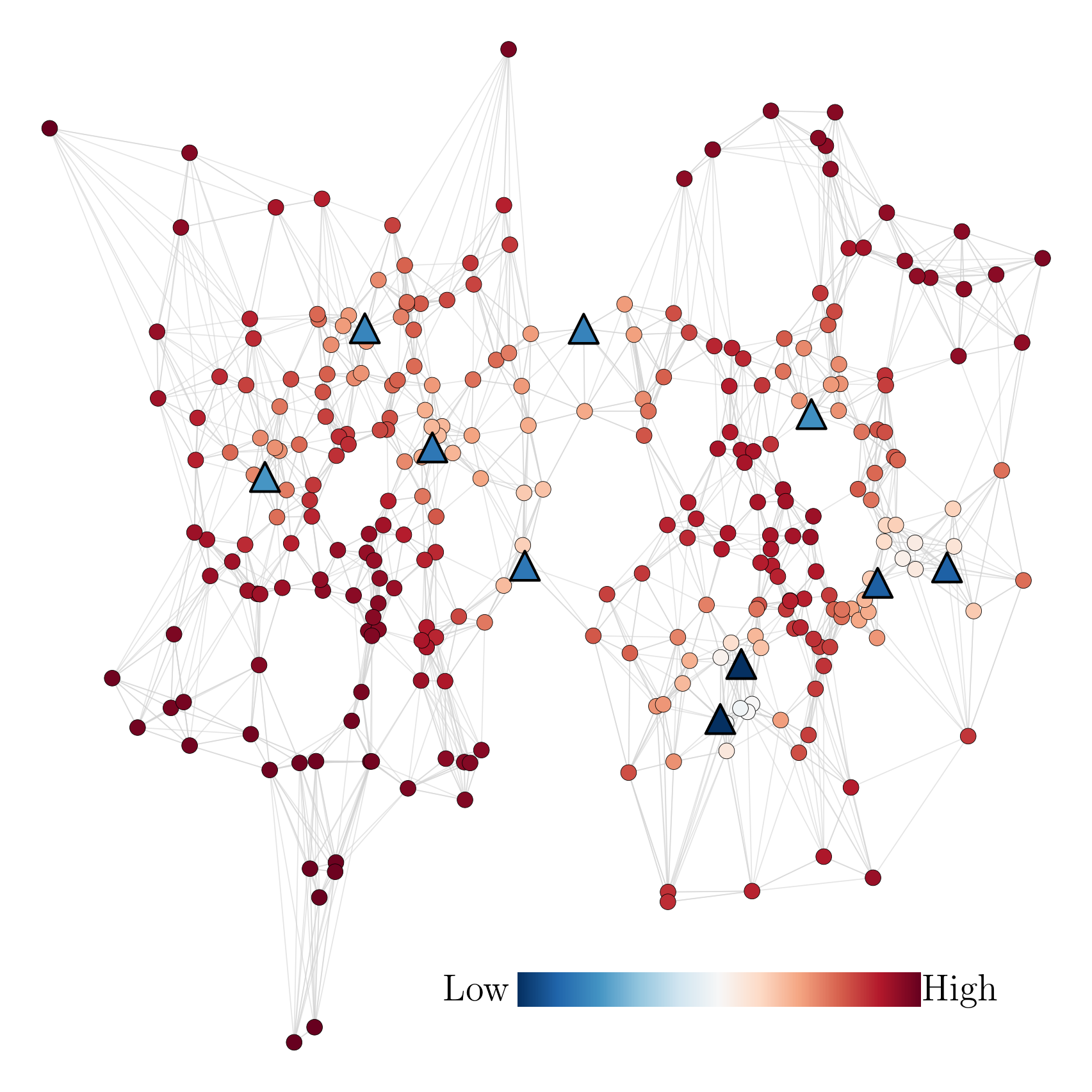}
        \caption{\centering Epistemic uncertainty}
    \end{subfigure}
    \caption{
    Probabilistic label spreading estimates soft labels based on few noisy annotations through graph-based propagation\,(a). It also estimates epistemic uncertainty of soft labels\,(b).
    }
    \label{fig:teaser}
\end{figure}

Methods of semi-supervised learning (SSL)~\cite{introduction_to_SSL_2009} are designed for settings with a mixture of labeled and unlabeled data and can be seen as supervised learning with additional information on the feature distribution~\cite{semi-supervised-learning}.
SSL can be divided into two main tasks: \emph{inductive learning}, which learns a model to predict labels for unseen test data, and \emph{transductive learning}, which predicts labels for unlabeled points to which the learning algorithm has access~\cite{introduction_to_SSL_2009,semi-supervised-learning}.
Of particular relevance to our work are crowd-sourced classification datasets, where multiple annotations are collected per data point and typically aggregated into a single clear-cut label.
Soft labels, in contrast, provide an estimate of the discrete conditional class distribution, capturing aleatoric uncertainty and containing more information than clear-cut labels.
Moreover, the number of individual annotator feedbacks per data point can be used to estimate the epistemic uncertainty of soft labels.
Based on labels (clear-cut or soft) on a small subset of data, transductive SSL techniques may be employed to generate labels for the remaining unlabeled data. Graph-based SSL methods have proven to work well even with very few annotations~\cite{Zhou_2003}.
However, widely used algorithms in this field only support clear-cut labels, neglecting valuable information in crowd-sourced datasets. 
We present a scalable, transductive algorithm to propagate soft label information (including aleatoric and epistemic uncertainty), which can be applied to reduce the number of annotations needed to infer reliable soft labels in a semi-automated labeling framework. 

We build on the graph-based \emph{label spreading algorithm}~\cite{Zhou_2003}, updating estimates incrementally as new annotations are obtained.
On a sparse neighborhood graph, a scalable solver efficiently solves the linear system involved. 
To satisfy the smoothness assumption\footnote{Two features which are close to one another in a high-density region have close corresponding targets~\cite{semi-supervised-learning}.} and thereby justify label spreading on the graph, we generate a semantic embedding using state-of-the-art vision encoders.
This way, our method is both powerful and flexible enough to be applied to large image datasets.
Moreover, our method provides explainability and the possibility to construct epistemic uncertainty estimates, i.e., confidence intervals.
We provide a mathematical analysis of estimation consistency showing that under the smoothness assumption and in the large data limit, label spreading yields soft labels that are probably approximately correct even if the number of annotations per data point converges to zero.
Apart from explainability and uncertainty quantification, the flexibility to steer the spreading intensity as well as the sequential update of the estimate with each annotation are advantages over other methods such as pseudo-labeling by a neural network.
We evaluate our method on several image classification benchmarks. While some of these datasets provide soft labels, we simulate a soft label ground truth on the others. Across these datasets, we observe superior performance of our method, compared to baseline algorithms. In particular, we achieve high-quality soft labels with very few noisy annotations on datasets such as CIFAR-10, Animals-10, EMNIST-digits.
Our main contributions can be summarized as follows:
\begin{itemize}[left=0pt]
    \setlength\itemsep{0em}
    \item
     We propose an efficient and scalable probabilistic label spreading algorithm that accounts for aleatoric and epistemic uncertainty and supports incremental updates.
     \item We provide mathematical guarantees on the consistency of our method and derive sample complexity bounds in the sense of probably approximately correct (PAC) learning.
     \item We perform an in-depth hyperparameter and ablation study with particular emphasis on the trade-off between spreading bias and epistemic uncertainty.
    \item 
     We show that our method achieves superior performance on the data-centric image classification benchmark for noisy and ambiguous label estimation~\cite{schmarje2022benchmark}.
    
\end{itemize}


\section{Related Work}
\textbf{Robust Supervised Learning.}
Several approaches adapt supervised learning techniques to handle crowd-sourced noisy annotations by explicitly modeling annotator reliability and prioritizing reliable annotations during training~\cite{rodrigues2014,Rodrigues_Pereira_2018,tanno2019learningnoisylabelsregularized,ibrahim2023,cao2019maxmig,learning_from_crowds}. Such frameworks, in which model parameters and annotator reliability are learned jointly are referred to as end-to-end systems~\cite{ibrahim2023}. These methods focus on robust supervised learning in the presence of annotation noise, while our work is concerned with equipping a dataset with reliable soft labels based on only a few noisy annotations.

\noindent
\textbf{Truth Inference.}
Substantial research has been conducted on truth inference, also referred to as label integration, which aims to infer a single consensus label from multiple, potentially conflicting annotations. More sophisticated approaches than majority voting have been developed and proved to be superior due to the fact that they typically also model annotator reliability and weight annotations accordingly~\cite{David-Skene,whitehill2009,Welinder2010,Liu2012,Zhou2012,karger2013,zhang2014,sheng2019}.
Label integration methods were coupled with active learning techniques to reduce annotation effort in~\cite{Yan2011}.
Our method does not focus on label integration but rather infers soft labels for each data point by diffusing the noisy information obtained on a fraction of the data. 

\noindent
\textbf{Label Propagation.}
Graph-based label propagation methods have been introduced by~\cite{zhu2002,Zhou_2003}. 
In \emph{label propagation}~\cite{zhu2002}, an affinity matrix $W$ is constructed using a Gaussian kernel and used to iteratively update labels on all data points while keeping the initial labels fixed.
Controlled by a stiffness parameter $\alpha$, \emph{label spreading}~\cite{Zhou_2003} employs a diffusion process to transport information throughout the graph via the heat kernel.
Variants of these methods exist, involving hypergraphs~\cite{nonlinear_label_spreading,higherorderlabelhomogeneityspreading}, efficiency and scalability~\cite{fast_label_propagation,efficient_label_propagation} or alternative graph constructions (see~\cite{van_engelen_survey_2020,semi-supervised-learning} for an overview).
Active learning techniques and human-in-the-loop approaches have been explored to further reduce annotation costs, e.g., in cooperative learning~\cite{cooperative_learning} and semi-automated data labeling~\cite{Semi-Automated-Data-Labeling}.

\noindent
\textbf{Learning with Soft Labels.}
Previous graph-based SSL methods primarily operate on hard labels and therefore do not support multiple, potentially conflicting annotations that typically occur in crowd-sourced datasets.
We extend this line of research by generalizing label spreading to soft labels, applying it to large-scale image datasets and studying its consistency in the large data limit.
Soft labels were initially proposed to reflect the confidence level in binary classification~\cite{nguyen_learning_2011} and later extended to multiclass settings to capture aleatoric uncertainty~\cite{Peng2014LearningOP}. Training schemes that draw from soft label information have been proposed~\cite{Peng2014LearningOP}.
While several approaches aim to distill clear-cut labels from soft labels ~\cite{review_of_consensus_algorithms,learning_from_crowds}, generating reliable soft labels remains a more challenging problem. 

\noindent
\textbf{Spreading of Soft Labels.}
Related to our approach is the soft-supervised learning framework~\cite{SoftSupervisedLearning} for text classification.
This approach focuses on a novel optimization objective: minimizing a weighted KL divergence while maximizing entropy as a regularizer.
Our method propagates the results of single annotator feedbacks, resulting in soft labels and uncertainty estimates.
Conditional harmonic mixing~\cite{Conditional_Harmonic_Mixing} is another probabilistic framework for graph-based transductive learning.
However, it takes a different approach to the estimation of soft labels.
This method represents nodes with class probability distributions and edges with conditional probability matrices, using KL divergence minimization across each node's neighborhood for posterior estimation.
Laplacian smoothing in the presence of noisy annotations was also considered in~\cite{zhang2019} for the task of label integration. Their method assumes multiple annotations per instance and a joint low-dimensional space of soft labels and features. 
Our algorithm naturally extends graph-based label spreading to support soft labels. It is applicable to large-scale image datasets, supports incremental updates and provides mathematically grounded uncertainty estimates.


\section{Method}
\label{sec: Method}
\textbf{Label Spreading.}
Zhou et al.~\cite{Zhou_2003} proposed the following graph-based algorithm which became well-known under the name label spreading.
From a set of data points \mbox{$\mathcal{X} = \{x_1, \ldots, x_l,x_{l+1}, \ldots, x_n\}$} out of which the first $l$ were assigned labels $\{y_1,\ldots, y_l\}$ where $y_i\in\{1,\ldots, C\} := \mathcal{C}$ and the remaining ones are unlabeled, a graph is constructed via an affinity matrix $W$. 
A Gaussian kernel models similarity, i.e., the weight $w_{ij} = \exp(-\frac{1}{2\sigma^2}\|x_i -x_j\|^2)$ is assigned to the edge connecting $x_i$ and $x_j$ based on their Euclidean distance scaled by the hyperparameter $\sigma^2$, the width of the Gaussian kernel. 
To avoid self-reinforcement, the diagonal of the affinity matrix is set to zero.  

The normalized graph Laplacian determines the diffusion of information for the graph and is given by \mbox{$\mathcal{L} = I - D^{-\frac{1}{2}} W D^{-\frac{1}{2}}$}, where \mbox{$I\in\RR^{n\times n}$} denotes the identity matrix and $D$ is the diagonal degree matrix, i.e., \mbox{$D_{ii} = \sum_{j=1}^n w_{ij}$}. 
Defining \mbox{$S:= D^{-\frac{1}{2}} W D^{-\frac{1}{2}}$} and introducing a parameter $\alpha \in (0,1)$ that steers the intensity of information diffusion, label spreading is defined by the iterative update rule
\begin{equation}
    F(t+1) = \alpha S F(t) + (1-\alpha) Y,
\end{equation}
where $Y\in \RR^{n\times c}$ encodes the initial label information, i.e., the $i$-th row of $Y$ is the one-hot encoding of the (clear-cut) label $y_i$.
Then, $F \in \RR^{n\times c}$ stores the propagation scores, i.e., the cumulative information obtained from labeled data points of a given class for each data point. 
The limit of this update rule is a Neumann series and can be computed by matrix inversion:
\begin{equation}\label{eq:stationary-state}
    F^\star = \lim\limits_{t\to\infty} F(t) =  (1 - \alpha)(I - \alpha S)^{-1} Y
\end{equation}
From the propagation scores $F^\star$, a classification $\hat{y}_i$ for $x_i$ can be derived by predicting the class with the largest propagation score $F^\star_{ij}$. 
Alternatively, normalizing the propagation scores yields a probability distribution over all classes.

\noindent
\textbf{Task Description.}
We consider an i.i.d.\ sample \mbox{$\mathcal{S}_n=\{x_1,\ldots,x_{\ndata}\}\subset  \RR^d$} of features  sampled from the marginal distribution $P(X)$. 
Given the number of classes $C$ associated with the features, we assume that each $x_i$ comes with a soft label $p_{i} = (p_{i}^1, \ldots, p_i^C)$ defining a probability distribution over the semantic space $\mathcal{C}$.
The annotation of image datasets through crowdsourcing is an illustrative example of how soft labels can be interpreted. 
For a single feature $x_i$ with soft label $p_i$, a crowd-sourcing query returns a feedback (the class assigned by the annotator) $c\in\mathcal{C}$, which is sampled according to $p_i = P(y \vert x_i)$. 
Asking multiple annotators for a label assignment and determining the relative frequencies then leads to an estimate of $p_i$.\footnote{Note that for probabilistic datasets such as CIFAR-10-H~\cite{cifar10-H} this is precisely how soft labels are constructed. While technically those are merely histogram estimates to the underlying distribution, we do not make this distinction explicit.}
We assume that a crowdsourcing process generates a feedback $c$ for the data point $x_i$ according to its soft label $p_i$, i.e., $c$ is an i.i.d.\@ sample of the categorical distribution associated with $x_i$.

We consider the following task: Given a budget of $m\in \NN$ crowdsourcing queries that can be conducted, the soft labels $\{p_1, \ldots, p_\ndata\}$ are to be estimated as accurately as possible. 
From a semi-supervised perspective, we usually consider $m < \ndata$, even though in the case of probabilistic datasets, any finite budget will result in merely an estimate to the actual soft labels. 
We quantify estimation performance using the root mean squared error (RMSE) across the dataset.

\noindent
\textbf{Probabilistic Label Spreading.}
\begin{figure}[tb]
    \centering
    \includegraphics[width=\linewidth]{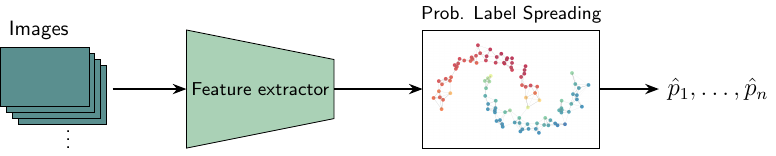}
    \caption{Overview of our method. Images are embedded using a feature extractor (e.g.\ a vision-language model combined with a dimensionality reduction technique). PLS is then applied in the embedding space to estimate soft labels for all features.}
    \label{fig:method-graph}
\end{figure}
We build on the label spreading algorithm \cite{Zhou_2003} and adapt it such that it supports soft labels and is tailored to the use case of semi-automated labeling in a crowdsourcing setting. 
The main modification is that label spreading is performed iteratively, each time with only a single seed $x_i$ that has been assigned a feedback $c$ sampled from its soft label $p_i$. 
In order to ensure efficiency and scalability to large datasets, we use a sparse $k$ nearest neighbor ($k$-NN) graph as well as efficient solvers to approximate the heat kernel $(I-\alpha S)^{-1}$.

As in \cite{Zhou_2003}, a Gaussian kernel is used to construct an affinity matrix $W$. 
However, two nodes are connected only if one is among the $k$-NN of the other. 
The resulting affinity matrix $W$ is symmetrized and then normalized by the degree matrix $D$. 
We set the bandwidth of the Gaussian kernel $\sigma^2$ as the average over the $k$-NN radii. 
Having constructed the graph structure for the features $\{x_1, \ldots, x_\ndata\}$, for a given budget $m$, accordingly many crowdsourcing queries are drawn. 
The obtained label information of each crowdsourcing query is spread via the heat kernel. 
More precisely, the following procedure is repeated $m$ times: 
Select a feature $x_i$ out of $\mathcal{S}_n$ uniformly and sample a feedback $c$ according to $P(y|x_i)$. 
Then perform \eqref{eq:stationary-state} with the single seed $x_i$ labeled as class $c$ and store the resulting propagation scores $\phi \in \RR^n$ per class $Y^c \in \RR^n$. 
Finally, normalize the cumulative propagation scores $Y^1, \ldots, Y^C \in \RR^n$ to obtain an estimated soft label of each feature and predict a uniform soft label if a feature has received zero information. Given, the sample $\{(x_{i_1}, c_{i_1}), \ldots, (x_{i_m}, c_{i_m}))\}$ where $c_{i_j}$ are one-hot encodings of the sampled feedback, the algorithm returns the estimate
\begin{equation}
    \hat{p}_q = \frac{\sum_{j=1}^{m} (I-\alpha S)^{-1}_{q,i_j}\, c_{i_j}}{\sum_{j=1}^{m} (I-\alpha S)^{-1}_{q,i_j}} \in[0,1]^C,\; q=1,\ldots,n\,.
\end{equation}
The probabilistic label spreading (PLS) algorithm accounts for the spatial density of both labeled and unlabeled data points. 
Information is spread more intensely in high-density regimes and 
the overall spreading intensity can be steered with the parameter $\alpha$.

\noindent
\textbf{PAC Learnability.}
We assume that the feature space $\mathcal{X}\subset \RR^d$ is a valid region based on definition 2 of~\cite{vonLuxburg2014} and that the target distribution $P(y\vert x)$ is Lipschitz-continuous. In particular, on a valid region the marginal density $p_X$ is bounded by constants $p_{\min},p_{\max} > 0$.
We consider $\varepsilon$-graphs to construct the affinity matrix, because the neighborhood relationship is directly linked to Euclidean distance.

To achieve consistency, we increase the spreading intensity $\alpha$ to reduce variance while also reducing the bandwidth of the underlying $\varepsilon$-graph to control bias. We configure the rates such that both variance and bias can be controlled, i.e., $\alpha\to 1$ sufficiently slow such that shrinking neighborhoods still fill up with data points. Bias is controlled via a path length argument for which we exploit the neighborhood structure and the fact that weights are multiplied with $\alpha$ on each edge so that the influence from nodes via long paths is small. 
We provide the proof of \autoref{thm: Theorem} in the supplementary material.
\begin{theorem}
\label{thm: Theorem}
    Let $\varepsilon> 0$ and $\delta \in(0,1)$. Under the above assumptions, when choosing the spreading intensity as \mbox{$\alpha_n = 1 - n^{\frac{1}{d+1}} \to 1$} dependent on the dataset size, the bandwidth of the epsilon-graph as \mbox{$h_n = \varepsilon /  (3 L_y \lceil\log_{\alpha_n} \varepsilon / (\sqrt{2}\,12) \rceil) \to 0$}, and providing an annotation budget of at least \mbox{$m_n = O(n^{1 - \frac{1}{2(d+1)}} \log n)$}, for n sufficiently large, the PLS algorithm returns estimated soft labels $\hat{p}_1, \ldots, \hat{p}_n$ such that \mbox{$P(\exists\, q\in\{1,\ldots, n\}: \| \hat{p}_q - p_q\|_2 > \varepsilon) \leq \delta$}.

    More precisely, depending on $\varepsilon$ and $\delta$, the following sample complexity bound holds for $n$:
    \begin{align}
        &P(\exists\, q\in\{1,\ldots, n\}: \| \hat{p}_q - p_q\|_2 > \varepsilon) \nonumber \\ &\leq 12 n C \exp\left( - \zeta  \frac{\nsamples}{n^{1-\frac{1}{2(d+1)}}}\right)
    \end{align}
    with some constant $\zeta$ depending on $\varepsilon$, the minimal and maximal density $p_{\min}$, $p_{\max}$, the feature space $\mathcal{X}$ and its dimension $d$, the number of classes $C$, and the Lipschitz constant $L_y$.
\end{theorem}

\noindent
\textbf{Baseline Algorithms.}
We compare our method to the following baseline algorithms for soft label estimation:
\begin{itemize}[left=0pt]
    \setlength\itemsep{0em}
    \item 
    No information spreading ($\alpha = 0$). This is commonly applied in practice; annotations are queried per data point and the relative class frequencies represent soft label estimates.
    \item 
    Gaussian Kernel Regression (GKR)~\cite{Nadaraya1964,Watson1964}. Labels are spread by a Gaussian kernel, corresponding to label spreading without the underlying graph.
    This method does not exploit the distribution of unlabeled data. 
    A feedback obtained in feature $x_i$ spreads to $x_j$ with a weight of $\exp({-\gamma \norm{x_i - x_j}_2^2})$, 
    steered by $\gamma > 0$. 
    \item 
    $k$-NN Regression.
    For fixed $x_i$, we determine the $k$ nearest neighbors which have obtained at least one feedback and average their histogram estimates.
\end{itemize}

In our numerical experiments, we use CLIP~\cite{CLIP} as a pre-processing to perform dimensionality reduction. Thus, we also consider zero-shot predictions of CLIP as a baseline.
We compute the cosine distances between the text embeddings of the prompts ``an image of a \{class\}'' and the image embedding.\footnote{We adjusted this basic prompt for the EMNIST-digits dataset ("A black and white image of the handwritten digit \{class\}") and the MTSD dataset ("an image of a street sign in the shape of a \{class\}").}
With the model's temperature parameter $\tau$, the softmax function converts distances $d_1, \ldots, d_C$ to class probabilities via  \mbox{$p_i =  e^{\tau d_i}/\sum_{j=1}^C e^{\tau d_j}$}, $i=1,\ldots, C$.

Semi-supervised learning techniques typically assume local consistency~\cite{Zhou_2003} meaning that features that are close to one another in the Euclidean sense are also semantically similar.
A meaningful graph structure can only be constructed if the data exhibits local consistency and distances can be measured in a meaningful way, which requires a low-dimensional feature space due to the curse of dimensionality~\cite{bellmann1957}. 
Hence, to apply our method on high-dimensional data such as images, we use a feature extractor as a pre-processing step to obtain low-dimensional features for which the assumption of local consistency holds to a reasonable extent. 
Figure \ref{fig:method-graph} gives an overview of our method.

We illustrate our algorithm on the two moons toy dataset, for which soft labels are generated using a simple feed-forward neural network. See appendix for the illustrations.
As the annotation budget increases, the estimated soft labels become more precise.
Since the graph forms a single connected component, each (virtual) annotator feedback propagates label information to the entire dataset.
With only a single feedback, all data points are assigned a clear-cut label of the sampled class; meaningful probability estimates emerge only once both classes have been observed.
In regions where one class dominates, predictions can be overconfident, but this effect diminishes as additional feedbacks also introduce less probable classes and refine the estimates.

\newcommand{\rot}[1]{#1} 

\begin{table*}[t]
\caption{Performance for different datasets and algorithms with 10\,\% budget. Mean RMSE and standard deviation across $10$ evaluations are reported in units of $\times 10^{-3}$. We indicate model hyperparameters $\alpha$ for PLS, $\gamma$ for GKR and $k$ for kNN.\\ ($*$) denotes optimized parameters.
}
\label{tab:RMSE_all_datasets_10_percent}
\centering
\scalebox{0.85}{
\begin{tabular}{llllllll} 
\toprule
    & \rot{TwoMoons} & \rot{Anim.10} & \rot{EMNIST-dig.} & \rot{CIF.10} & \rot{CIF.10-H} & \rot{Tiny I.N.} & \rot{MTSD} \\
\midrule
PLS (0.5) & $65.7 \,(27.6)$ & $71.0 \,(1.1)$ & $57.8 \,(0.3)$ & $114.9 \,(0.5)$ & $123.5 \,(2.0)$ & $47.5 \,(0.2)$ & $192.6 \,(0.4)$ \\
PLS (0.9) & $50.3 \,(6.9)$ & $64.7 \,(0.9)$ & $54.2 \,(0.2)$ & $105.9 \,(0.4)$ & $109.8 \,(0.9)$ & $42.8 \,(0.1)$ & $122.4 \,(0.4)$ \\
PLS (0.99) & $64.9 \,(3.5)$ & $62.5 \,(0.4)$ & $56.0 \,(0.1)$ & $110.1 \,(0.2)$ & $113.2 \,(0.7)$ & $45.4 \,(0.1)$ & $88.0 \,(0.4)$ \\
PLS ($*$) & \textbf{\boldmath$52.6$}$\,(10.4)$ & \textbf{\boldmath$61.9$}$\,(0.5)$ & \textbf{\boldmath$53.7$}$\,(0.2)$ & \textbf{\boldmath$104.9$}$\,(0.4)$ & \textbf{\boldmath$109.1$}$\,(1.1)$ & \textbf{\boldmath$42.5$}$\,(0.1)$ & $86.9 \,(0.4)$ \\
\midrule
GKR (0.1) & $458.7 \,(2.3)$ & $75.4 \,(0.3)$ & $72.1 \,(0.2)$ & $170.4 \,(0.7)$ & $177.9 \,(1.1)$ & $60.4 \,(0.0)$ & $95.8 \,(0.1)$ \\
GKR (1) & $322.0 \,(10.1)$ & $65.3 \,(0.2)$ & $61.0 \,(0.1)$ & $115.4 \,(0.3)$ & $114.6 \,(0.5)$ & $53.4 \,(0.1)$ & $86.0 \,(0.3)$ \\
GKR (10) & $77.6 \,(15.9)$ & $65.0 \,(0.8)$ & $59.2 \,(0.1)$ & $108.3 \,(0.3)$ & $112.4 \,(1.8)$ & $44.4 \,(0.1)$ & $113.1 \,(0.7)$ \\
GKR ($*$) & $59.5 \,(18.2)$ & $63.0 \,(0.4)$ & $55.6 \,(0.1)$ & $108.0 \,(0.2)$ & $110.3 \,(0.9)$ & $44.3 \,(0.1)$ & \textbf{\boldmath$85.7$}$\,(0.3)$ \\
\midrule
kNN (5) & $62.9 \,(13.0)$ & $69.4 \,(0.9)$ & $61.3 \,(0.2)$ & $115.4 \,(0.3)$ & $117.4 \,(1.1)$ & $47.0 \,(0.1)$ & $164.7 \,(0.5)$ \\
kNN (20) & $176.9 \,(21.4)$ & $66.0 \,(0.3)$ & $60.2 \,(0.1)$ & $110.6 \,(0.3)$ & $112.4 \,(0.9)$ & $44.8 \,(0.1)$ & $108.5 \,(0.7)$ \\
kNN (50) & $378.1 \,(12.9)$ & $65.7 \,(0.2)$ & $60.5 \,(0.1)$ & $111.5 \,(0.2)$ & $114.2 \,(0.5)$ & $47.6 \,(0.1)$ & $93.5 \,(0.5)$ \\
kNN ($*$) & $62.5 \,(19.5)$ & $64.9 \,(0.1)$ & $57.1 \,(0.2)$ & $110.1 \,(0.3)$ & $113.3 \,(0.5)$ & $44.8 \,(0.0)$ & $87.8 \,(0.3)$ \\
\midrule
CLIP 0-shot & nan & $65.4$ & $288.6$ & $126.5$ & $110.5$ & $43.9$ & $164.0$ \\
\bottomrule
\end{tabular}
}

\end{table*}

\section{Numerical Results}
\label{sec: Numerical Results}

We compare the performance of our method and baseline algorithms on several image classification datasets. We examine how performance is affected by external parameters such as annotation budget and dataset size. Detailed hyperparameter and ablation studies foster a deeper understanding. We underline the applicability of our framework by benchmarking it on the data-centric image classification benchmark for noisy and ambiguous label estimation~\cite{schmarje2022benchmark}. 

\noindent
\textbf{Implementation Details.}
Determining the solution of the label spreading algorithm requires the inversion of a dense matrix so that the computational cost scales cubic with the number of data points. 
An approximate solution can be computed by evaluating a partial sum of the Neumann series, which is done in practice~\cite{scikit-learn}.
To reduce computational costs, the fully connected graph can be replaced with a sparse one, e.g., by only connecting nodes if one is among the $k$-NN of the other or, only if their Euclidean distance does not exceed $\varepsilon$.
The $\varepsilon$-graph induces a symmetric affinity matrix $W$, which is generally not the case for the $k$-NN graph since the $k$-NN relationship is not symmetric. 
One may account for this by symmetrizing the affinity matrix, $\Tilde{W} := (W + W^T)/2$. 
If $k$ is small compared to $n$, $\mathcal{L}$ is a sparse matrix. 
Our method uses a $k$-NN graph as it is more intuitive to choose a natural number $k$ than defining a kernel bandwidth $\varepsilon$.
We employ algebraic multigrid methods~\cite{Trottenberg2000Multigrid} to efficiently solve the sparse linear system involved in label spreading, specifying an error tolerance of $10^{-6}$.
Based on insights of our hyperparameter analysis, we define a default parameter configuration used in the following experiments if not stated otherwise.
The feature space has a dimension of $d=20$ and consists of the image embeddings generated by CLIP~\cite{CLIP} with a ViT-B/32 backbone, which are further reduced using UMAP~\cite{UMAP} with default parameters. We use $k=20$ neighbors to construct the graph and to compute the kernel bandwidth $\sigma^2$. 

\noindent
\textbf{Benchmark Comparison.}
We compare the performance of the PLS algorithm with baselines on the following datasets: Animals-10~\cite{animals10}, EMNIST-digits~\cite{EMNIST}, CIFAR-10~\cite{cifar10}, the crowd-sourced CIFAR-10-H~\cite{cifar10-H}, Tiny ImageNet~\cite{tinyimagenet} and Mapillary Traffic Sign Dataset (MTSD) -- a proprietary dataset of street signs. Some datasets come with soft labels while others do not. If soft labels are unavailable, we simulate them through the predicted class probabilities of a neural network. 
We describe the datasets and the simulation of soft labels in more detail in the appendix.
For each algorithm, we consider a set of plausible parameter values and additionally conduct a hyperparameter tuning using the Bayesian optimization framework SMAC~\cite{SMAC_optimization_framework}. Performance in terms of RMSE is evaluated for different budgets of $1\,\%$, $10\,\%$ feedbacks relative to the dataset size. Note that a budget of $10\,\%$ does not imply that every tenth data point is assigned exactly one annotator feedback as sampling is conducted with replacement. 
We chose RMSE over distributional distance measures such as KL-divergence due to its interpretability. 
The RMSE represents the estimated standard deviation with which our soft label estimate varies around the true soft label.

Table \ref{tab:RMSE_all_datasets_10_percent} lists the performance in terms of RMSE averaged over $10$ runs for a budget of $10\,\%$.
Our method performs best on all but the MTSD industry dataset albeit the differences are small. However, we observe a significant performance gain on the Two Moons dataset where density-based estimation is beneficial and purely distance-based estimation may fail, particularly in the regime of few samples.
For a very small budget of $1\,\%$, CLIP zero shot classification is best on three out of the six datasets while the GKR baseline performs best on two datasets as well (see appendix). An explanation for the superior performance of the GKR baseline in the regime of very low supervision may be that it does not rely on a local graph structure and hence, every data point is guaranteed to obtain some information. Generally, we cannot claim the same for the PLS algorithm as it might happen that some features are not connected to any of the annotated data points. Even if there is a path connecting them, it might be so long that propagated information approaches zero.\footnote{To prevent numerical instabilities, we add a small value to the propagation scores prior to normalizing which leads to a uniform estimate in case of vanishing propagation scores.}

\begin{figure*}[tb]
  \centering
  
  \begin{minipage}{\textwidth}
    \centering

    \begin{subfigure}[t]{0.48\textwidth}
      \centering
      \includegraphics[width=\linewidth]{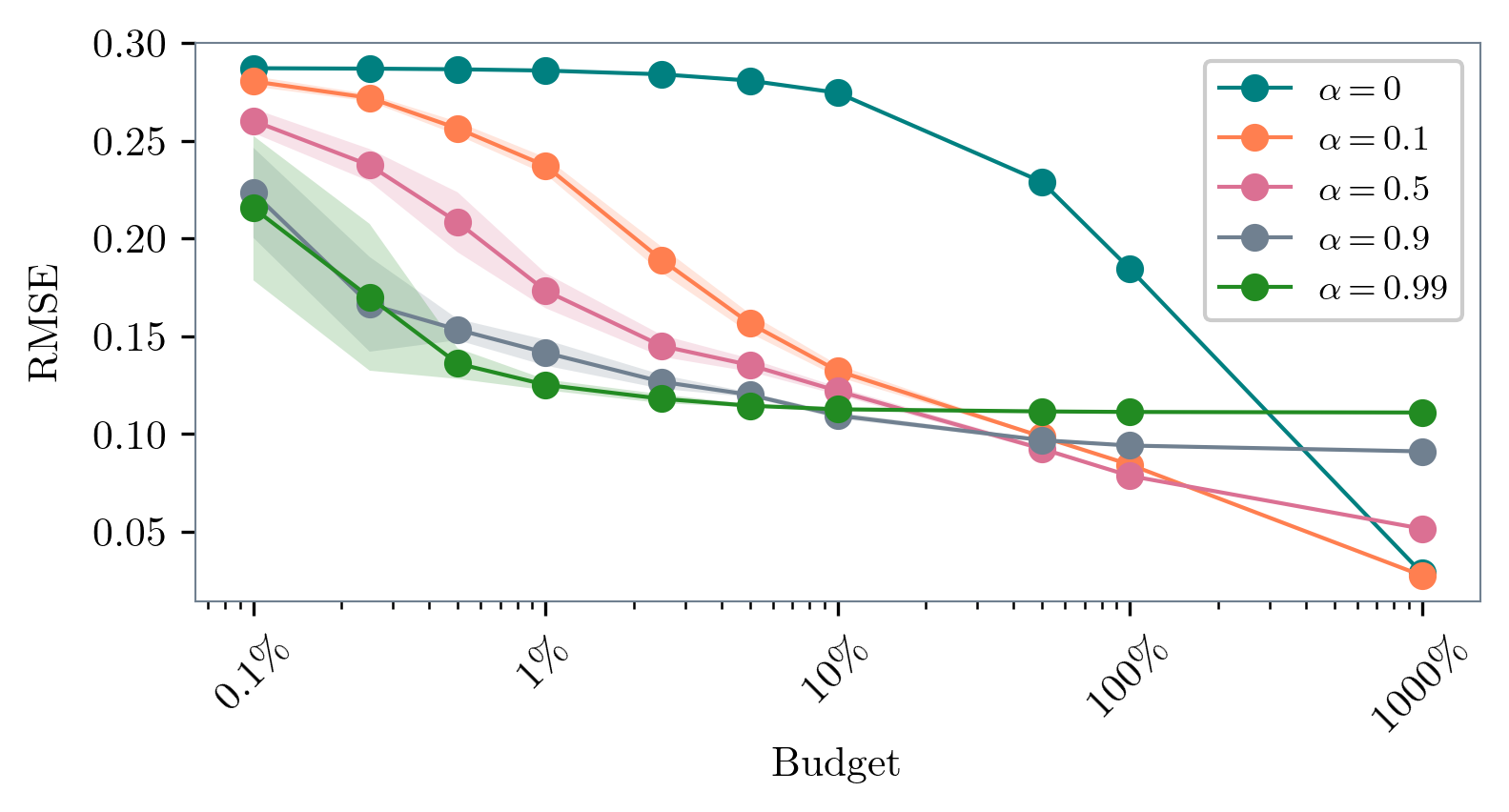}
      \caption{Performance trajectories for different spreading intensities $\alpha$.}
      \label{fig:PLS_RMSE_trajectories_on_CIFAR10-H}
    \end{subfigure}
    \hfill
    \begin{subfigure}[t]{0.48\textwidth}
      \centering
      \includegraphics[width=\linewidth]{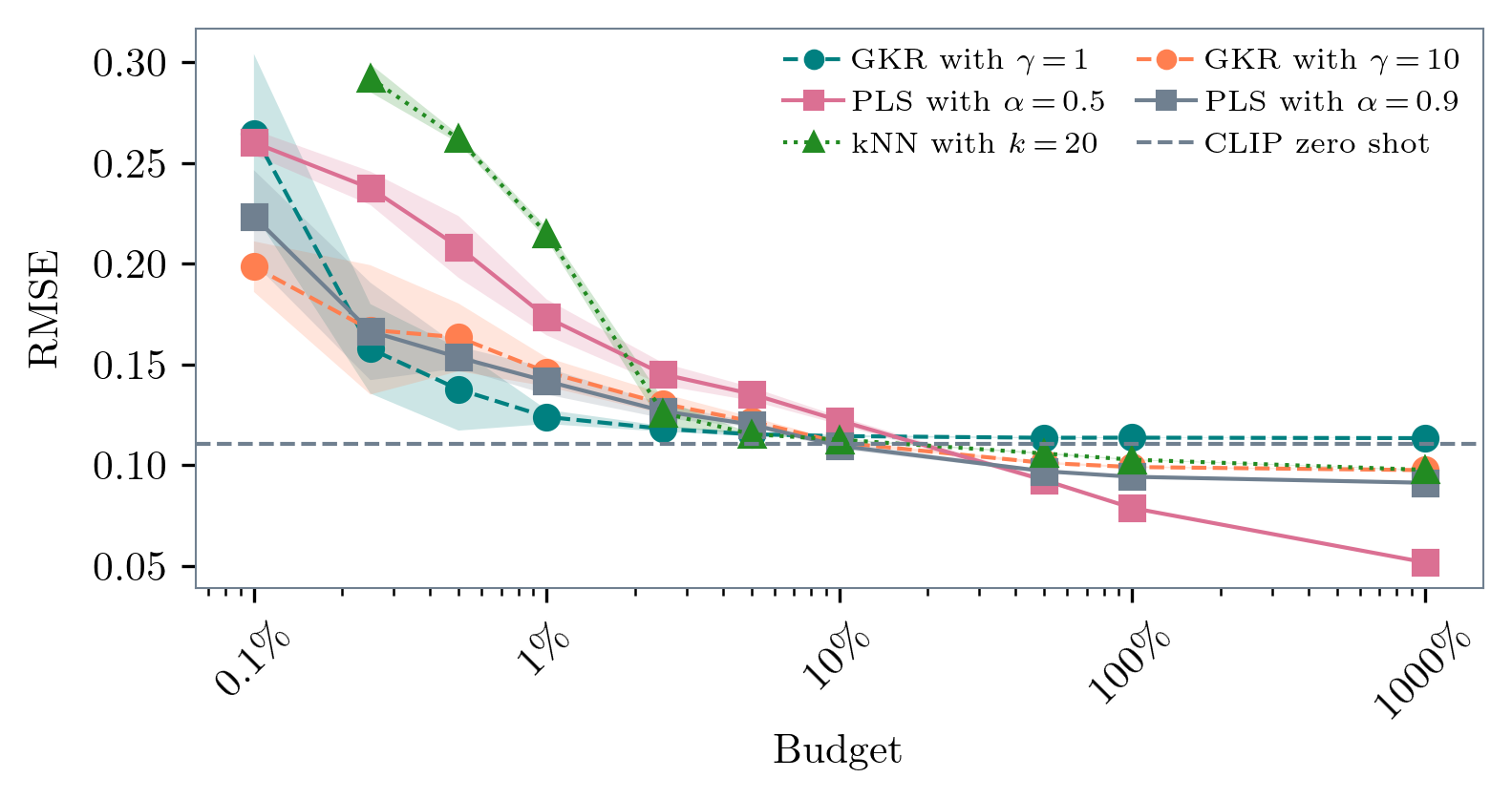}
      \caption{Comparison with baseline algorithms.}
      \label{fig:comparison_of_RMSE_trajectories_on_CIFAR10-H}
    \end{subfigure}

    \caption{Comparison of performance trajectories on CIFAR-10-H for our method with different spreading intensities $\alpha$ (a) and, in comparison with baseline algorithms (b). With a growing number of annotations provided to the algorithm, RMSE decreases.}
  \end{minipage}

\end{figure*}

\noindent
\textbf{Performance Depending on Annotation Budget.}
Given that the CIFAR-10-H dataset contains actual crowd-sourced annotations, we consider this dataset in more detail. CIFAR-10-H comprises the original $10,000$ images of the CIFAR-10 test data with approximately $50$ annotations per image \cite{cifar10-H}. We consider the relative class frequencies in the annotations as the soft label associated to an image.

Figure \ref{fig:PLS_RMSE_trajectories_on_CIFAR10-H} displays the performance of the PLS algorithm for different budgets and spreading parameters $\alpha$. Larger values of $\alpha$ lead to a superior performance in the regime where few feedbacks have been assigned but also come with inferior performance when the annotation budget is relatively large. Intuitively, if very few data points carry information, this information should be spread further and vice-versa. In other words, the larger $\alpha$, the stiffer the estimated function in accordance to the regularization framework of the label spreading algorithm \cite{Zhou_2003}. It can be observed, that the performance exhibits a large variance for a small number of samples, particularly, if $\alpha$ is large. In this case, the performance heavily depends on the sample of data points which have been provided with annotator feedback.

The baseline of not spreading information at all ($\alpha = 0$) requires many annotations to approximate the soft labels with data point-wise estimates. However, in the regime where many annotations per image are acquired, this method surpasses in performance. After all, it corresponds to sampling from the actual soft label of individual data points, making it unbiased and independent of the the feature space structure. As the maximal budget in this experiment is $10$ times the dataset size, the class probability estimates are still slightly off from the ground truth because the budget for dataset construction was around $50$ times the dataset size.

Figure \ref{fig:comparison_of_RMSE_trajectories_on_CIFAR10-H} compares the performance trajectories of our method with the ones of the baseline algorithms. As previously, we observe a relatively similar performance. In particular, the GKR baseline can compare to the PLS algorithm with large $\alpha$-values in the regime of few feedbacks. The $k$-NN baseline algorithm achieves a similar performance if a certain number of feedbacks is supplied but cannot perform well in the regime of few feedbacks. This is to be expected since it requires at least $k C$ feedbacks to create a situation where averaging over the $k$ nearest neighbors is meaningful (assuming well-clustered data). Given that all algorithms can be steered towards a probability estimation in a per-data point class-histogram manner ($\alpha = 0$, $k = 1$, $\gamma \to \infty$), we can find parameters that perform well for large budgets.

\begin{table}[tb]
\caption{
RMSE for different datasets and feature extractors (CLIP, ResNet-101 and ViT-B/32) with 10\,\% budget. The extracted features are further reduced using UMAP with a target dimension of 20.
}
\label{tab:ablation_feature_space}
\centering
\resizebox{\linewidth}{!}{
\renewcommand{\rot}[1]{\rotatebox{70}{\parbox{1.1cm}{\centering #1}}}
\begin{tabular}{lrrrrrr}
\toprule
    & \rot{Anim.10} & \rot{EMNIST-digits} & \rot{CIF.10} & \rot{CIF.10-H} & \rot{Tiny I.N.} & \rot{MTSD} \\
\midrule
CLIP (0.5) & 0.071 & 0.061 & 0.117 & 0.123 & 0.047 & 0.191 \\
CLIP (0.9) & 0.065 & 0.057 & 0.107 & \textbf{0.111} & 0.043 & 0.123 \\
CLIP (0.99) & \textbf{0.064} & 0.059 & 0.111 & 0.113 & 0.045 & 0.088 \\
\midrule
RN-101 (0.5) & 0.076 & 0.078 & 0.117 & 0.125 & 0.042 & 0.191 \\
RN-101 (0.9) & 0.069 & 0.074 & 0.109 & 0.118 & 0.037 & 0.121 \\
RN-101 (0.99) & 0.067 & 0.079 & 0.112 & 0.122 & 0.039 & 0.089 \\
\midrule
ViT-B (0.5) & 0.093 & 0.059 & 0.116 & 0.118 & 0.041 & 0.191 \\
ViT-B (0.9) & 0.086 & \textbf{0.056} & \textbf{0.105} & 0.111 & \textbf{0.036} & 0.119 \\
ViT-B (0.99) & 0.087 & 0.057 & 0.107 & 0.111 & 0.036 & \textbf{0.087} \\
\bottomrule
\end{tabular}
}
\end{table}

\begin{figure}[tb]
    \centering
    \includegraphics[width=\linewidth]{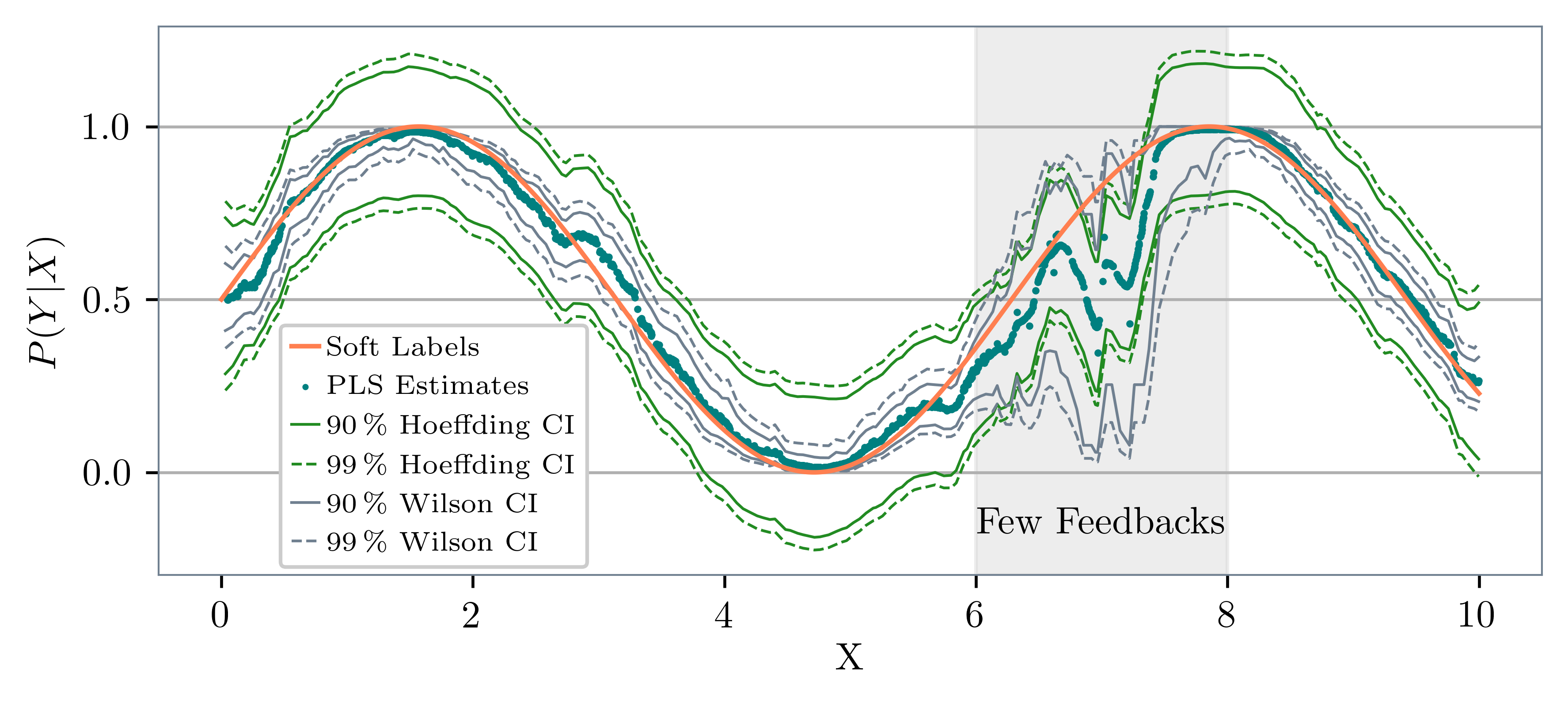}
    \caption{Confidence Intervals for the estimates of our method derived from 1) Wilson score and 2) Hoeffding type bounds.}
    \label{fig:confidence_intervals}
\end{figure}

\noindent
\textbf{Hyperparameter and Ablation Studies.} 
As an ablation study, we determine the performance of our method depending on the feature extractor used. 
In addition to CLIP, we employ a ResNet-101 and a vision transformer (ViT/B-32) as ImageNet-pretrained models to create different feature spaces. We apply UMAP to further reduce the feature dimension to $d=20$. For each dataset, we fix a budget of $10\,\%$ and apply the PLS algorithm for different $\alpha$ values. Table \ref{tab:RMSE_all_datasets_10_percent} (b) lists the resulting RMSE scores averaged over $10$ evaluations for the three feature extractors on each dataset.
Given that our method builds on the assumption of local consistency, it can be expected that its performance heavily depends on the data structure. Indeed, it can be observed that the quality of the predicted soft labels depends on the feature extractor used. For instance, the ImageNet-pretrained models lead to a superior performance on the TinyImageNet dataset, which is a subset of ImageNet. Given that all models are of high-capacity, the resulting feature spaces all seem to capture semantic similarity relatively well so that the differences in performance are moderate. 

Apart from the image pre-processing technique, the hyperparameters of the algorithm include the number of neighbors $k$, which steers the connectivity of the graph and the spreading parameter $\alpha$.
Intuitively, we can expect that:
\begin{itemize}[left=0pt] 
    \item Choosing $k$ too small can impede information propagation and break graph connectivity, while a dense graph risks propagating information between unrelated features.
    \item If $\alpha$ is small, overall information diffusion is limited and very local. The larger $\alpha$, the more intense and further information is spread, resulting in a stiff estimate.
    \item Further, both of the parameters are not independent as they both influence the degree of information exchange. E.g., a large spreading parameter may increase information exchange between dissimilar data points in a dense graph.
\end{itemize}
In the previous experiments, larger values of $\alpha$ lead to a superior performance in the regime of few feedbacks but also come with inferior performance for large annotation budgets. This trade-off can also be observed in a detailed hyperparameter study provided in the appendix, where performance is displayed for a variety of $\alpha$ values and different budgets for the CIFAR-10-H dataset. To determine the impact of $k$ on the algorithm's performance, we fix a budget of $10\,\%$ and evaluate our method on CIFAR-10-H depending on $k$ for selected $\alpha$ values. We observe that, if $\alpha$ is small, considering more neighbors in the graph construction proves beneficial whereas the opposite seems to be true for very large values of $\alpha$. 
This interaction particularly shows when considering both parameters jointly. For this experiment, we considered three different budgets and different combinations of $\alpha$ and $k$ for which the algorithm's performance on CIFAR-10-H is evaluated. For relatively small values of $\alpha$, the performance increases with a larger $k$ whereas for large values of $\alpha$ the opposite is true. In particular, if both $\alpha$ and $k$ are small, the performance of the algorithm is poor as the degree of information exchange is too small. 
The choice of hyperparameters should depend on the data structure and the annotation budget. Well-structured data justifies increased information exchange but if this is not the case, the connectivity of the graph needs to be steered carefully. Generally, the more label information is acquired, the smaller $\alpha$ may be chosen in practice.

\noindent
\textbf{Confidence Intervals.}
In our method, information propagated through the graph is normalized. This way, propagation scores obtained in neighboring data points can be interpreted as partially many virtual experiments. Due to this design choice, we can consider confidence intervals (CIs) based on the number of virtual experiments and their outcomes. In the example in figure \ref{fig:confidence_intervals}, we consider Wilson score based CIs~\cite{Wilson1927}. We observe that the soft labels mostly lie within CIs and the width of the CIs varies with the amount of information obtained. In particular, CIs are wide where provided label information is scarce.
In our theoretical analysis in the appendix, we apply Hoeffding's inequality for bounded random variables (Theorem 2.2.6 of~\cite{vershynin_highdimprob_2026}) to bound the deviation from the expected estimate of our method. 
Hoeffding CIs are also displayed in figure~\ref{fig:confidence_intervals} and include a bias bound that is based on knowledge of the derivative of the soft label ground truth. We refer to the appendix for more detailed information on this experiment.

\noindent
\textbf{Data-Centric Image Classification Benchmark.}
The Data-Centric Image Classification benchmark, introduced by \cite{schmarje2022benchmark}, evaluates label improvement methods such as semi-supervised learning algorithms.
The task is to label 10 probabilistic datasets based on provided annotations such that an image classifier trained on those labels generalizes to unseen data.
Since labels are probabilistic, performance is measured using the KL divergence. For comparison, the mean and median performance differences relative to a fully supervised baseline are reported across datasets.
As shown in table~\ref{tab:benchmark_results}, our method outperforms the state of the art in this benchmark. It is best on low annotation budgets as well as on large budgets where each image is annotated at least once. This reflects the flexibility of our approach, enabled by the choice of the spreading intensity~$\alpha$, which can be reduced for larger annotation budgets.
\begin{table}[t]
    \caption{Benchmark comparison~\protect\cite{schmarje2022benchmark} of label improvement methods under various budgets (PLS $\alpha$: 0.9, 0.75, 0.1).} 
    \label{tab:benchmark_results}
    \centering
    \renewcommand{\arraystretch}{1.0} 
    \resizebox{\linewidth}{!}{
    \begin{tabular}{lcc cc cc}
        \toprule
        \makecell[r]{\textbf{Budget}}   & \multicolumn{2}{c}{10\%} & \multicolumn{2}{c}{100\%} & \multicolumn{2}{c}{1000\%} \\
        \cmidrule(lr){2-3} \cmidrule(lr){4-5} \cmidrule(lr){6-7}
        \textbf{Method} & Median & Mean (SEM) & Median & Mean (SEM) & Median & Mean (SEM) \\
        \midrule
        ELR+           & -0.16  & -0.59 (0.22)  & -0.13  & -0.17 (0.06)  & 0.15  & 0.24 (0.06)  \\
        SGNP           & -0.26  & -0.59 (0.27)  &  0.00  & -0.15 (0.08)  & 0.20  & 0.26 (0.04)  \\
        DivideMix      & -0.21  & -0.78 (0.25)  & -0.03  & -0.17 (0.08)  & 0.16  & 0.31 (0.07)  \\
        Mean           & -0.14  & -0.59 (0.22)  & -0.12  & -0.22 (0.08)  & N/A  & N/A  \\
        $\pi$          & -0.33  & -0.63 (0.19)  & -0.06  & -0.17 (0.08)  & N/A  & N/A  \\
        $\pi$+DC3      & -0.32  & -0.62 (0.20)  & -0.10  & -0.22 (0.06)  & N/A  & N/A  \\
        PseudoV2h & -0.30  & -0.39 (0.16)  & -0.03  & -0.19 (0.08)  & 0.19  & 0.28 (0.04)  \\
        PseudoV2s & -0.28  & -0.40 (0.19)  & -0.04  & -0.17 (0.06)  & 0.01  & 0.00 (0.02)  \\
        MOCOv2         & -0.29  & -0.63 (0.22)  & -0.08  & -0.13 (0.10)  & N/A  & N/A  \\
        \midrule
        PLS (ours)
                       & \textbf{-0.42}  & \textbf{-0.82 (0.19)}  & \textbf{-0.17}  & \textbf{-0.33 (0.08)}  & \textbf{-0.02} & \textbf{-0.03  (0.01)}\\
        \bottomrule
    \end{tabular}
    }
\end{table}

\noindent
\textbf{Runtime.}
After data pre-processing and dimensionality reduction ($\lessapprox$ 10 GPU hours on a 24 GB NVIDIA Quadro P6000 and 2 Intel(R) Xeon(R) Gold 6138 CPUs with 20 cores, 40 threads and 512 GB of RAM), spreading a single annotator feedback ranges between 1\,ms on a dataset with 1,000 data points and 260\,ms on a dataset with 1,000,000 data points. Total runtime then scales linearly with the annotation budget.

\section{Limitations}
\label{sec: Limitations}
Similar to local nonparametric SSL algorithms, our method requires the target function to be smooth and would suffer from the curse of dimensionality if the target function varies significantly, as illustrated by~\cite{Bengio_2006}. 
SSL methods are applicable for semi-automated labeling only if there is reason to believe that the local consistency assumption holds. To create a feature embedding in which Euclidean distances reflect semantic similarity, we rely on powerful image encoders, which however, have their own limitations to date. In turn, the smoothness assumption may only hold to a certain extent in many applications and the risk of spreading wrong label information exists. By steering the spreading intensity, the trade-off between a rather stiff solution obtained with minimal annotation workload and a more accurate estimate requiring more annotations can be maneuvered.


\section{Conclusion}
Our work is motivated by the question of how large, unannotated datasets can be consistently equipped with soft labels while minimizing annotation effort. To this end, we proposed a learning algorithm for semi-automated dataset annotation in crowdsourcing settings where probabilistic labels are desired. Our method generalizes label spreading to a probabilistic, transductive framework that allows for differing annotations by multiple annotators.
Beyond point estimates, our approach provides epistemic uncertainty quantification, reflecting confidence in the inferred soft labels.
We establish consistency and PAC-style guarantees under a smoothness assumption, showing that reliable soft label estimates can be obtained even when the number of annotations per data point vanishes. Extensive experiments on standard image classification datasets demonstrate that our method outperforms baseline approaches and substantially reduces annotation workload compared to per-image labeling, without sacrificing label quality. The flexibility of the method is governed by intuitive hyperparameters controlling the degree of information propagation.
Finally, we demonstrate the practical relevance of our approach through state-of-the-art results on a data-centric image classification benchmark for semi-automated labeling. To support reproducibility and further research, we make our source code publicly available under \url{https://github.com/JonathanKlees/prob_label_spreading}.

\newpage
\section*{Acknowledgements}
J.K.\ and M.R.\ acknowledge support by the BMFTR within the project ``RELiABEL'' (grant no.\ 16IS24019B).
D.K.\ also acknowledges support within the project ``RELiABEL'' (grant no. 16IS24019A).
P.S.\ and M.R.\ acknowledge support by the state of
North Rhine-Westphalia and the European Union within
the EFRE/JTF project ``Just scan it 3D'' (grant no. EFRE20800529).

\bibliographystyle{IEEEtran}
\bibliography{literature}

\newpage

\renewcommand{\thefigure}{A.\arabic{figure}} 
\renewcommand{\thetable}{A.\arabic{table}}   
\setcounter{figure}{0}                        
\setcounter{table}{0}                         

\appendix
\label{sec:Appendix}
\subsection{Supplementary Material Overview}
This appendix provides a comprehensive complement to the main manuscript, offering additional empirical studies, all methodological details, as well as the proof of consistency in the large data limit, which were beyond the scope of the main paper. Below, we provide a structured overview of the appendix to help the reader navigate the supplementary experiments, methodological clarifications, and theoretical results.
\begin{itemize}[left=0pt]
\item
    \textbf{Section \ref{sec-appdx:additional-numerical-results}} contains information on datasets and soft labels that we used as well as additional insightful experiments that extend the analysis of our method.
    \begin{itemize}[left=0pt]
        \item It starts with a detailed description of all datasets and our procedure for simulating soft labels.
        \item Next, it presents the two moons dataset as an illustrative example to visually explain the behavior of our method.
        \item We report additional benchmark results and detailed hyperparameter studies for different annotation budgets.
        \item We investigate how performance behaves when labeling datasets of increasing size with a fixed annotation budget. Here, we also focus on the overall received information depending on the choice of the spreading intensity.
        \item We analyze how performance generalizes from annotated to unannotated data, how annotation budgets should be distributed onto data points and how performance can be estimated if soft labels are unavailable.
        \item We also provide an in-depth comparison of dimensionality reduction techniques and demonstrate how embedding quality impacts performance.
        \item We provide details on confidence interval construction.
        \item Finally, we elaborate on hardware specifications and runtime of experiments.
    \end{itemize}
\item
    \textbf{Section \ref{sec-appdx:detailed-description-of-all-algorithms}} gives a detailed, self-contained description of the proposed probabilistic label spreading algorithm and all baselines. Beyond pseudocode, this section explains the interpretation of propagated feedback, clarifies design choices, and documents the numerical solvers and GPU-accelerated implementations used. This level of detail is intended to ensure full reproducibility and to facilitate adoption or extension by other researchers.
\item
    \textbf{Section \ref{sec-appendix:consistency-proof}} provides a rigorous and self-contained theoretical analysis of our method, proving theorem 1 of the main manuscript. We establish consistency results for label spreading estimators under plausible assumptions such as smoothness. This way, we support the empirical success of our method through mathematical guarantees. The proof formalizes under which conditions probabilistic label spreading accurately estimates the underlying soft labels as the amount of data and annotation budget increase. Based on the findings, we also derive epistemic uncertainty estimates, i.e., confidence intervals.
\end{itemize}

\subsection{Additional Numerical Results}
\label{sec-appdx:additional-numerical-results}
\subsubsection{Details on Datasets and Simulated Soft Labels}

We conducted numerical experiments on the following datasets, considering only the training data (if such a split exists): Animals-10~\cite{animals10}, CIFAR-10~\cite{cifar10}, CIFAR-10-H~\cite{cifar10-H}, EMNIST-digits~\cite{EMNIST}, Tiny ImageNet~\cite{tinyimagenet} and a proprietary dataset that consists of street signs and multiple crowdsourced annotations per image, which we call the Mapillary Traffic Sign Dataset (MTSD). All datasets represent image classification problems. The MTSD dataset comprises $7$ different classes regarding the shape of the street signs and Tiny ImageNet consists of $200$ different classes from the ImageNet dataset for each of which there are $500$ examples of $64\times 64$ images in the training data. All other datasets contain $10$ different classes. The main properties of the datasets are listed in table \ref{tab:dataset_overview}.
\begin{table}[tb]
    \caption{Overview of the datasets under consideration}
    \label{tab:dataset_overview}
    \centering
    \begin{tabular}{lrrc}
    \toprule
       Dataset  & \# classes & \# images & soft labels \\
    \midrule
       Animals-10  & 10 & 26,179 & \xmark \\
       CIFAR-10  & 10 & 50,000 & \xmark \\
       CIFAR-10-H  & 10 & 10,000 & \checkmark \\
       EMNIST-digits  & 10 & 240,000 & \xmark \\
       TinyImageNet & 200 & 100,000 & \xmark \\
       MTSD  & 7 & 59,559 & \checkmark \\
       \bottomrule
    \end{tabular}
\end{table}

Except for the CIFAR-10-H and the MTSD dataset, only clear-cut labels are provided for the considered datasets. Thus, we simulate soft labels by considering the predicted class probabilities of an image classifier. The choice of a particular image classifier induces a trade-off between the accuracy of the predictions and the associated entropy in the softmax output. If the image classifier achieves high accuracy, the entropy is small on average and we approach the case of clear-cut labels. Conversely, if the image classifier achieves only a low accuracy, the entropy of the softmax output tends to be rather high, leading to unrealistically uncertain soft labels. For the task of simulating soft labels, we chose EfficientNet-B0 \cite{efficientnet} as a pretrained model which is fine-tuned on each dataset in a $k$-fold cross-validation setting where we chose $k=3$ and trained each model for $5$ epochs with a batch size of $64$ and a learning rate of $10^{-4}$. The final set of soft labels is composed of the softmax outputs on the validation sets of all cross-validation splits.
The distribution of normalized entropy of the resulting soft labels, together with the classification accuracies are displayed in figure \ref{fig:entropy_of_prob_labels}. 
Note that the maximal entropy is attained by the uniform distribution and has the value $\log C$, where $C\in\NN$ denotes the number of classes present in the dataset. Dividing the entropy of a soft label by $\log C$ yields a value in the unit interval, which we refer to as normalized entropy. It should be noted that the entropy distribution depends heavily on the dataset. Image datasets that are easier to classify (such as EMNIST-digits) exhibit predictions of low entropy. On the other hand, more difficult image classification tasks can be linked to more uncertain predictions and a lower classification accuracy, e.g., the Tiny ImageNet dataset. Noteworthy is also that the EfficientNet-B0 model generates soft labels with lower entropy on CIFAR-10 when compared to the crowdsourced soft labels one CIFAR-10-H.
This may be due to the fact that the resolution of the crowdsourced soft labels is limited by the number of annotations per image and deviations of only a few annotators from the majority voted class already influence the entropy noticeably.
\begin{figure}[tb]
    \centering
    \includegraphics[width=\linewidth]{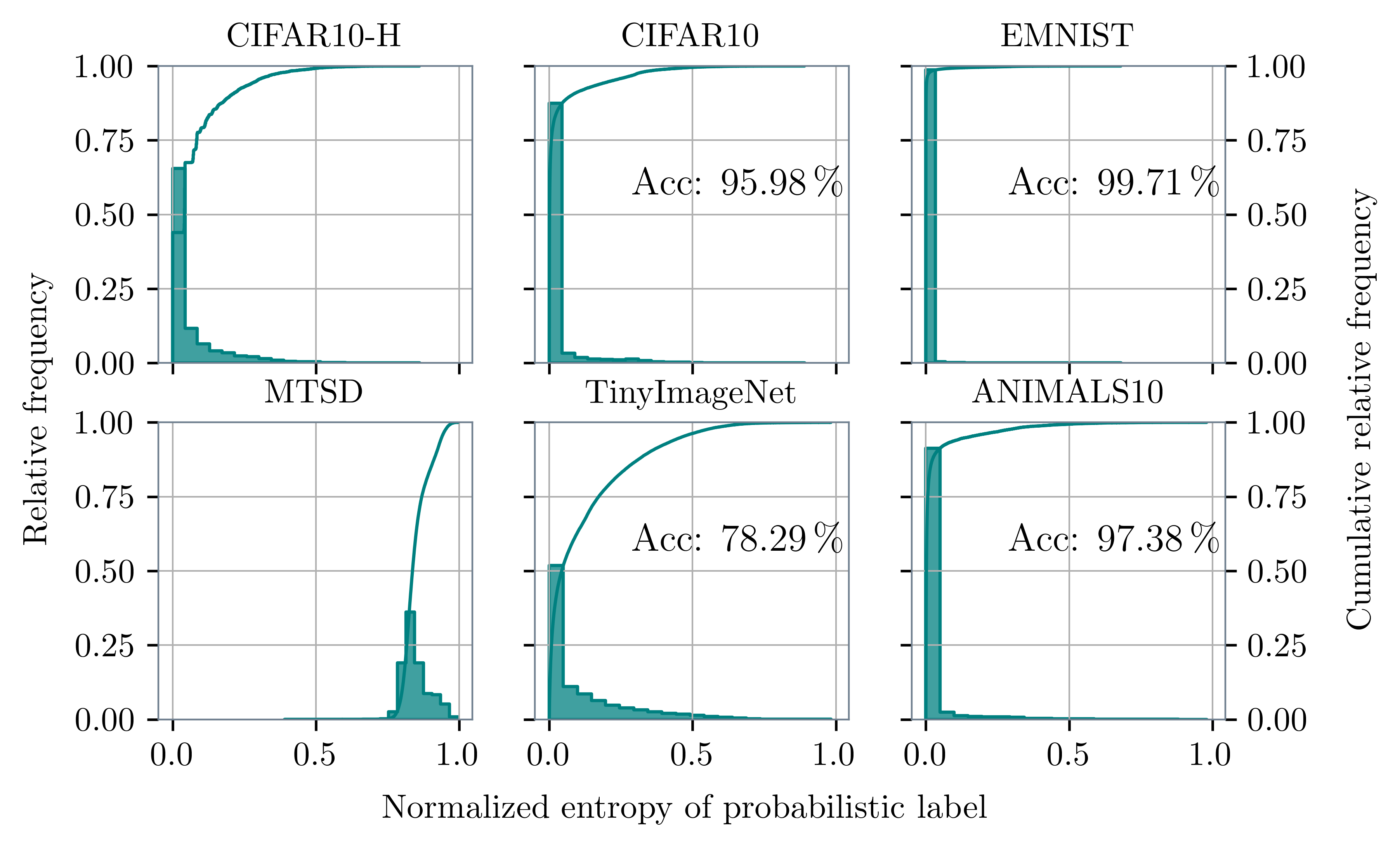}
    \caption{Distribution of the normalized entropy in the soft labels of each dataset. The histogram's values corresponds to the left y-axis and the line plot indicates the empirical cumulative distribution function using the right y-axis.}
    \label{fig:entropy_of_prob_labels}
\end{figure}

We created a probabilistic two moons dataset in the following way. We generate $1,000$ two-dimensional features using the \emph{make\_moons} function from the scikit-learn library~\cite{scikit-learn} where we specify a noise level of $0.1$. Soft labels are generated with a simple feedforward neural network consisting of two linear hidden layers of size $20$ and ReLU activations. The resulting simulated data with its color-coded ground truth is displayed in figure \ref{fig:gt_two_moons}.
\begin{figure}[tb]
    \centering
    \includegraphics[width=0.75\linewidth]{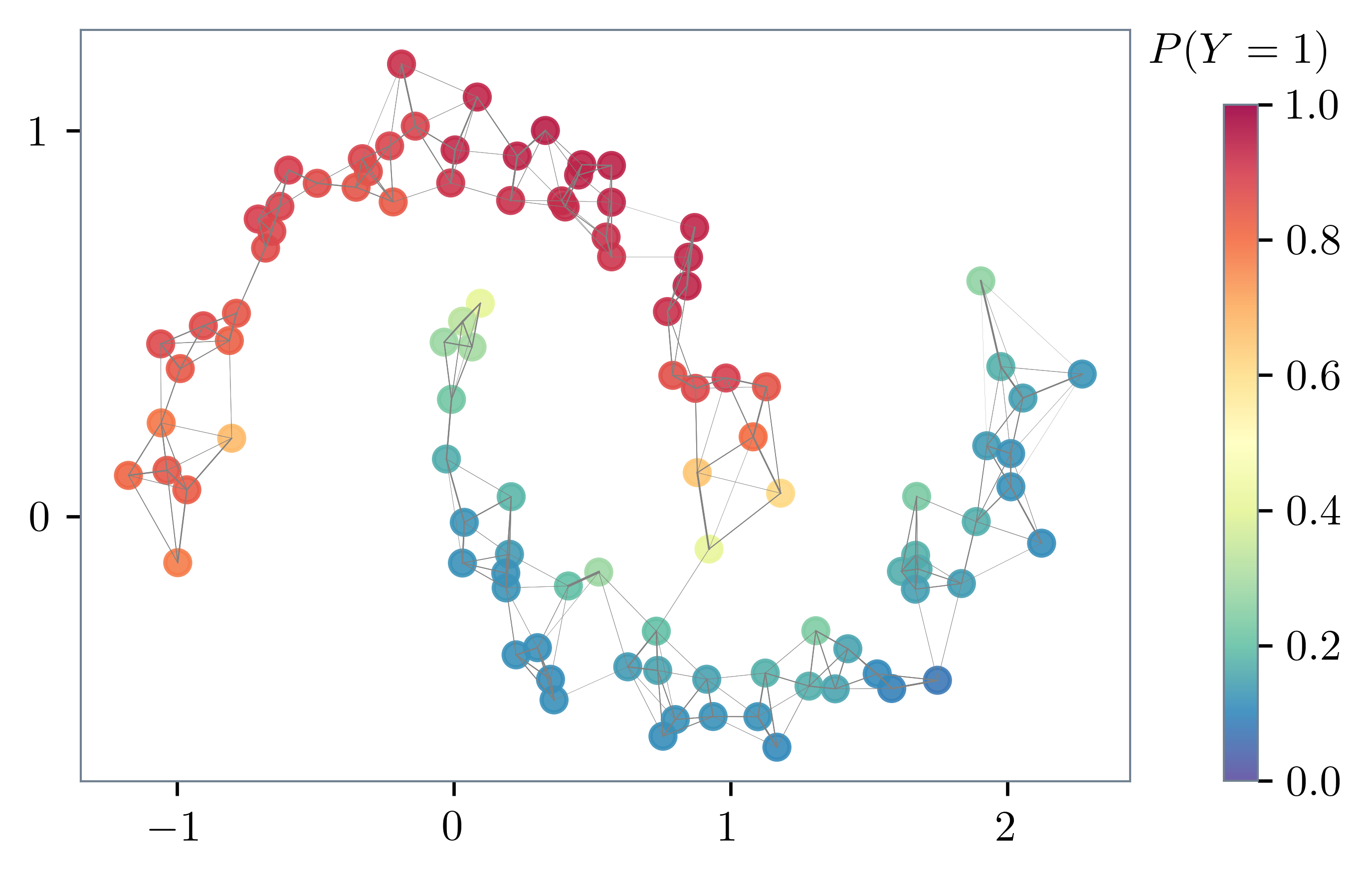}
    \caption{Distribution of soft labels in the two moons dataset. To improve readability, a random subset of $10\,\%$ of the data is displayed.}
    \label{fig:gt_two_moons}
\end{figure}

\subsubsection{Illustration on the Two Moons Dataset}

The simulated data of two moons is well-suited to illustrate the intended behavior of our method. In figure \ref{fig:two_moons_illustration}, the estimated soft labels are displayed for different numbers of provided annotations. The incremental improvement of the estimate with each annotation can be observed. We employed the PLS algorithm with $k=5$ and $\alpha=0.9$ and considered a dataset of size $100$ for visualization purposes. The comparison with the simulated ground-truth soft labels in figure \ref{fig:gt_two_moons} indicates that the estimate is still off from the soft ground truth. More annotations are required to further improve the estimated soft labels. In contrast, the simpler classification task seems to be solved quite well already with very limited label information available.
\begin{figure}[tb]
    \centering
    \includegraphics[width=\linewidth]{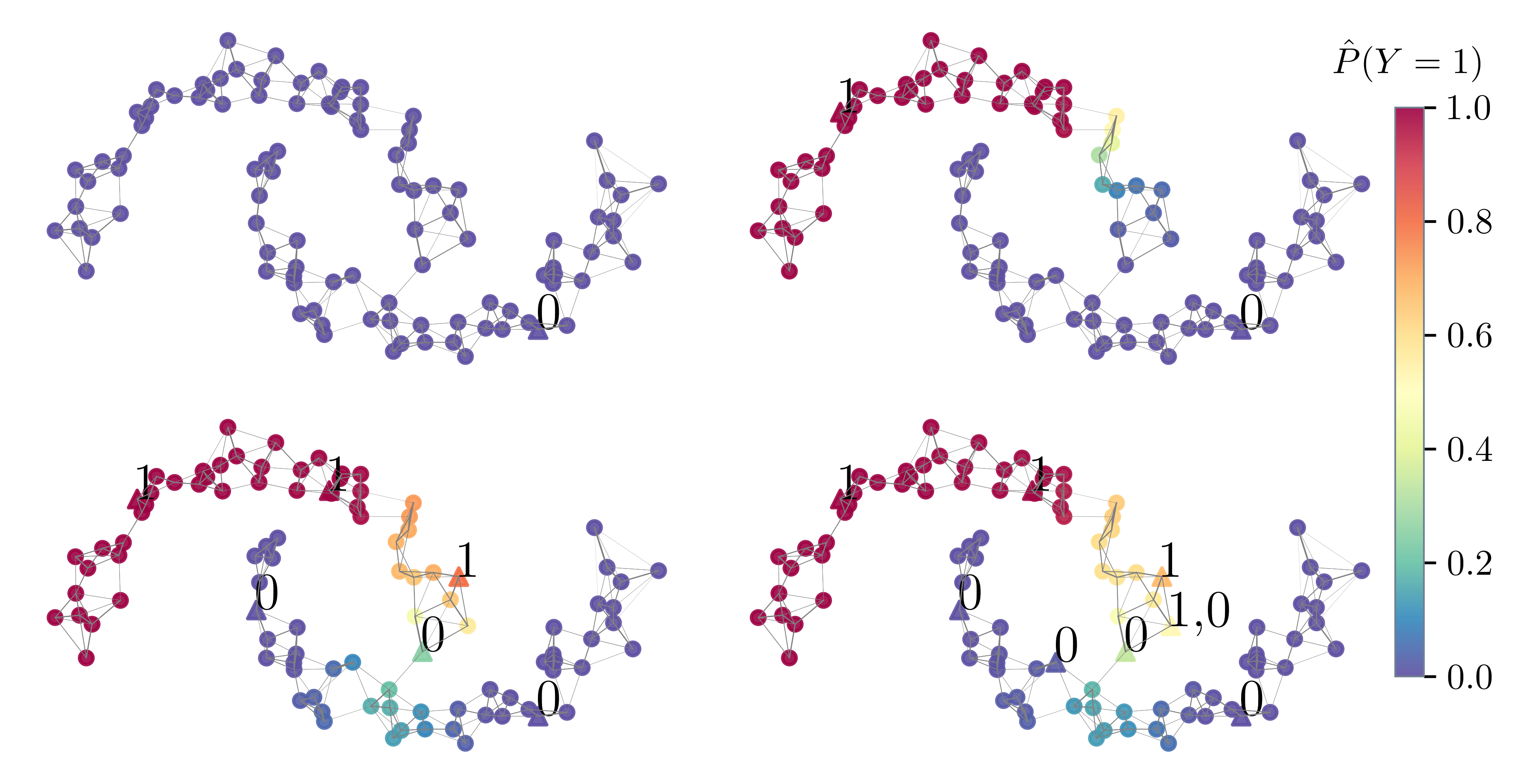}
    \caption{Estimated soft labels for an increasing number of feedbacks on the two moons dataset. Numbers indicate the annotator feedbacks obtained through a simulated crowd-sourcing query for the corresponding data points marked with a triangle. The color of all data points corresponds to their predicted soft label.}
    \label{fig:two_moons_illustration}
\end{figure}

\subsubsection{Performance Comparison on Different Datasets}
Similar to the performance comparison reported for a budget of $10\,\%$ in the numerical results section of the main manuscript, we display the results for a smaller budget of $1\,\%$ in table \ref{tab:RMSE_all_datasets_1_percent}.
The NaN entries for the kNN baseline on the two moons dataset result from an insufficient number of annotator feedbacks ($<k$), which are required to define a $k$ nearest neighbor regression. More specifically, the dataset contains $1,000$ data points, resulting in $10$ feedbacks for this budget. Hence, for $k>10$, the $k$-NN baseline is not available. We observe large standard deviations in RMSE values for the two moons dataset. This reflects that it matters how the $10$ annotator feedbacks are distributed among the features. As discussed in the main manuscript, we observe that the GKR baseline performs well in the regime of very small numbers of annotator feedbacks and that the CLIP zero shot performance is hard to surpass with this very limited label information. 
\renewcommand{\rot}[1]{#1} 
\begin{table*}[tb]
\caption{RMSE for different datasets and algorithms with 1\,\% budget. Mean RMSE and standard deviation across $10$ evaluations are reported in units of $\times 10^{-3}$. We indicate model hyperparameters $\alpha$ for PLS, $\gamma$ for GKR and $k$ for kNN. ($*$) denotes optimized parameters.}
\label{tab:RMSE_all_datasets_1_percent}
\centering
\scalebox{0.85}{
\begin{tabularx}{\linewidth}{llllllll} 
\toprule
\multirow{2}{*}{} & \multicolumn{7}{c}{\textbf{Datasets}} \\
\cmidrule(lr){2-8}
\textbf{Algorithms} & \rot{TwoMoons} & \rot{Anim.10} & \rot{EMNIST-digits} & \rot{CIF.10} & \rot{CIF.10-H} & \rot{Tiny I.N.} & \rot{MTSD} \\
\midrule
PLS (0.5) & $286.7 (31.8)$ & $125.8 (3.1)$ & $86.1 (1.1)$ & $150.4 (4.4)$ & $172.1 (1.0)$ & $59.5 (0.3)$ & $241.7 (2.8)$ \\
PLS (0.9) & $196.0 (84.5)$ & $81.3 (6.1)$ & $62.5 (0.9)$ & $124.9 (0.4)$ & $141.4 (2.3)$ & $53.9 (0.4)$ & $191.8 (2.3)$ \\
PLS (0.99) & $141.0 (98.2)$ & $70.8 (1.6)$ & $58.0 (0.2)$ & $116.2 (0.7)$ & $124.5 (5.2)$ & $49.7 (0.3)$ & $121.2 (3.4)$ \\
PLS ($*$) & \textbf{\boldmath$111.8$}$\,(61.2)$ & $67.6 (1.5)$ & \textbf{\boldmath$57.9$}$\,(0.4)$ & $115.5 (1.1)$ & $127.5 (4.8)$ & $49.5 (0.3)$ & $98.4 (1.3)$ \\
\midrule
GKR (0.1) & $475.1\, (21.5)$ & $77.7\, (1.4)$ & $72.4\, (0.4)$ & $171.8\, (2.5)$ & $182.3\, (4.5)$ & $60.5\, (0.0)$ & $97.5\, (0.5)$ \\
GKR (1) & $372.3\, (79.6)$ & $70.1\, (1.4)$ & $61.3\, (0.1)$ & $116.9\, (0.9)$ & $126.6\, (4.8)$ & $54.6\, (0.2)$ & $106.3\, (4.0)$ \\
GKR (10) & $297.7\, (97.5)$ & $80.5\, (7.7)$ & $62.2\, (0.7)$ & $127.4\, (2.9)$ & $153.8\, (11.1)$ & $53.2\, (0.5)$ & $206.9\, (3.0)$ \\
GKR ($*$) & $266.7\, (120.8)$ & $69.2\, (2.0)$ & $58.5\, (0.4)$ & \textbf{\boldmath$114.2$}$\,(0.7)$ & $131.0\, (4.1)$ & $51.3\, (0.3)$ & \textbf{\boldmath$94.9$}$\,(1.7)$ \\
\midrule
kNN (5) & $433.7 \,(62.4)$ & $72.0 \,(2.1)$ & $65.0 \,(0.6)$ & $123.3 \,(1.7)$ & $134.7 \,(8.0)$ & $54.6 \,(0.4)$ & $172.0 \,(2.8)$ \\
kNN (20) & nan & $85.1 \,(8.5)$ & $62.1 \,(0.2)$ & $117.9 \,(1.1)$ & $216.3 \,(4.7)$ & $58.0 \,(0.2)$ & $114.5 \,(1.9)$ \\
kNN (50) & nan & $188.5 \,(6.1)$ & $61.7 \,(0.2)$ & $127.6 \,(1.6)$ & $270.1 \,(2.7)$ & $59.8 \,(0.1)$ & $100.1 \,(1.4)$ \\
kNN ($*$) & $323.9 \,(129.5)$ & $68.9 \,(2.9)$ & $58.8 \,(0.1)$ & $116.1 \,(1.1)$ & $133.6 \,(3.7)$ & $54.7 \,(0.2)$ & $97.2 \,(1.1)$ \\
\midrule
CLIP 0-shot & nan & \textbf{\boldmath$65.4$} & $288.6$ & $126.5$ & \textbf{\boldmath$110.5$} & \textbf{\boldmath$43.9$} & $164.0$ \\
\bottomrule
\end{tabularx}
}
\end{table*}

To optimize the parameters of each algorithm we used the SMAC framework for Bayesian optimization \cite{SMAC_optimization_framework}. We allowed for 200 evaluations in total for each algorithm and specified the following parameter ranges. For the GKR baseline, we set $\gamma\in(10^{-6}, 10^6)$, for the kNN baseline we specified that $k$ must not exceed the number of annotator feedbacks and we specified $\alpha\in(0,1)$ and $k\leq 100$ for the PLS algorithm where we set $\alpha\sim Beta(5,1)$ as a prior distribution as we expect larger $\alpha$ values to be beneficial.

\subsubsection{Hyperparameter Analysis for Different Budgets}
Compared to the comparison of a few performance trajectories in the main manuscript, \autoref{fig: More extensive hyperparameter study} depicts a more extensive hyperparameter study as a heatmap visualization. It includes additional values for the spreading intensity $\alpha$ in (a) and demonstrates that the choice of a suitable spreading intensity depends on the annotation budget.  
\begin{figure*}[tb]
    \centering
  
  \begin{minipage}{\textwidth}
    \centering
    
  \begin{subfigure}[t]{0.48\linewidth}
    \includegraphics[width=0.9\linewidth]{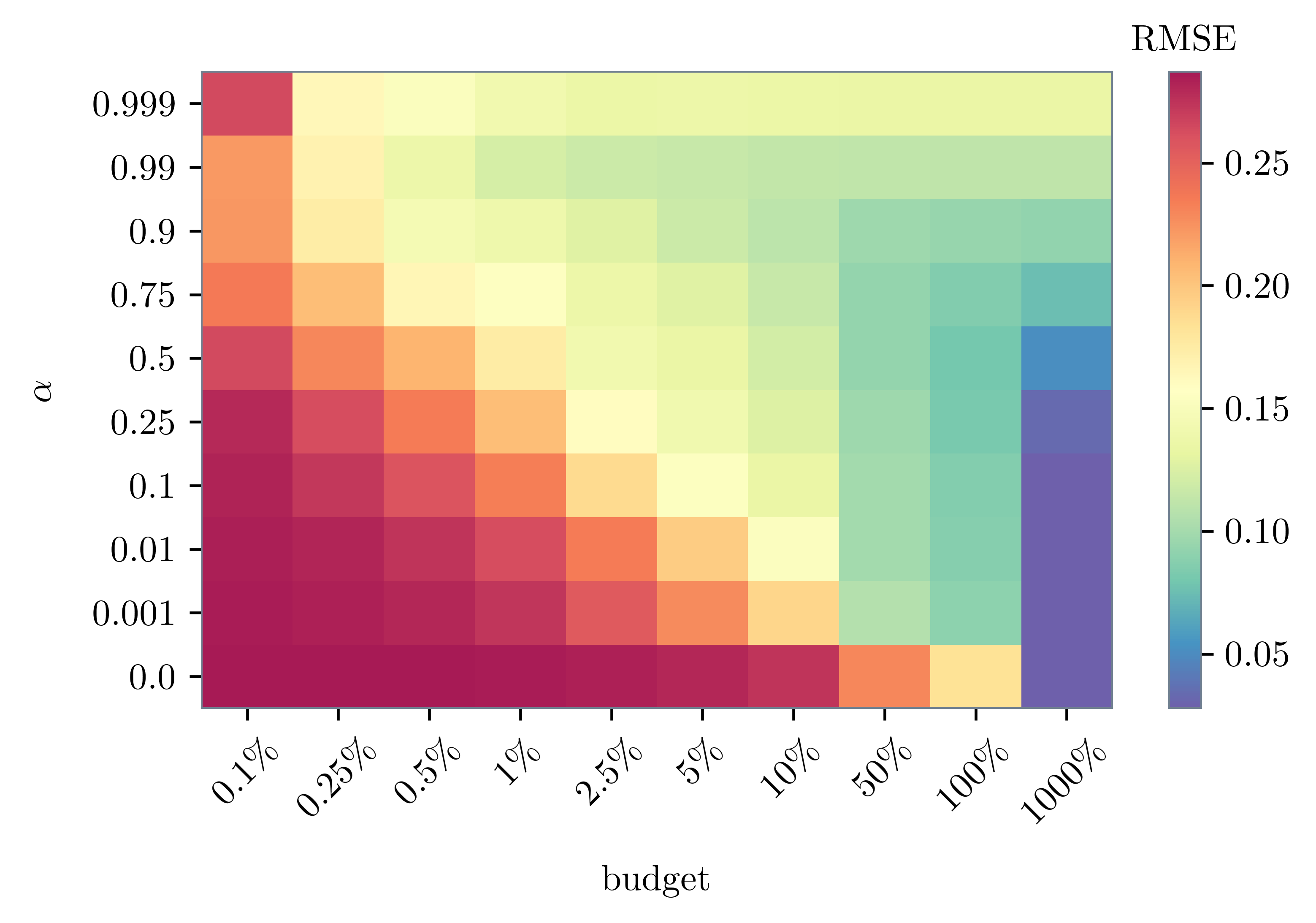}
    \caption{Performance of our method for different spreading intensities $\alpha$ and different annotation budgets.}
    \label{fig:alpha_heatmap}
  \end{subfigure}
  \hfill
  \begin{subfigure}[t]{0.48\linewidth}
    \includegraphics[width=\linewidth]{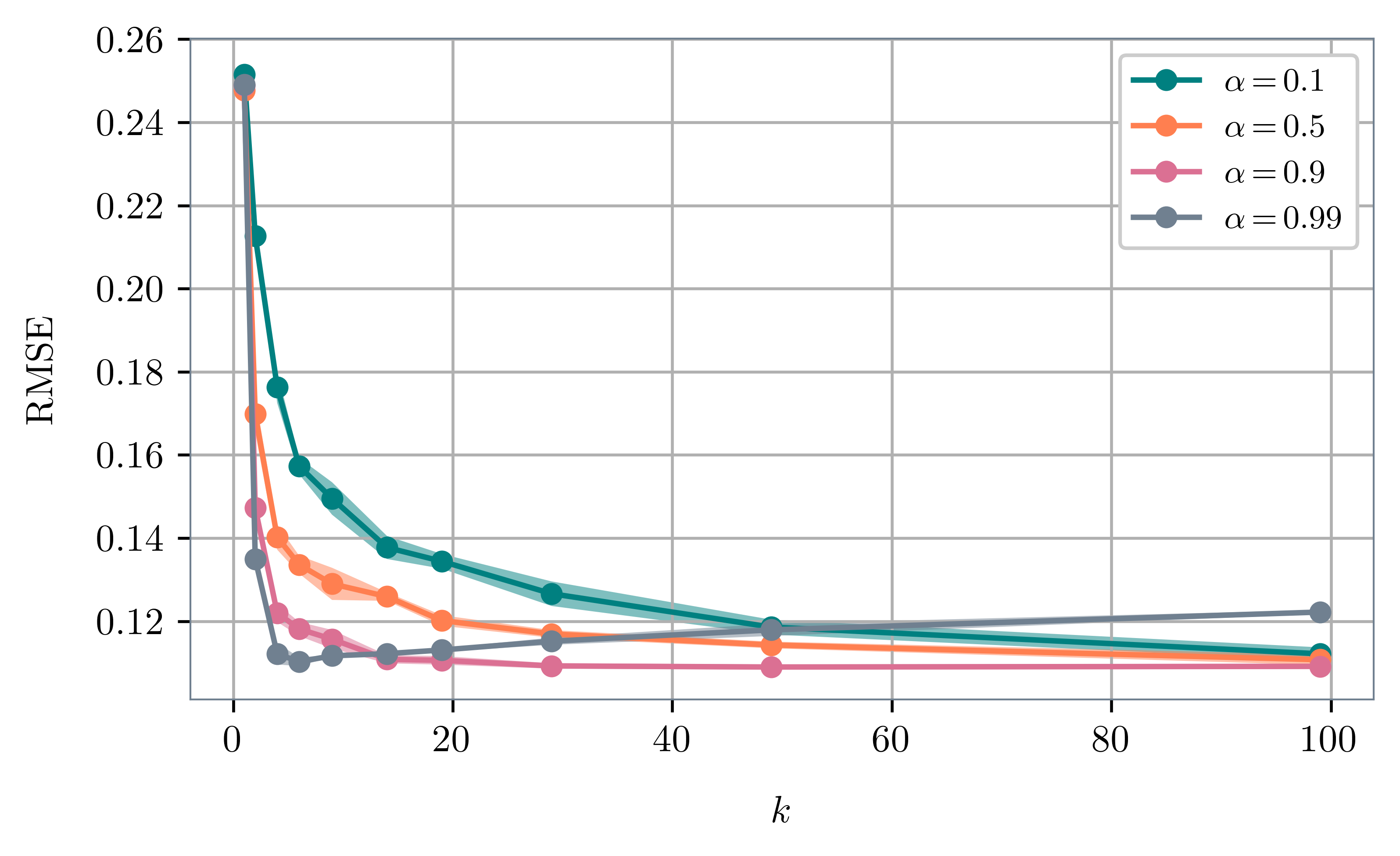}
    \caption{Performance of our method for a varying number of nearest neighbors $k$ and a fixed budget.}
    \label{fig:performance_depending_on_k}
  \end{subfigure}
  \caption{Performance of our method depending on the hyperparameters $\alpha$ (a) and $k$ (b) on the CIFAR-10-H dataset. Mean and additionally standard deviation in (b) of the RMSE over $10$ runs is displayed.}
  \label{fig: More extensive hyperparameter study}

  \end{minipage}
\end{figure*}

We conducted a hyperparameter grid search for three different annotation budgets. This analysis extends the findings of the main manuscript and supports the conclusions regarding the influence of the hyperparameters $\alpha$ and $k$. In particular, in figure \ref{fig:hyperparameter_analysis_for_different_budgets} it can be seen once more that for larger annotation budgets, smaller $\alpha$ values lead to improved performance. Hence, the choice of hyperparameters should take the annotation budget into account. 
\begin{figure*}[tb]
  \centering
  \begin{subfigure}[t]{0.325\linewidth}
    \includegraphics[width=\linewidth]{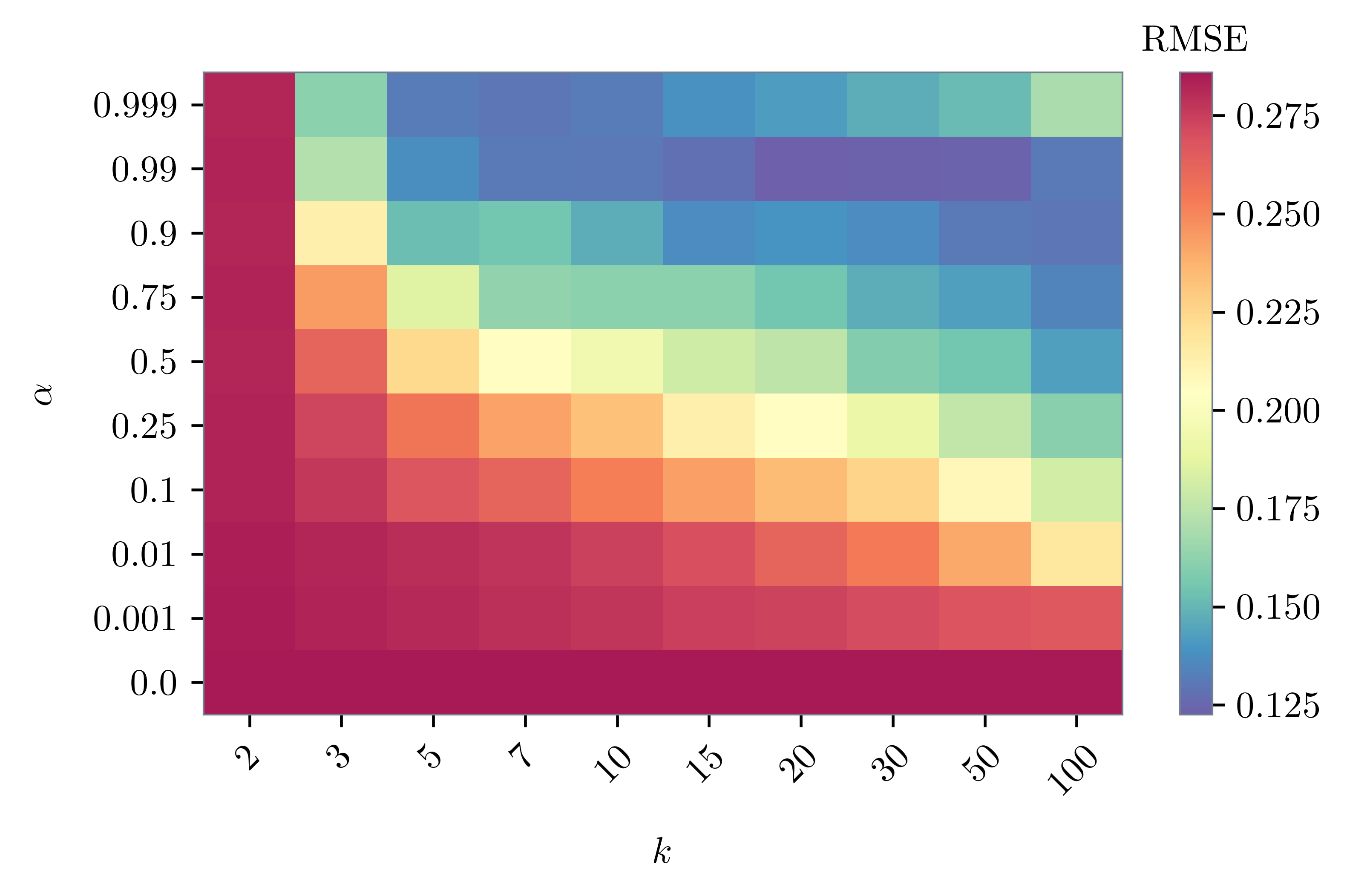}
    \caption{Budget of $1\,\%$}
    \label{fig:alpha_k_heatmap_1}
  \end{subfigure}
  \hfill
  \begin{subfigure}[t]{0.325\linewidth}
    \includegraphics[width=\linewidth]{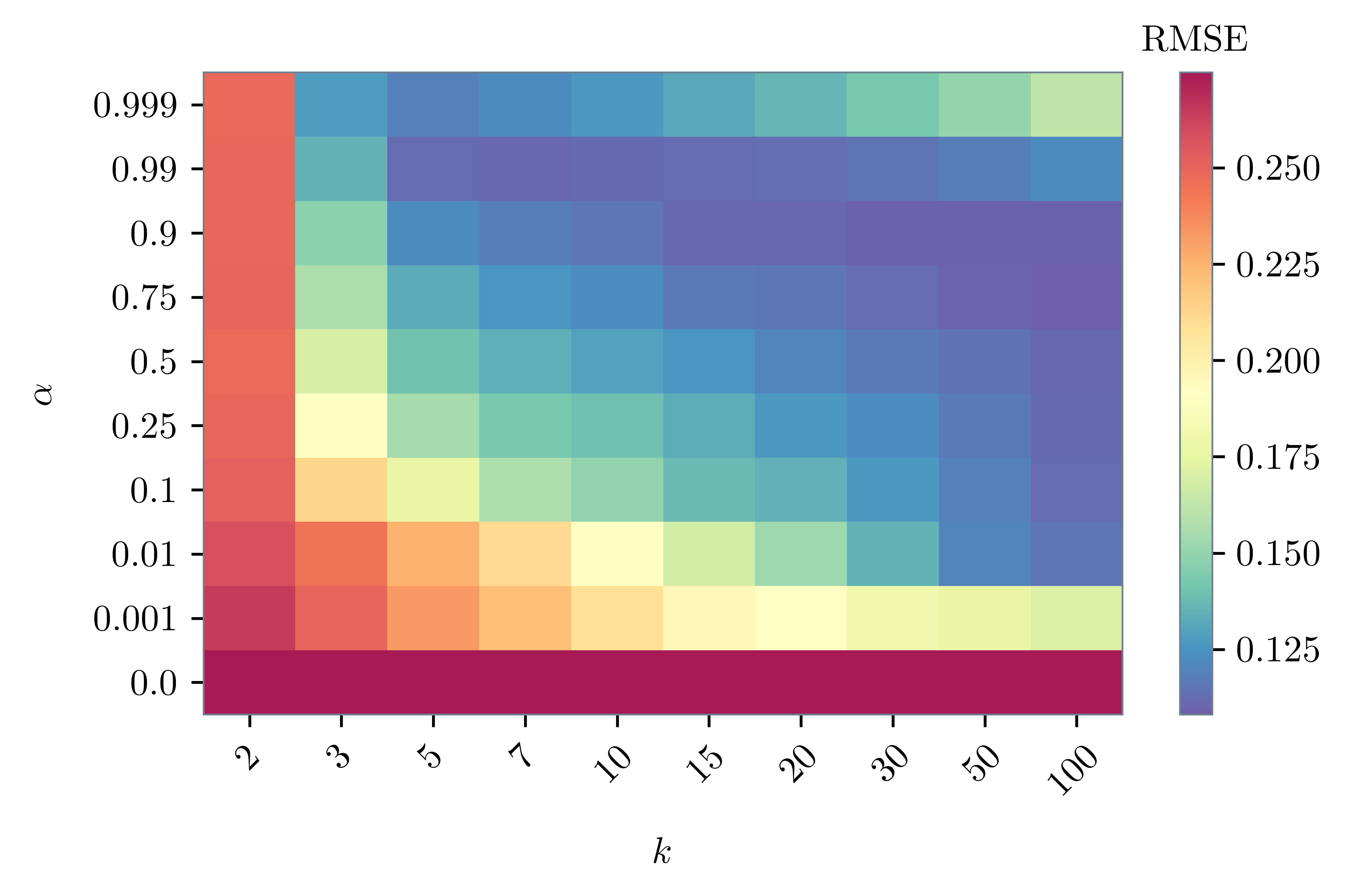}
    \caption{Budget of $10\,\%$}
    \label{fig:alpha_k_heatmap_10}
  \end{subfigure}
  \hfill
  \begin{subfigure}[t]{0.325\linewidth}
    \includegraphics[width=\linewidth]{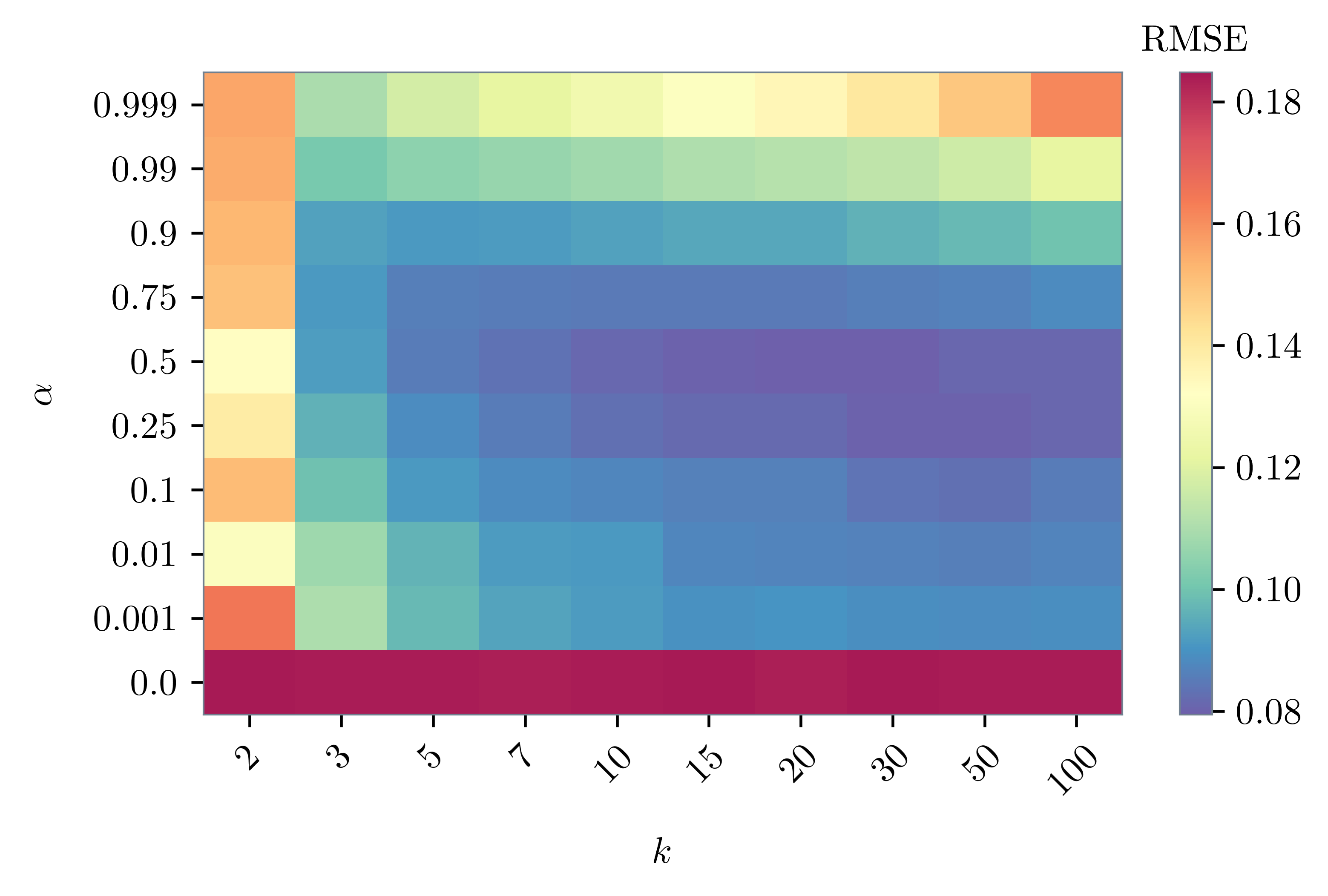}
    \caption{Budget of $100\,\%$}
    \label{fig:alpha_k_heatmap_100}
  \end{subfigure}
  \caption{Performance of our method for different spreading parameters $\alpha$, different number of neighbors $k$ as well as three different annotation budgets. For different combinations of the hyperparameters $\alpha$ and $k$, the resulting RMSE averaged over $10$ runs is displayed.}
  \label{fig:hyperparameter_analysis_for_different_budgets}
\end{figure*}

\subsubsection{Labeling Large Datasets with few Annotations}

Here, we consider the case of a fixed number of annotator feedbacks and determine the performance for different dataset sizes to address the question of how the number of data points in the dataset influences the performance.
We only vary the size of the dataset, i.e., we consider increasingly large subsets of the CIFAR-10 dataset and provide a fixed number of $100$ and $500$ annotator feedbacks. As depicted in figure \ref{fig:performance_depending_on_n_data}, the performance stagnates with increasing dataset size for sufficiently large values of $\alpha$ but decreases if $\alpha$ is chosen too small.
A high spreading intensity seems to be required to ensure, that the growing number of data points are provided with information by the few annotated data points.
It can be observed that the information received on average in terms of propagation scores decreases if the number of data points grows (see figure \ref{fig:average_received_feedback}). It can only be assumed that if the number of data points continues to grow, the performance might start to decline also for large values of $\alpha$. Note that to prevent numerical instabilities, we added a value of $10^{-4}$ to the propagation scores received prior to normalizing them. Consequently, if the received information approaches this value, the estimate will tend to become more uniform and performance decreases.
Overall, assuming that the data is well-structured, $\alpha$ is related to the amount of data points that are provided with sufficient information by a fixed amount of annotations and it might be beneficial to increase $\alpha$ if the annotation budget is small.
\begin{figure*}[tb]
  \centering
  \begin{subfigure}[t]{0.49\linewidth}
    \includegraphics[width=\linewidth]{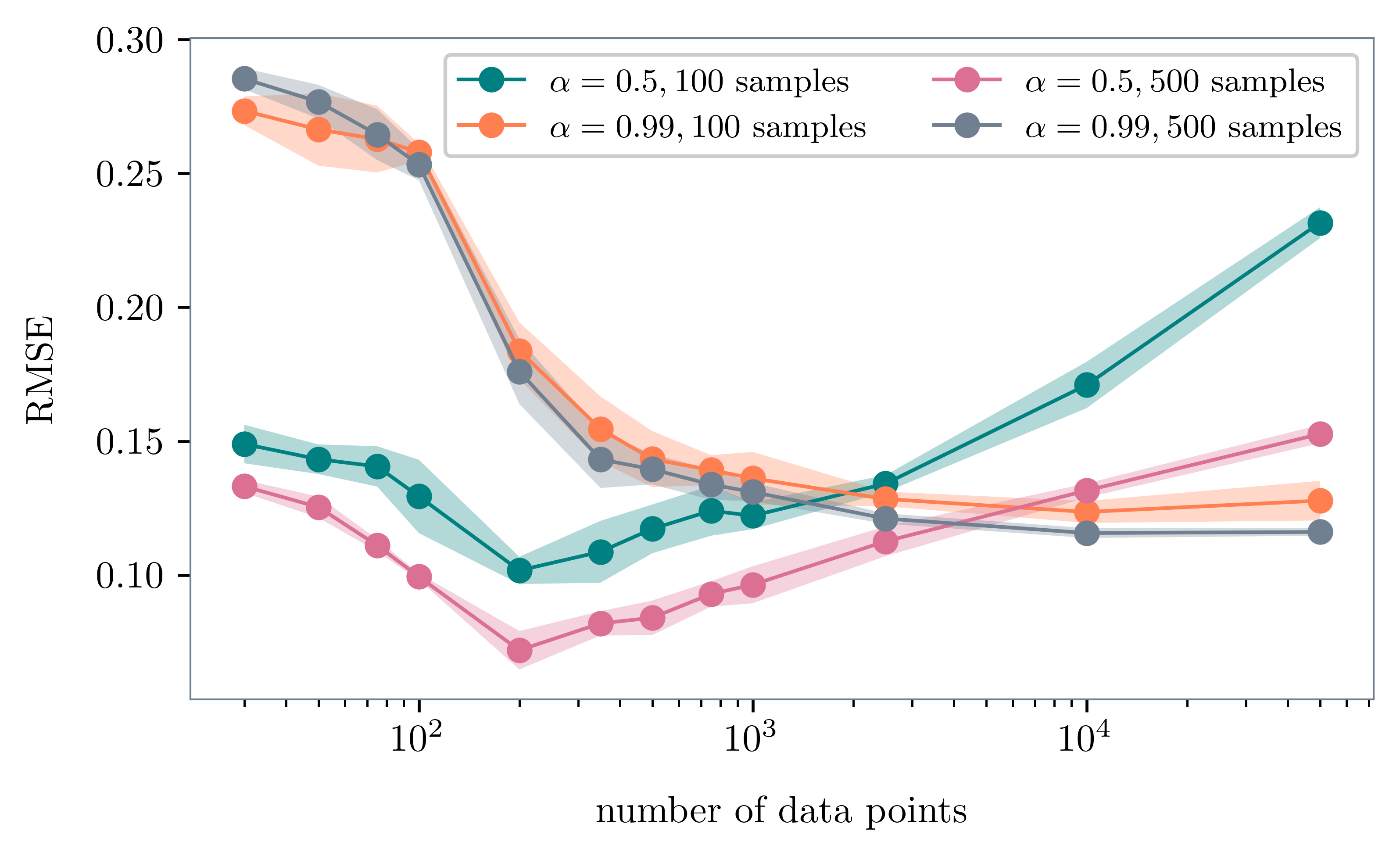}
    \caption{Performance with a fixed budget depending on the dataset size.}
    \label{fig:performance_depending_on_n_data}
  \end{subfigure}
  \hfill
  \begin{subfigure}[t]{0.49\linewidth}
    \includegraphics[width=\linewidth]{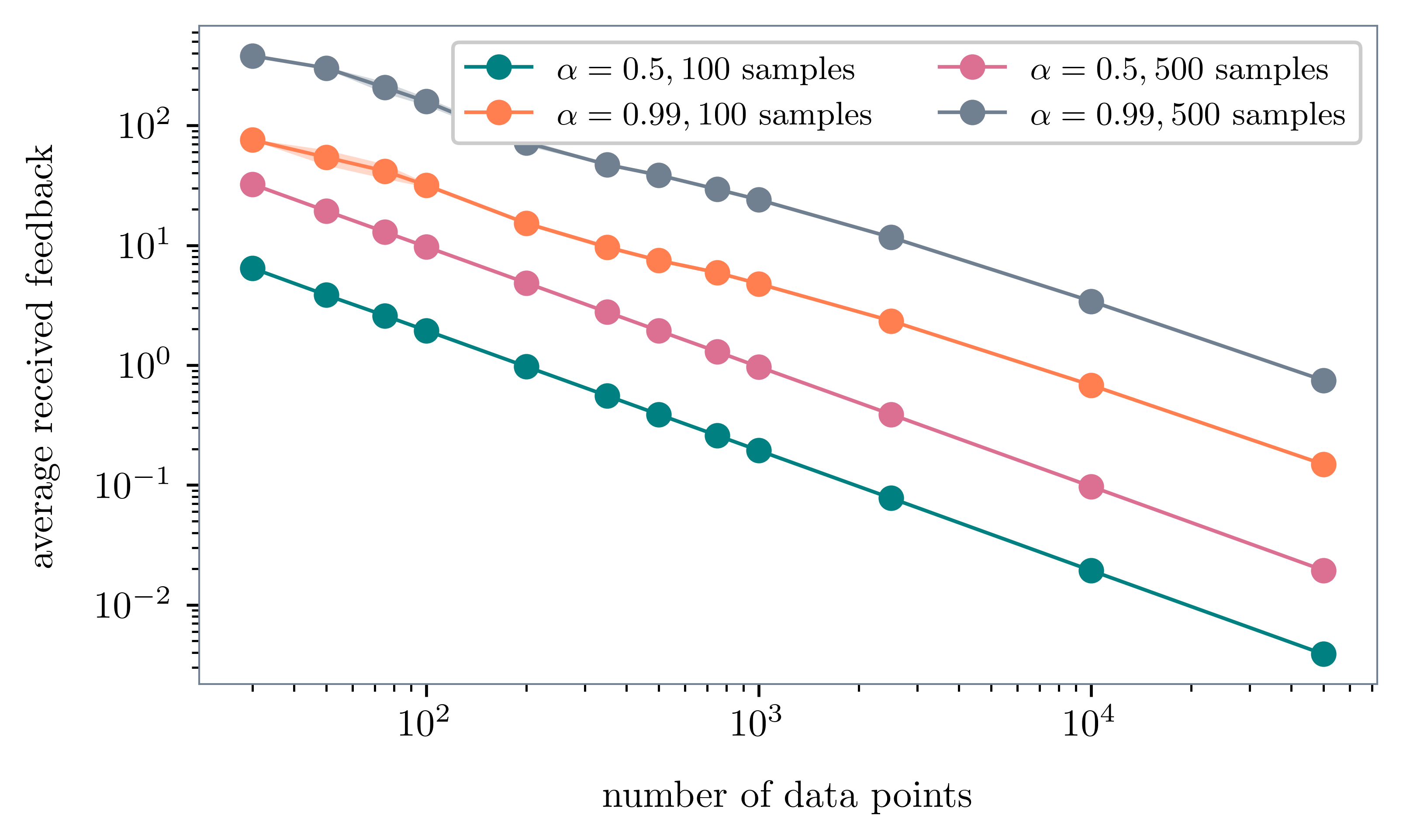}
    \caption{Average received feedback of all data points depending on the dataset size.}
    \label{fig:average_received_feedback}
  \end{subfigure}
  \caption{Performance of our method with a fixed budget depending on the size of a subset of the CIFAR-10 dataset under consideration. All other parameters are fixed here and we consider two different values of $\alpha$ and two different budgets. We observe that performance decreases with a growing number of data points for $\alpha = 0.5$ whereas this is not the case for $\alpha = 0.99$.
  }
  \label{fig:dataset_size_plot}
\end{figure*}

\subsubsection{Received Feedback Depending on Spreading Intensity}
In the PLS algorithm, the propagation scores represented by the entries of $(I-\alpha S)^{-1}$ are normalized row-wise by the maximal row-wise entry in $(I-\alpha S)^{-1}$. 
The maximal value corresponds to the data point that was annotated so that it receives a propagation score of $1$ and the other data points receive a smaller value depending on the choice of $\alpha$. Hence, we can interpret the cumulative propagation score a data point received as the number of ``virtual'' crowd-sourcing experiments that were conducted in neighboring points. Here, we are interested in the dependence of this value on the spreading intensity $\alpha$. Considering, the CIFAR-10-H dataset, we provide the algorithm with an annotation budget of $10\,\%$ and evaluate the average cumulative propagation score for different $\alpha$ values. The resulting score can be interpreted as the average number of virtual experiments conducted for each data point. Given that this score can also be determined for each class individually it can be used for uncertainty quantification such as constructing confidence intervals for a multinomial distribution for instance. Figure \ref{fig:feedback_depending_on_alpha} shows that the average cumulative propagation score is close to $0.1$ for very small values of $\alpha$, which reflects the fact that the specified annotation budget is $10\,\%$. The smallest eigenvalue of $(I-\alpha S)$ approaches zero like $1-\alpha$ as $\alpha \to 1$. Hence,  the largest eigenvalue $(I-\alpha S)^{-1}$ is of order $1/(1-\alpha)$ which explains the behavior of the observed growth of the cumulative annotator feedback after spreading.
If the graph consists of a single connected component, in the limit $\alpha\to 1$, the spread annotator feedbacks are bounded by total number of annotator feedbacks, here 1,000. However, in practice, $\alpha<1$ is required since in the limit case, the spread annotator feedback converges to the eigenvector corresponding to the smallest eigenvalue of $S$ independently of the data point that was provided with feedback.
\begin{figure}[tb]
    \centering
    \includegraphics[width=0.9\linewidth]{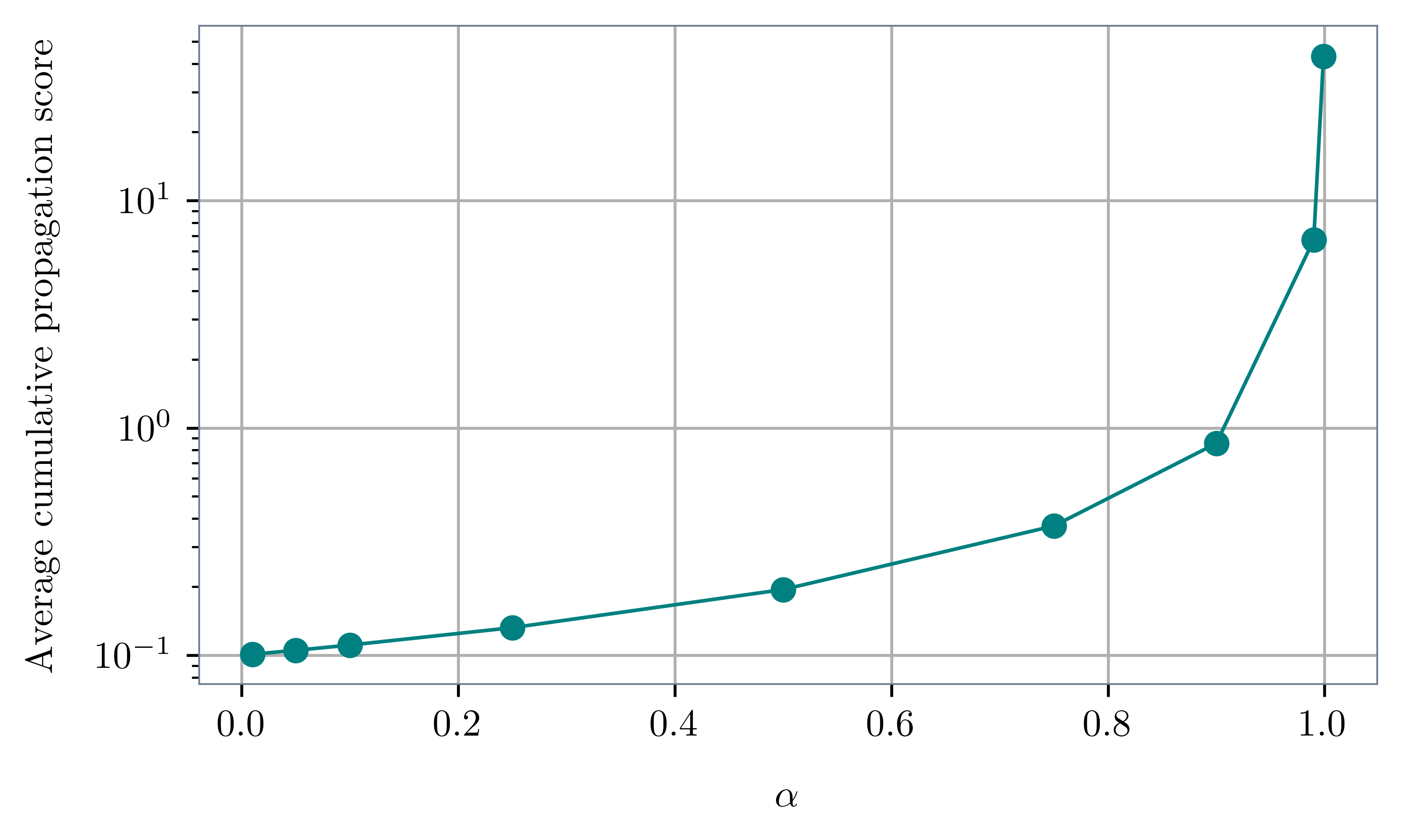}
    \caption{The average received propagation score obtained through PLS on the CIFAR-10-H dataset with an annotation budget of $10\,\%$, depending on the spreading parameter $\alpha$.}
    \label{fig:feedback_depending_on_alpha}
\end{figure}

\subsubsection{Performance on Training vs. Test Data} 
Here, we examine how well our method generalizes to data that has not received direct label information. We consider two sets of data, a training set that receives annotator feedbacks and a test set that does not receive any annotator feedbacks during crowdsourcing. We build a graph over the union of both sets, training and test, to apply our method. We conduct this experiment on the CIFAR-10-H dataset which we randomly split into $80\,\%$ training data and $20\,\%$ test data. 
For a small annotation budget, however, very few data points out the training data receive annotator feedbacks, such that the difference in RMSE between training and test data can be expected to be rather marginal. 
Indeed, figure \ref{fig:test_training_rmse} shows that for small annotation budgets, test and training RMSEs are well aligned. However, when increasing the budget, the RMSE on the test data stagnates at a certain level. This holds for a range of $\alpha$-values, all stagnating at a similar performance. This emphasizes that the annotator feedback should be distributed across all datapoints, rather than creating labeled and unlabeled subsets.
\begin{figure}[tb]
    \centering
    \includegraphics[width=\linewidth]{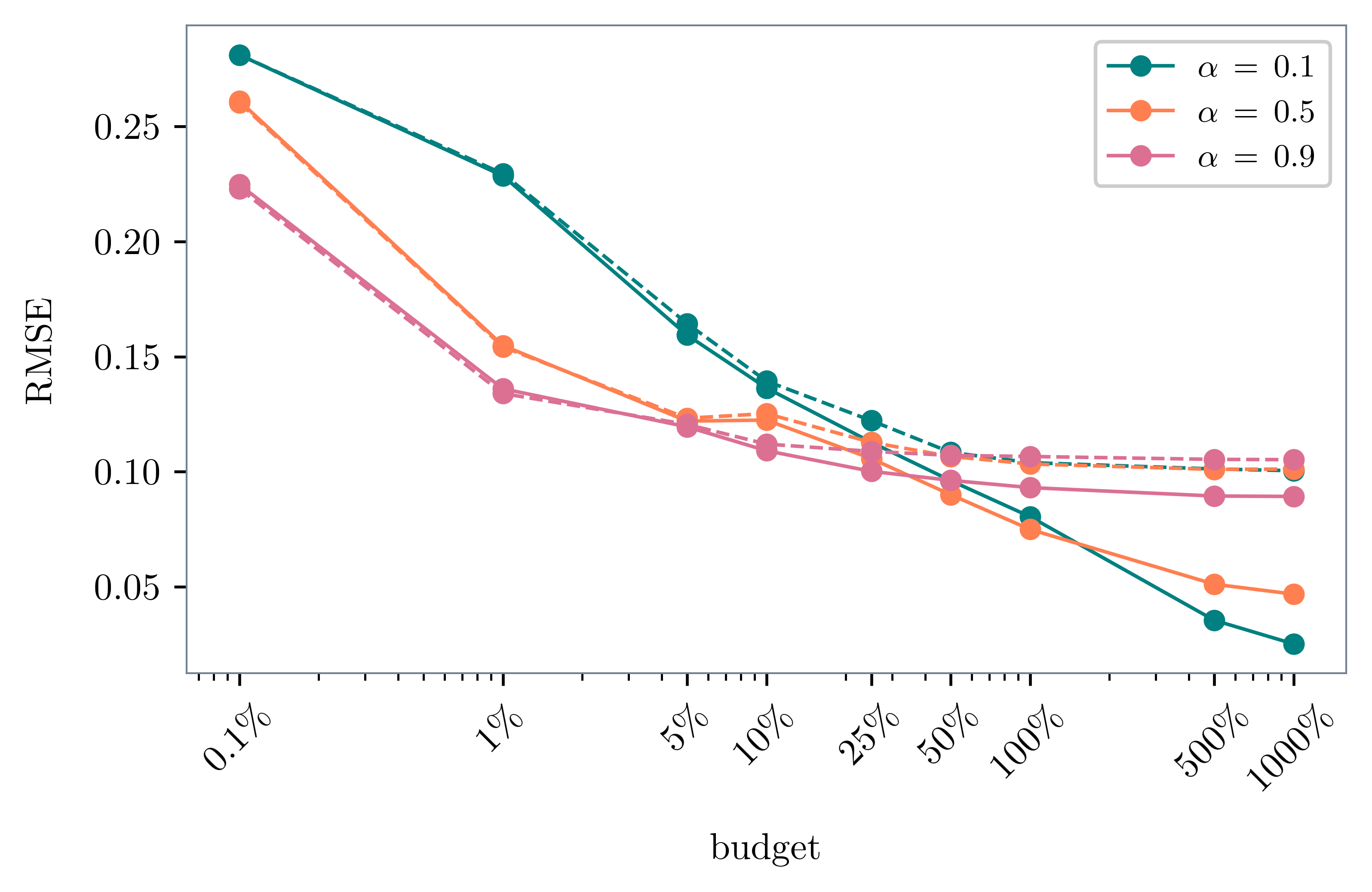}
    \caption{Performance of the PLS algorithm in terms of RMSE on the training data (solid lines) for which label information can be queried and on the test data (dashed lines) which receives exclusively indirect label information.
    }
    \label{fig:test_training_rmse}
\end{figure}

\subsubsection{Fewer high-quality labels or more but noisy labels?}

Motivated by the previous experiment, we now address the question whether it is beneficial to distribute the annotation budget on fewer datapoints in order to obtain more accurate estimates of their soft labels or not. To do so, we restrict the assignment of feedbacks to a subset of the data and vary the size of this subset.
In figure \ref{fig:subdata_for_labels}, we observe that a restriction of the annotator feedback to a subset of data is not beneficial. Indeed, this can be expected since each annotator feedback spreads individually across the graph. Thus, a uniform distribution across the graph should provide a decent coverage of the dataset. The effect of restricting the annotator feedbacks to a subset gets even more pronounced for small values of $\alpha$.
\begin{figure}[tb]
    \centering
    \includegraphics[width=\linewidth]{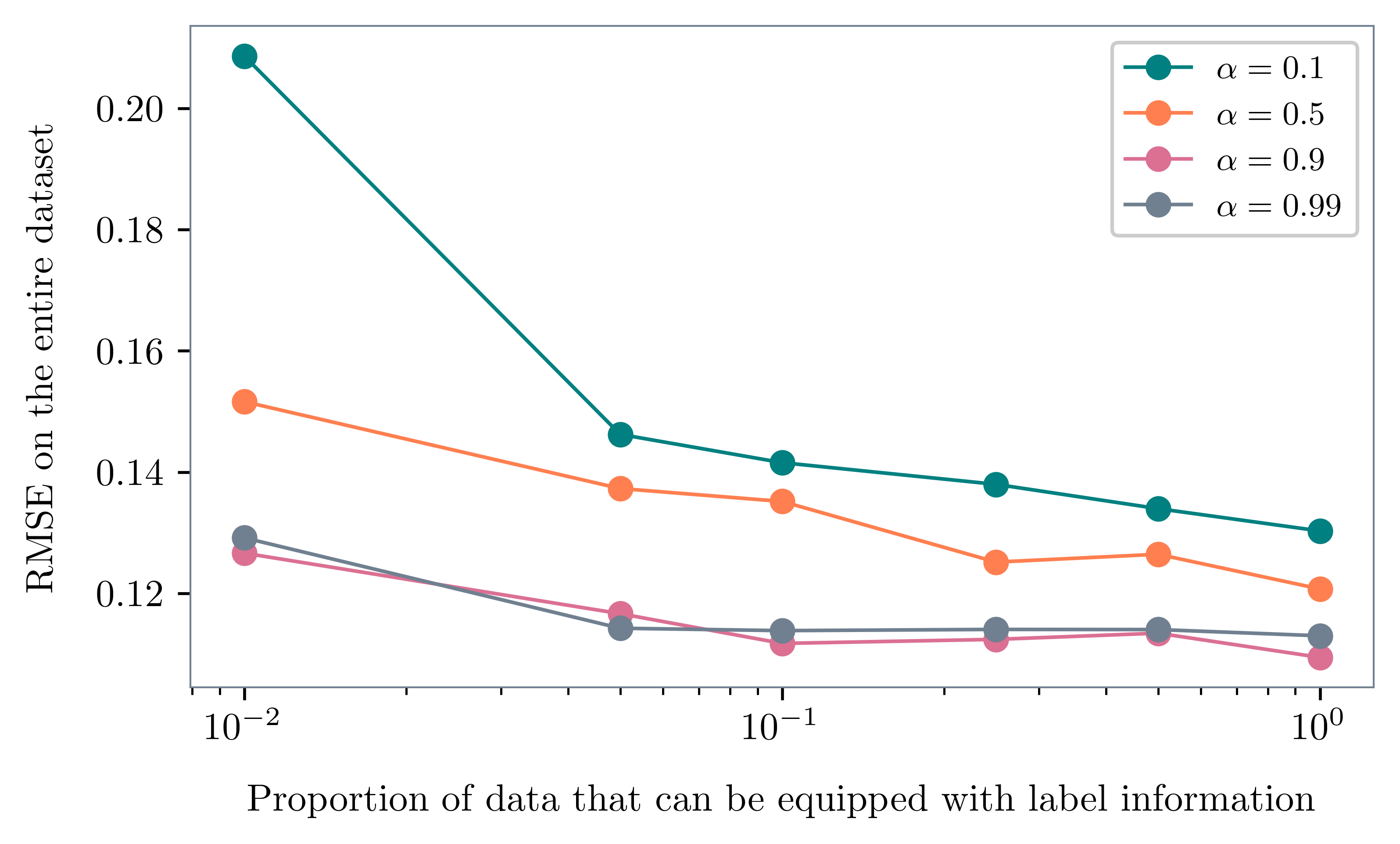}
    \caption{Performance of the PLS algorithm with a fixed budget of $10\,\%$ depending on the relative size of the sub-dataset. Annotator feedbacks are restricted to be assigned to this sub-dataset. It can be noted that overall, this restriction is not beneficial in terms of performance.} 
    \label{fig:subdata_for_labels}
\end{figure}

\subsubsection{Performance Estimation with an Annotated Dataset}

So far, the evaluation of our method relied on the knowledge of the soft ground truth, which we generated artificially. In practice, however, soft labels are unknown and often approximated via the relative frequencies obtained through multiple annotations. Hence, we consider the case where the performance of the algorithm in terms of RMSE should be estimated via a suitable test dataset equipped with soft labels based on multiple annotations. 
We study how large such a test dataset needs to be and how many annotations per data point are necessary to accurately estimate the performance of our PLS algorithm.
We consider the CIFAR-10-H dataset, employing a split into training and test data where the test dataset comprises $2,000$ images and the training dataset $8,000$ images. We provide the algorithm with $1,000$ annotator feedbacks on the training data and specify the parameters $\alpha = 0.9$ and $k =20$. As the graph is made up of both training and test data, we obtain predicted soft labels for the entire dataset. In the following, we try to estimate the algorithm's performance on the test dataset without access to the soft labels. 

Since the actual soft labels of the test data are unknown, we simulate the construction of a small dataset to estimate the performance in the following way. For subsets of the test dataset of varying sizes (10 up to 2,000) we assign different numbers of feedbacks per datapoint (1, 3, 5 and 10) to obtain estimated soft labels. Based on these soft labels, we estimate the performance in terms of RMSE. As the construction of such a dataset involves randomness, we repeat this procedure $10$ times and report the mean and standard deviation of the resulting RMSE estimates. Figure \ref{fig:estimated_rmse} depicts the estimated RMSE as a function of the number of annotator feedbacks and the number of datapoints in the test set.
A small number of annotator feedbacks per data point leads to noisy RMSE estimates, for small and for comparatively large test sets. Also larger datasets with few feedbacks per test datapoint can lead to biased RMSE estimates. This is to be expected since a small number of feedbacks consistently leads to noisy estimated soft labels. Hence, it turns out that a larger number of feedbacks per test datapoint is equally important as the size of the test set. An accurate RMSE estimate can be achieved with a $1,000$ annotator feedbacks distributed across 100 datapoints. Of course, the amount of annotator feedbacks to be invested depends on the accuracy of the estimated RMSE required in a given application and should therefore be chosen accordingly.

\subsubsection{Different Dimensionality Reduction Techniques}

Probabilistic label spreading builds on the underlying graph structure and the assumption that nearby datapoints have similar soft labels. Due to the curse of dimensionality~\cite{bellmann1957}, distances in high dimensional are not meaningful anymore. Hence, we employ embedding techniques to reduce the image dimensions before constructing the graph. The purpose of this data pre-processing is to obtain a meaningful low-dimensional representation of the images where similar images are close to one another and dissimilar ones are further apart. Ideally, the smoothness assumption is met to a large extent. From a practical point of view, low-dimensional feature representations also come with less computational effort when constructing the graph based on distance computations. 

We evaluated the following dimensionality reduction or embedding techniques: principle component analysis (PCA) \cite[Sec.~20]{Murphy_prob_ML}, Isomap \cite{Isomap}, UMAP \cite{UMAP}, CLIP \cite{CLIP} as well as the combination of applying CLIP and an additional but different one of the aforementioned techniques subsequently to further reduce the dimension of the CLIP embeddings.
All techniques are applied with default parameters and the CLIP model ViT-B/32 was used. 
Changing the perspective on labels towards clear-cut labels for this particular experiment, we consider the quality of the embeddings in terms of the average silhouette score \cite[Sec.~21]{Murphy_prob_ML} they exhibit. This score reflects how well-clustered the data is in the embedding space with respect to the Euclidean distance. In the following experiment, we consider different target dimensions for all dimensionality reduction techniques and evaluate the quality of the embedding space with the average silhouette score. 

For the CIFAR-10 dataset, we obtain the scores as depicted in figure \ref{fig:silhouette_scores_CIFAR10}. It can be concluded that the data is only well-structured if CLIP is employed and the combination of CLIP and UMAP produces the highest scores for this dataset, also rather independently of the target dimension. We found that the combination of CLIP and UMAP generally yields the highest silhouette scores across the considered datasets and a target dimension of $d=20$ is a reasonable choice.
The quality of the embedding impacts the performance of the PLS algorithm. Considering the parameters $\alpha = 0.99$, $k = 20$ and providing a budget of $1\,\%$ relative to the dataset size, we evaluate the performance on CIFAR-10 for different feature spaces in figure \ref{fig:performance_of_PLS_depending_on_embedding}. We note that the performance tends to be superior for embeddings that exhibit a larger average silhouette score. 

\begin{figure}[tb]
    \centering
    \includegraphics[width=\linewidth]{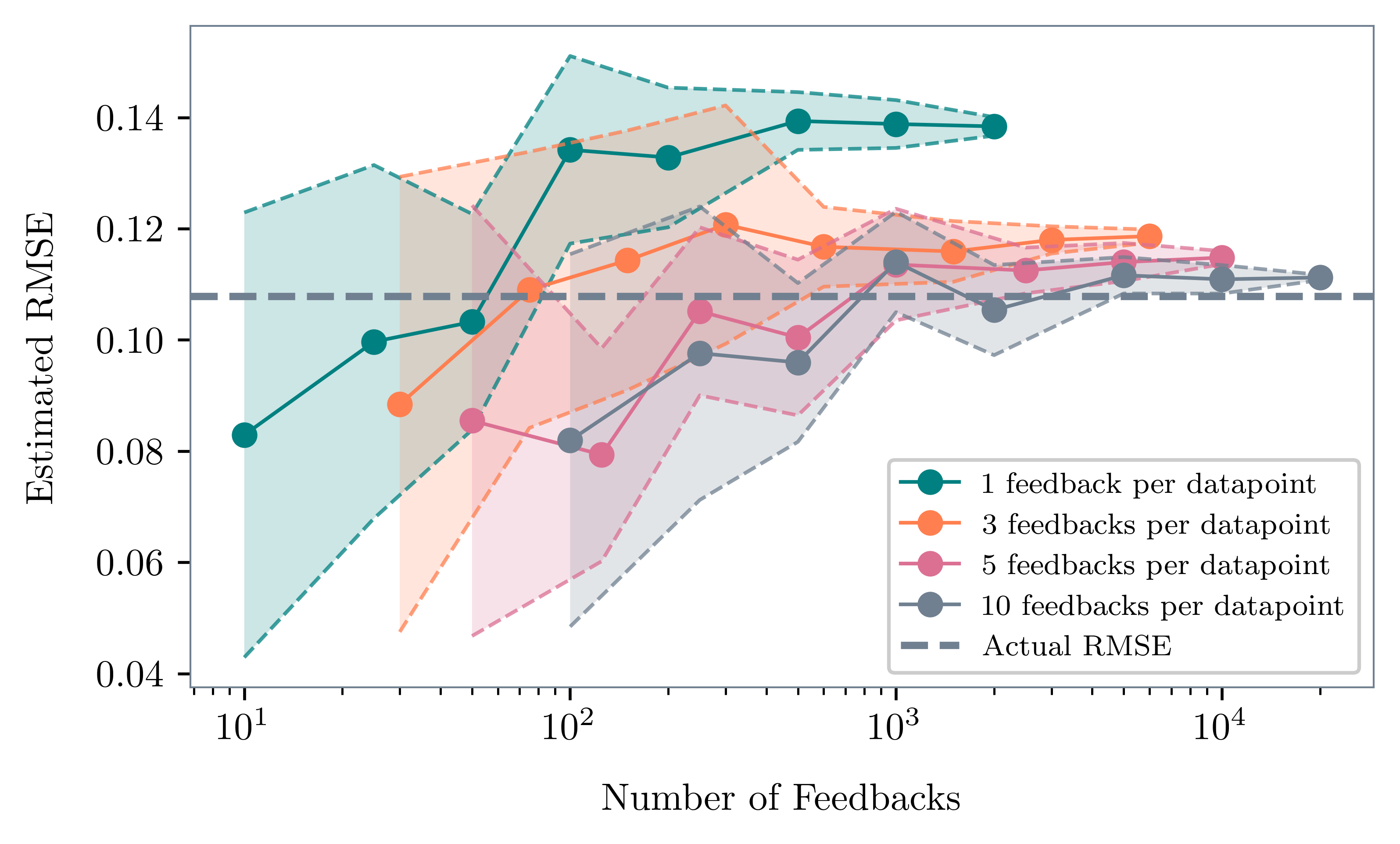}
    \caption{Estimated RMSE based on annotated datasets that differ in size and label quality. Soft labels are estimated with histograms of varying accuracy (1, 3, 5, 10 feedbacks per datapoint). The dataset size and the label quality result in a total number of feedbacks needed to construct the dataset.}
    \label{fig:estimated_rmse}
\end{figure}

\begin{figure*}[tb]
  \centering
  \begin{subfigure}[t]{0.49\linewidth}
    \includegraphics[width=\linewidth]{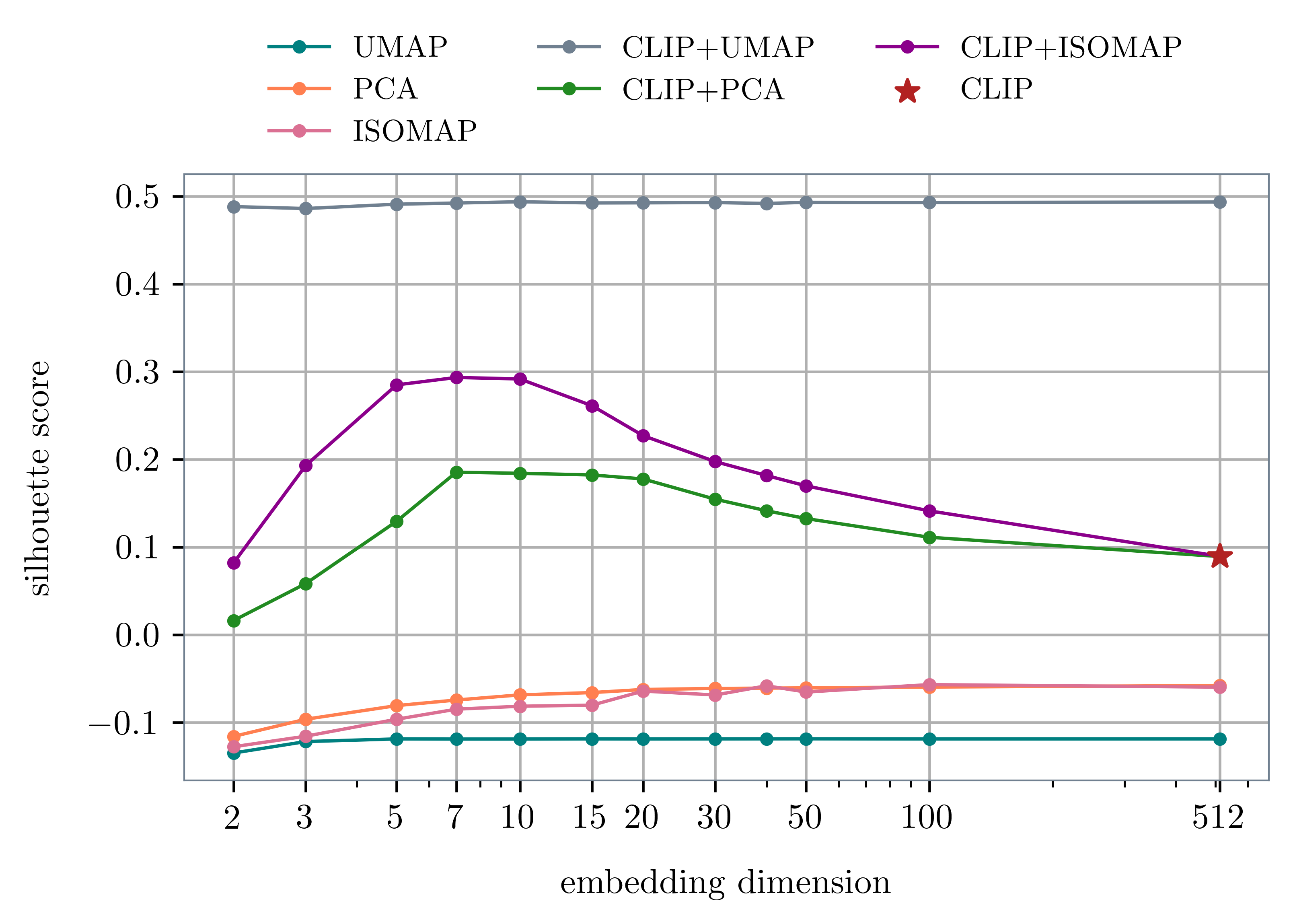}
    \caption{Average silhouette score for embeddings of the CIFAR-10 dataset constructed with different embedding techniques and dimensions.}
    \label{fig:silhouette_scores_CIFAR10}
  \end{subfigure}
  \hfill
  \begin{subfigure}[t]{0.49\linewidth}
    \includegraphics[width=\linewidth]{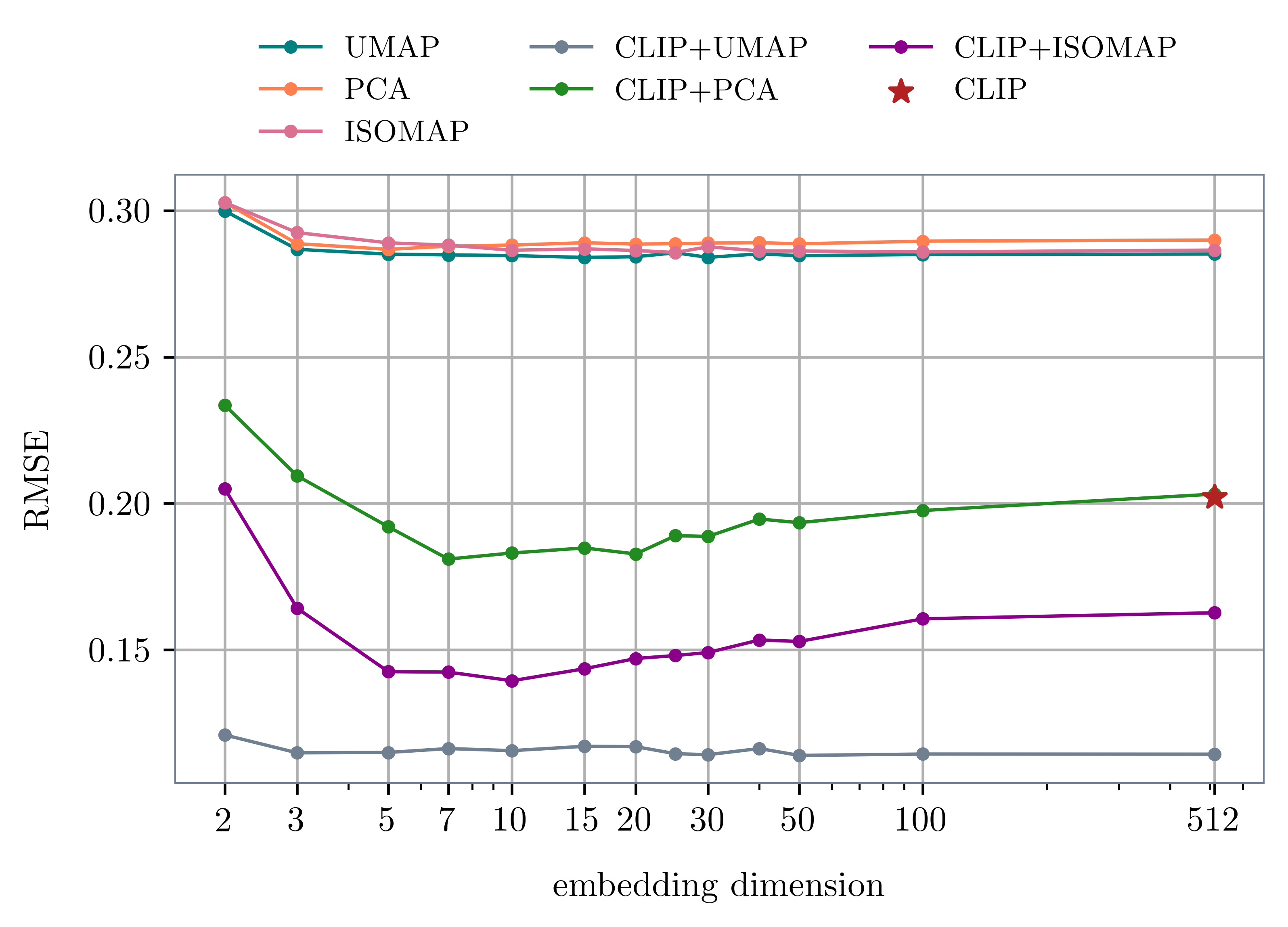}
    \caption{Performance of PLS on CIFAR-10 for different embeddings. Here, only the feature space is varied while all other parameters are fixed.}
    \label{fig:performance_of_PLS_depending_on_embedding}
  \end{subfigure}
  \caption{Analysis of the impact of the embedding techniques employed and the specified target dimension on the quality of the embedding space as well as the performance of the PLS algorithm.}
  \label{fig:varying_the_embedding_space}
\end{figure*}

\subsubsection{Details on Confidence Interval Construction}
Our method spreads received label information in such a way that the data point that was provided with an annotator feedback receives a propagation score of 1 and neighboring data points accordingly lower propagation scores, depending on the choice of the spreading intensity $\alpha$. Hence, we can interpret the obtained propagation scores as the number of virtual experiments conducted. As we have this information per class, we can construct confidence intervals based on proportions, i.e., absolute frequencies of observations. For this, we first convert the propagation scores into integers by rounding them down. In the case of two classes, we obtain a number $k_i$ for the virtual absolute frequency of class $1$ and a number $n_i$, which is the total number of virtual experiments conducted in the data point $x_i$. Based on these two integers, the Wilson score confidence interval~\cite{Wilson1927} for the data point $x_i$ can be computed for any confidence level.

Based on the considerations made in the proof in section \ref{sec-appendix:consistency-proof}, we derive Hoeffding type confidence intervals. Assuming two classes, for the estimated soft label $\hat{p}_q$ of $x_q$, the bias variance decomposition reads as
\begin{equation}
\label{eq: CI details error decomposition single component}
       \vert\hat{p}_q - p_q\vert  \leq \left\vert \hat{p}_q  - \EE_{Y|X}\left[\hat{p}_q \right] \right\vert + \left\vert \EE_{Y|X}\left[\hat{p}_q \right] - p_q \right\vert 
\end{equation}
Let us denote the propagation score that $x_q$ receives from $x_{i_j}$ by $\varphi_{q,i_j}$ and the cumulative propagation score $x_q$ received by $\Phi_q$ (summed over all $x_{i_j}$ that were provided with annotator feedback).
Assuming Lipschitz continuity, the bias of the estimator can be bounded through:
\begin{align}
    \label{eq: CI details bias bound}
    \left\vert \EE_{Y|X}\left[\hat{p}_q \right] - p_q \right\vert &= 
    \left\vert \sum_{j=1}^{m}  \frac{\varphi_{q,i_j}}{\Phi_q} \left( p_{i_j} - p_q\right) \right\vert \nonumber \\
    &\leq \sum_{j=1}^{m} \frac{\varphi_{q,i_j}}{\Phi_q} \left\vert p_{i_j} - p_q \right\vert \nonumber \\
    &\leq \sum_{j=1}^{m}\frac{\varphi_{q,i_j}}{\Phi_q} \min\{\,1,\,  L \| x_{i_j} - x_q\|\,\} =:\mathcal{B}
\end{align}

Assuming that we have knowledge of the Lipschitz-constant $L$, we can compute this bias bound. The first summand in equation \eqref{eq: CI details error decomposition single component} represents the variance associated with the sampling of feedbacks. As in the proof, we apply Hoeffding's inequality to obtain
\begin{equation}
\label{eq: CI details Hoeffding for conf int}
    P\left( \vert \hat{p}_q  - \EE_{Y|X}\left[\hat{p}_q \right]\vert > \varepsilon \right) \leq 2 \exp\left(-\frac{2\varepsilon^2}{\sum_{j=1}^{m}  \left(\frac{\varphi_{q,i_j}}{\Phi_q}\right)^2}\right)
\end{equation}
For any given failure probability $\delta\in(0,1)$, we can determine the smallest $\varepsilon_0>0$ possible such that the Hoeffding bound in \ref{eq: CI details Hoeffding for conf int} yields a bound for the failure probability smaller than $\delta$. This $\varepsilon_0$ then reflects the variance part of the error and, together with the bias bound $\mathcal{B}$, leads to the confidence interval bounds in \eqref{eq: CI details  Confidence Interval}.
\begin{align}
\label{eq: CI details Confidence Interval}
    P\left(\vert (\hat{p}_q)^c  - (p_q)^c \vert > \varepsilon_0 + \mathcal{B}  \right) \leq \delta
\end{align}

For the example given in the main manuscript, we have a soft label ground truth that is defined by
\begin{equation*}
    P(Y = 1 | x) = \frac{1}{2}(\sin(x) +1)
\end{equation*}
We sampled 2,000 data points uniformly in the range of 0 to 10, constructed a $k-$NN graph with $k = 20$ and ran the PLS algorithm with $\alpha = 0.99$. Here, we provided on average one simulated annotator feedback per data point but also included a region where annotator feedbacks are more scare, namely only 1 in twenty data points is provided with annotator feedback here.

\subsubsection{Hardware Specifications and Runtime of Experiments}
All experiments were conducted on a single NVIDIA Quadro P6000 GPU with 24 GB of memory and two Intel(R) Xeon(R) Gold 6138 CPUs
with 2.00 GHz and 20 cores/40 threads with 512 GB of RAM. Computationally expensive procedures are mainly due to data pre-processing including the fine-tuning of a model on each dataset as well as applying dimensionality reduction techniques but are limited to around 10 GPU hours. Evaluating our method ranges from a few seconds to several minutes depending on the dataset and annotation budget under consideration. In particular, the GKR baseline comes with memory constraints on the EMNIST-digits dataset since a dense graph based on Euclidean distances is constructed in the GKR method. For this reason, we only compared the performance of the algorithms for annotation budgets of up to $10\,\%$. Overall, the final experiments presented in this paper can be replicated on similar hardware within 48 hours.

\subsection{Detailed Description of all Algorithms}
\label{sec-appdx:detailed-description-of-all-algorithms}
This section complements the method descriptions from the main manuscript by providing more detailed descriptions of our PLS algorithm as well as of the baselines methods GKR and $k$-NN.

\subsubsection{Probabilistic Label Spreading}
In this section, the PLS algorithm is discussed in full algorithmic detail. The first loop occurring from line 8 to line 15 is the core of the algorithm where for the given annotation budget, the result of accordingly many annotator feedbacks is spread. The propagated information of each annotation is stored and normalized in the end to obtain a probability vector as an estimate. Note that the propagation scores are divided by the maximal value in step 12. In this way, the value of $1$ is added in the array entry corresponding to the data point that was provided with an annotator feedback, which reflects that one crowdsourcing experiment for this data point was conducted. Other data points receive partial feedback, depending on the choice of $\alpha$, which has the interpretation of accordingly many virtual experiments being conducted.

\begin{algorithm}[htb]
\caption{Probabilistic Label Spreading}
\label{alg:PLS}
\centering
\small
\begin{algorithmic}[1]
    \State \textbf{Input: }
        \parbox[t]{\dimexpr\linewidth-1.8cm\relax}{%
            $\mathcal{D} = (x_i, p_i)_{i = 1,\ldots,\ndata}$ \Comment{Probabilistic Dataset}\\
            $\alpha \in (0,1)$, $k \in \NN$  \Comment{Parameters} \\ 
            $m \in \NN$ \Comment{Budget / number of annotations} 
        }
    \vspace{0.25cm}
    \State $\sigma^2 = \frac{1}{n}\sum_{i=1}^n \| x_i - X_k(x_i) \|^2 $ \Comment{Avg. squared distance to $k$ NN}
    \vspace{0.1cm}
    \State $W \gets (w_{i,j})_{i,j = 1,\ldots, \ndata}$
    
    \noindent with $w_{i,j} = \begin{cases}
        \exp\left(-\frac{\|x_i - x_j\|^2 } { 2\sigma^2}\right), & i \neq j \land x_j \in \text{NN}_k(x_i) \\
        0, & \text{else}
    \end{cases}$
    \State $A \gets (W + W^\top) / 2$ \Comment{Symmetrization}
    \State $D \gets \operatorname{diag}(A \mathbf{1}_\ndata)$
    \State $S \gets D^{-\frac{1}{2}} \cdot A \cdot D^{-\frac{1}{2}}$ \Comment{Normalization by degree matrix}
    \vspace{0.25cm}
    \State Initialize $N, Y^1, \ldots, Y^C = (0, \ldots, 0) \in \mathbb{R}^{\ndata}$
    \For{$s = 1, \ldots, m$} \Comment{Spread each annotation}
        \State Draw $x_q \sim \mathcal{U}(\{x_1, \ldots, x_\ndata\})$
        \State Draw $c \sim p_q$
        \State $\phi \gets (I - \alpha S)^{-1} \cdot e_q$ with $e_q \in \mathbb{R}^{\ndata}$
        \State $\phi \gets \phi / \|\phi\|_\infty$ \Comment{Normalization for interpretability}
        \State $Y^c \gets Y^c + \phi$
        \State $N \gets N + \phi$
    \EndFor
    \For{$i = 1, \ldots, \ndata$} 
        \If{$N_i \neq 0$}
            \For{$c = 1, \ldots, C$}
                \State $\hat{p}_{i,c} \gets Y^c_i / N_i$ \Comment{Normalize class-wise feedback}
            \EndFor
        \Else
            \State $\hat{p}_{i} \gets (1 / C, \ldots, 1/C)$ \Comment{Uniform label if no feedback}
        \EndIf
    \EndFor

    \State \textbf{Return: } $\hat{p}_1, \ldots, \hat{p}_n $
\end{algorithmic}
\end{algorithm}

Given that the linear system in step 11 needs to be solved $m$ times, we rely on an efficient procedure to do so. Namely, we employ algebraic multigrid methods~\cite{Trottenberg2000Multigrid} to solve the system in a hierarchical approach. In the implementation of our method, we approximate the solution using the FGMRES solver~\cite{FGMRES} preconditioned by algebraic multigrid. The solves are performed with GPU support using the AMGX library~\cite{AMGX}.

\subsubsection{Gaussian Kernel Regression}
A sensible baseline algorithm for spreading soft labels is the diffusion of information via a Gaussian kernel. This fully supervised baseline essentially corresponds to the probabilistic label spreading algorithm without the underlying graph structure. In particular, it does not utilize the additional information on the marginal distribution of $X$. For a sampled data point $x_i$, a label $c_i$ is drawn according to $P(y\vert x_i)$. The impact of the annotator feedback in $x_i$ on the estimated soft label $\hat{p}_j$ of a data point $x_j$ is then given by
\begin{equation*}
    \varphi_x(z) = e^{-\gamma \| x_i -x_j \|_2^2}
\end{equation*}
where the rate of exponential decay with the Euclidean distance of two features is steered by the parameter $\gamma > 0$. Note that $h = \frac{1}{\gamma}$ is the bandwidth of the Gaussian kernel. Kernel-based estimators of this type were first introduced by Nadaraya and Watson \cite{Nadaraya1964,Watson1964}. 
The full baseline algorithm builds on a set of $m \in \NN$ queried annotator feedbacks. The estimated soft label $\hat{p}_x$ of a feature $x$ is determined by aggregating the class-wise impacts of all queries. Given, the sample $\{(x_{i_1}, c_{i_1}), \ldots, (x_{i_m}, c_{i_m}))\}$ where $c_{i_j}$ are one-hot encodings of the sampled feedback, the GKR algorithm returns the following estimate:
\begin{equation}
    \hat{p}_q = \frac{ \sum_{j=1}^{m} \exp(-\gamma \| x_q -x_{i_j} \|_2^2) \, c_{i_j} }{\sum_{j=1}^{m} \exp(-\gamma \| x_q -x_{i_j} \|_2^2) } \in[0,1]^C,\; q=1,\ldots,n
\end{equation}

\begin{algorithm}[htb]
\caption{Baseline algorithm: Gaussian Kernel Regression}
\centering
\small
\begin{algorithmic}[1]
    \State \textbf{Input: }
    \parbox[t]{\dimexpr\linewidth-1.8cm\relax}{%
        $\mathcal{D} = (x_i, p_i)_{i = 1,\ldots,\ndata}$ \\
        $\gamma > 0$  \\ 
        $m \in \NN$
    }
    \State $W \gets (w_{i,j})_{i,j} \text{ where } w_{i,j} = e^{-\gamma \| x_i -x_j \|_2^2}$ \Comment{Gaussian Kernel}
    \State Initialize $N, Y^1, \ldots, Y^c = (0, \ldots, 0) \in \mathbb{R}^{n_{\text{data}}}$
    \For{$s = 1, \ldots, m$}
        \State Draw $x_q \sim \mathcal{U}(\{x_1, \ldots, x_\ndata\})$
        \State Draw $c \sim p_q$
        \State $\phi \gets W \cdot e_q$ with $e_q\in \mathbb{R}^{\ndata}$ \Comment{Spread annotator feedback}
        \State $Y^c \gets Y^c + \phi$
        \State $N \gets N + \phi$
    \EndFor
    \For{$i = 1, \ldots, \ndata$}
            \For{$c = 1, \ldots, C$}
                \State $\hat{p}_{i,c} \gets Y^c_i / N_i$
            \EndFor
    \EndFor

    \State \textbf{Return: } $\hat{p}_1, \ldots, \hat{p}_n $
\end{algorithmic}
\end{algorithm}

\subsubsection{Nearest Neighbor Regression}
Another fully supervised baseline algorithm is the k-nearest neighbor regression (kNN). Here, for each feature $x_i$, the $k$ nearest neighbors out of the data points that were assigned at least one annotator feedback during the crowd-sourcing process are determined. 
We refer to those as $k$ nearest annotated neighbors. For any data points $x$ we estimate its soft label via the class-wise relative frequencies of the assigned annotator feedbacks of the $k$ nearest annotated neighbors of $x$. 
In other words, considering all annotator feedbacks over the $k$ nearest annotated neighbors $x_j$, each feedback corresponding to a class $y_j$. The probability estimate for class $c$ is then provided by $\frac{1}{k}\sum_{j=1}^k \1_{\{y_{i_j} = c\}}$ where $\1_{\{\cdot\}}$ denotes the indicator function. It equals $1$ if $y_{i_j} = c$ and $0$ otherwise.
\begin{algorithm}[htb]
\caption{Baseline algorithm: $k$-NN regression}
\centering
\small
\begin{algorithmic}[1]
    \State \textbf{Input: }
    \parbox[t]{\dimexpr\linewidth-1.8cm\relax}{%
        $\mathcal{D} = (x_i, p_i)_{i = 1,\ldots,\ndata}$ \\
        $k\in\NN$  \\ 
        $m \in \NN$
    }
    \State Initialize $N, Y^1, \ldots, Y^c = (0, \ldots, 0) \in \mathbb{R}^{\ndata}$
    \For{$s = 1, \ldots, m$} \Comment{Query annotator feedbacks}
        \State Draw $x_q \sim \mathcal{U}(\{x_1, \ldots, x_\ndata\})$
        \State Draw $c \sim p_q$
        \State $Y_q^c \gets Y_q^l + 1$
        \State $N_q \gets N_q + 1$
    \EndFor
    \State $I \gets N[\neq 0]$ \Comment{Indices of annotated data points}
    \For{$i = 1, \ldots, \ndata$}
        \State $J \gets \mathrm{kNN}\vert_{I}(x_i)$ \Comment{Indices of kNN out of annotated data points}
        \State $Y_i \gets \sum_{j \in J} Y_j$
        \State $N_i \gets \sum_{j \in J} N_j$
    \EndFor

    \For{$i = 1, \ldots, \ndata$}
            \For{$c = 1, \ldots, C$}
                \State $\hat{p}_{i,c} \gets Y^c_i / N_i$
            \EndFor
    \EndFor
    
    \State \textbf{Return: } $\hat{p}_1, \ldots, \hat{p}_n $

\end{algorithmic}
\end{algorithm}


\clearpage
\newpage

\onecolumn

\subsection{Consistency Proof}
\label{sec-appendix:consistency-proof}

\newcommand{\bandwidth}{h}
\newcommand{\sample}{\mathcal{S}}

We consider the feature space $\mathcal{X}=\RR^d$ and the target space $\mathcal{Y} = \Delta_C$, the probability simplex with $C$ categories. 
We prove that label spreading estimators are consistent estimators of the conditional distribution $P_{Y|X}$ under certain assumptions.

We assume the following:
\begin{itemize}
    \item The target distribution $P_{Y\vert X}$ is Lipschitz-continuous w.r.t. $X$, i.e.,
    \begin{equation*}
        \exists \, L_y > 0:\, \|P_{Y \vert X = x} - P_{Y \vert X = x^\prime}\|_2 \leq L_y \cdot \| x -x^\prime \|_2 \quad \forall\, x,x^\prime \in \mathcal{X}
    \end{equation*}
    \item $\mathcal{X}\subset \RR^d$  is a valid region according to the following definition (adapted from definition 2 of \cite{vonLuxburg2014}):
        \begin{definition}[Valid region] \label{def: valid region}
        Let $p$ be any density on $\mathbb{R}^d$. We call a connected subset $\mathcal{X} \subset \mathbb{R}^d$ a valid region if the following properties are satisfied:
        \begin{enumerate}
            \item The density on $\mathcal{X}$ is bounded away from $0$ and infinity, that is for all $x \in \mathcal{X}$ we have that $0 < p_{\text{min}} \leq p(x) \leq p_{\text{max}} < \infty$ for some constants $p_{\text{min}}, p_{\text{max}}$.
            
            \item For any radius $r$ that is not too large, a minimal fraction of the Euclidean ball $B_r(x)$ around any data point $x \in \mathcal{X}$ lies within the feature space $\mathcal{X}$. I.e., there exist positive constants $\rho > 0$ and $r_0 > 0$ such that if $r < r_0$, then for all points $x \in \mathcal{X}$ we have 
            \[
            \mathrm{vol}(B_r(x) \cap \mathcal{X}) \geq \rho \, \mathrm{vol}(B_r(x))
            \]

        \end{enumerate}
        \end{definition}
\end{itemize}

Let $n\in \NN$ denote the number of data points based on which the graph is constructed and $\nsamples\in \NN$ the total number of feedbacks that are assigned to this sample of data points $\{x_1,\ldots,x_n\}$, i.e. the number of crowdsourcing queries conducted. 
Under the assumptions stated above, choosing $\alpha \to 1$ sufficiently slow and considering a sufficiently sparse neighborhood graph, the probabilistic label spreading (PLS) algorithm yields consistent soft label estimates if the budget $\nsamples$ grows suitably with $n$ as $n\to\infty$. This is formalized in Theorem 1 of the main manuscript which states:
\begin{theorem}[Theorem 1 of the main manuscript]
    Let $\varepsilon> 0$ and $\delta \in(0,1)$.
    Under the assumptions made above, when choosing the spreading intensity dependent on the dataset size as \mbox{$\alpha_n = 1 - n^{\frac{1}{d+1}} \to 1$}, the bandwidth $h_n$ of the epsilon-graph as \mbox{$h_n = \varepsilon /  (3 L_y \lceil\log_{\alpha_n} \varepsilon / (\sqrt{2}\,12) \rceil) \to 0$}, and providing an annotation budget of at least $m_n = O(n^{1 - \frac{1}{2(d+1)}} \log n)$, for n sufficiently large, the PLS algorithm returns estimated soft labels $\hat{p}_1, \ldots, \hat{p}_n$ such that $P(\exists\, q\in\{1,\ldots, n\}: \| \hat{p}_q - p_q\|_2 > \varepsilon) \leq \delta$.

    More precisely, depending on $\varepsilon$ and $\delta$, the following sample complexity bound holds for $n$:
    \begin{align}
        P(\exists\, q\in\{1,\ldots, n\}: \| \hat{p}_q - p_q\|_2 > \varepsilon) \nonumber  \leq 12 n C \exp\left( - \zeta(p_{\min}, p_{\max},\varepsilon, \mathcal{X}, d, C, L_y)  \frac{\nsamples}{n^{1-\frac{1}{2(d+1)}}}\right)
    \end{align}
    with some constant $\zeta$ depending on $\varepsilon$, the minimal and maximal density $p_{\min}$, $p_{\max}$, the feature space $\mathcal{X}$ (its boundary shape), the feature dimension $d$, the number of classes $C$, and the Lipschitz constant $L_y$.
\end{theorem}

For technical reasons, we consider a slight variation of the PLS algorithm compared to the implementation described in the main manuscript. In particular, we consider an $\varepsilon$-graph instead of a $k$-NN graph and assign unit weights on edges in graph construction instead of Gaussian kernel based ones. The purely distance based linking of nodes in the $\varepsilon$-graph induces a symmetric neighborhood relation and is easier to handle mathematically. We start with a brief overview of the proof:

\begin{tcolorbox}[breakable]
\begin{center}
    {\bfseries Idea / Motivation} 
\end{center}

In the following, the idea is that with a growing number of data points $\ndata$, the graph will get more dense so that it requires a long path to exceed a distance of $\varepsilon$ from any vertex $x_q$. That way, if $x_j$ is outside of an $\varepsilon$-surrounding of $x_q$, its impact on $x_q$ will be small due to the fact that the spreading parameter $\alpha$ multiplies with each edge. On the other hand, the Lipschitz-continuity of $P_{Y\vert X}$ guarantees that nodes within the $\varepsilon$-surrounding of $x_q$ will have a similar soft label. Overall, the error will be small within the $\varepsilon$-surrounding whereas the cumulative propagation score of all points outside of the $\varepsilon$-surrounding tends to zero.

We have two sources of randomness for which we need to ensure that with large probability the error will be small if sufficiently many data points ($n\in\NN$) are present and a sufficient number of feedbacks ($\nsamples\in\NN)$ has been assigned:
\begin{enumerate}
    \item With large probability, the graph is sufficiently dense everywhere so that every data point $x_q$ is influenced primarily by the nodes in a small neighborhood. 
    To achieve this, we exploit the neighborhood structure of the $\varepsilon$-graph and shrink its bandwidth through a scaling with $1/l_n$ where $l_n$ is the path length for which we want to ensure that it remains within a small surrounding. We adjust the rate $l_n$ such that the shrinking balls still fill up with data points. 
    \item Having constructed a sufficiently dense graph (fixed $\ndata$), when acquiring $\nsamples$ feedbacks, it might happen that we equip only a single data point with label information. In fact, this rather extreme scenario happens with probability $(\frac{1}{\ndata})^{\nsamples}$. Hence, the variance due to selecting data points randomly for the crowdsourcing queries needs to be controlled and the same holds true for the variance involved in drawing feedbacks from soft labels.

    For this argument, we derive the fact that the variance of the estimator can be bounded in terms of the maximal effective kernel weight. Hence, we need to ensure that this maximal weight vanishes with high probability. To do so, we use the property derived in the first step. However, we need to show that it also holds true with high probability for the \emph{effective} kernel weights that only involve the data points that were assigned a feedback. For this, we apply concentration inequalities which require that $\nsamples$ grows suitably with $n$ and $n$ is sufficiently large.
\end{enumerate}
    \centering \textbf{Overview of the step-wise reasoning in this proof}
    \begin{itemize}
        \item (1) We are dealing with some kernel regression method for which we do not know much about the weight distribution. In addition, the observations are noisy as they are sampled from the underlying distribution $P_{Y\vert X}$.
        \item (2) By the assumption of the boundedness of the marginal density, we can guarantee that slowly shrinking balls ($\varepsilon / l_n$) fill up with data points. Exploiting the sparse kernel with shrinking bandwidth $\varepsilon / l_n$, we can bound the impact from outside of the $\varepsilon$-surrounding for the graph weights through $\alpha_n^{l_n} < \varepsilon$ by construction of the sequences $\alpha_n$ and $l_n$. By graph construction, we know that it takes a path of at least length $l_n$ to leave the $\varepsilon$-surrounding.
        \item We derive an error decomposition in (3) and exploit the Lipschitz continuity assumption to bound the bias within an $\varepsilon$-surrounding in (4).
        \item Guaranteeing the same in a probabilistic sense for the effective kernel is much more involved but can be done with the following considerations.
        \begin{itemize}
            \item (5) First, we bound the maximal kernel weight with high probability provided that there are sufficiently many data points. This weight bound $\varphi_{\max}$ decreases slowly ($n^{-\frac{\gamma}{2}})$ as $n$ grows (for this, we also need $\alpha_n \to 1$). 
            \item (6) With the maximal weight (nominator) bounded from above, we also need to lower bound the denominator (cumulative feedback).  For this, we apply Bernstein's concentration inequality to the random variable $\Phi_q$ that encodes the cumulative feedback obtained by the effective kernel. As the expectation of $\Phi_q$ is $\frac{\nsamples}{n}$, this imposes a growth condition on $\nsamples$. We need $\varphi_{\max} / \Phi_q^{B_\varepsilon} \to 0$, so essentially, $n^{-\frac{\gamma}{2}} \frac{n}{\nsamples} \to 0$, i.e., $\nsamples = O(n^{1-\frac{\gamma}{2}})$.
            \item (7) Bounding the effective kernel weight by a null sequence also bounds the individual observations $\varphi_{q,j} y_{i_j} / \Phi_q^{B_\varepsilon}$. We then apply Hoeffding's inequality for sums of bounded RVs in order to bound the probability of deviating from the mean. This gives an exponential tail bound and therefore, allows for a union bound argument. This controls the variance of the estimate in every data point.
            \item (8) By the concentration of the cumulative feedback we can also control the bias. By the path length argument, we know that the expected value of the feedback obtained from nodes outside of an $\varepsilon$-surrounding is small and we can guarantee the same with high probability by the concentration inequality.
        \end{itemize}
        \item (9) With both bias and variance under control, we conclude the proof and gather the rates and conditions.
    \end{itemize}

    \textbf{Rates \& Conditions}

    \begin{itemize}
        \item $\alpha_n = 1- n^{-\frac{1}{d+1}}$, spreading intensity increases sufficiently slow
        \item $l_n = \lceil \log_{\alpha_n} \varepsilon \rceil$, path length that shrinks the bandwidth of the $\varepsilon$-graph sufficiently fast to control bias
        \item We showed that $n / l_n^d = o(n^{\frac{1}{d+1}})$, i.e. shrinking balls still fill up slowly. Let us denote $\gamma = \frac{1}{d+1}$.
        \item By choosing $\alpha_n\to 1$ and the fact that balls fill up, we can bound the maximal heat kernel weight w.h.p.\ by $\varphi_{\max}\leq c_0 n^{-\frac{\gamma}{2}}$ where $ c_0$ is a constant depending on the the density's minimal and maximal value, $\varepsilon$, the boundary constant of the feature space and its dimension $d$.
        \item The Bernstein concentration inequality imposes the growth condition $\nsamples = O(n^{1-\frac{\gamma}{2}})$. A union bound gives an additional $\log n$ factor. The same growth condition and union bound argument appear in Hoeffding's inequality.
        \item E.g., we could choose $m_n = n^{1-\frac{1}{3}\gamma}$ to guarantee convergence in probability. 
    \end{itemize}
\end{tcolorbox}

{\bfseries Step 1 -- Derive the formula for the estimator} 

Let $\varepsilon> 0$ and $\delta\in(0,1)$ be given.
We adjust the error probability by the number of individual probabilistic arguments that need to hold simultaneously and define $\delta^\prime := \frac{\delta}{6}$. 

As a first step, we derive a compact formula for the point estimate $\hat{p}_q$ returned by the PLS algorithm for the soft label $p_q$ of a feature $x_q\in\{x_1,\ldots, x_n\}$. 

Let $\ndata\in\NN$ denote the number of data points we draw from the data generating distribution $P_{X}$ on $\mathcal{X}$. We obtain a sample of data points $\datasample = \{x_1, \ldots, x_\ndata\} \sim {P_X}^\ndata$. 
Next, $\nsamples\in\NN$ crowdsourcing experiments are conducted, i.e., for $j = 1,\ldots,m$, a data point $x_{i_j}\in\datasample$ is sampled uniformly $x_{i_j}\sim \mathcal{U}(\datasample)$ and a class $c_{i_j}\in\{1,\ldots,C\}$ is drawn according to its soft label, $c_{i_j} \sim P_{Y\vert X = x_{i_j}}$. Defining the corresponding one-hot encoding $y_{i_j}:=(\1_{\{c_{i_j} = 1\}}, \ldots, \1_{\{c_{i_j} = C\}})^\top$, we obtain the i.i.d.\ sample
\begin{equation}
    \mathcal{S} = \{(x_{i_1}, y_{i_1}),\ldots,(x_{i_m},y_{i_m})\}\subset \mathcal{X}\times \mathcal{Y}.
\end{equation}

For any $\alpha \in (0,1)$, the graph $(\datasample, E, W)$ is constructed according to some affinity Matrix $W$, which will be introduced further below. 
Defining $S:=D^{-1} W$, where $D$ is the degree matrix, we determine the propagation score of any $x_j\in \datasample$ with respect to any other data point $x_q\in\datasample$ by the corresponding entry of the heat kernel $(I-\alpha S)^{-1}$.
Note that if a crowd-sourcing experiment is conducted for $x_j$, the entire connected component of $x_j$ receives some feedback that acts as a weight in transmitting the result of the crowd-sourcing query. Let $e_j$ denote the $j$-th unit vector in $\RR^\ndata$. For all vertices of the graph, this feedback is determined through the label spreading algorithm \cite{Zhou_2003} with the single seed $x_j$:
\begin{equation}
    \Tilde{\varphi}_j := (I-\alpha S)^{-1} e_j \in \RR^\ndata_+
\end{equation}
Let us denote the propagation score a feature $x_q\in \datasample$ receives from $x_j\in \datasample$ by
\begin{equation}
    \Tilde{\varphi}_{q,j} := e_q^T (I-\alpha S)^{-1} e_j 
\end{equation}

Note that the heat kernel can be written as a Neumann series:
\begin{equation}\label{eq: heat kernel neumann series}
    (I-\alpha S)^{-1} = \sum_{i=0}^\infty (\alpha S)^i
\end{equation}
The matrix $S = D^{-1} W$ is the affinity matrix normalized by the degree matrix and therefore, a right stochastic matrix, meaning that each row of $S$ sums up to 1 and the spectrum of $S$ is contained within $[-1, 1]$ (cf. \cite{Zhou_2003}). Hence, $\norm{S}_2\leq 1$, which implies $\norm{\alpha S}_2<1$ justifying equation \eqref{eq: heat kernel neumann series}. In particular, the rows of each power of $S$ sum up to 1 as products of right stochastic matrices are right stochastic. Thus, we can calculate the sum of all graph weights $\Tilde{\varphi}_{q,j}$ impacting $x_q$ as:
\begin{align}
    \sum_{j=1}^n \Tilde{\varphi}_{q,j} = \sum_{j=1}^n \sum_{i=0}^\infty \alpha^i S^i_{q,j} = \sum_{i=0}^\infty \alpha^i \sum_{j=1}^n  S^i_{q,j} = \sum_{i=0}^\infty \alpha^i = \frac{1}{1-\alpha}
\end{align}
Where we used the fact that each entry of $S^i$ is bounded by $1$ and hence, the appearing series is bounded by the geometric series to exchange the order of summation.
Thus, we define the normalized graph weights as
\begin{equation}
    \label{eq:def_phi}
    \varphi_{q,j} := (1-\alpha) \Tilde{\varphi}_{q,j} =  (1-\alpha) \sum_{i=0}^\infty \alpha^i S^i_{q,j}
\end{equation}

Based on the sample $\mathcal{S}$, the probabilistic label spreading algorithm returns the following estimate $\hat{p}_q$ for the soft label of $x_q \in \{x_1,\ldots, x_n\}$:
\begin{equation}
    \label{eq:estimate}
        \hat{p}_q = \frac{\sum_{j=1}^{\nsamples} \Tilde{\varphi}_{q, i_j} y_{i_j}  }{\sum_{j=1}^{\nsamples} \Tilde{\varphi}_{q, i_j}} = \frac{\sum_{j=1}^{\nsamples} \varphi_{q, i_j} y_{i_j}  }{\sum_{j=1}^{\nsamples} \varphi_{q, i_j}}
\end{equation}

In essence, this is the average over the annotations $y_{i_j}$ assigned to data points $x_{i_j}$ weighted by the impact their corresponding features $x_{i_j}$ induce on $x_q$ through the heat kernel, similar to general kernel regression methods~\cite{Nadaraya1964,Watson1964}. As a feature $x_{i_j}$ may appear multiple times in the sample $\sample$, this weighting also accounts for how many feedbacks a data point has been provided with.

\textbf{Choosing a sequence $\alpha_n \to 1$}

By equation \ref{eq:def_phi} the weight induced by the heat kernel of any data point $x_q$ w.r.t. to itself is 
\begin{equation}
\varphi_{q,q} = (1-\alpha) + (1-\alpha)\sum_{i=1}^\infty \alpha^i S^i_{q,q}    
\end{equation}
Hence, if $x_q$ occurs in the sample $\sample$, it will significantly influence the estimate $\hat{p}_q$ with this weight if $\alpha$ remains constant. 
To later ensure that the variance of the estimator vanishes, we need the weights of the individual summands to vanish. Hence, we define a sequence $\alpha_n$ such that $\alpha_n\to 1$ as $n\to \infty$.
For any given $n\in\NN$, we choose 
\begin{equation}
    \alpha_n := 1 - \frac{1}{\sqrt[d+1]{n}}
\end{equation}
where $d$ is the dimension of the feature space.

{\bfseries Step 2 -- Argument via Path Length and $\varepsilon/l_n$-Sparsity of the Graph}

The main argument of this proof builds on considering the shortest path lengths in the constructed graph, which can be used to bound the weight induced by the heat kernel. More precisely, as derived below, for any given $\alpha_n\in(0,1)$, the combined impact of all nodes that are connected to $x_q$ only via paths of at least length $l_n\in\NN$ can be bounded by the partial geometric series $\sum_{i=l_n}^\infty \alpha_n^i = \frac{\alpha_n^{l_n}}{1-\alpha_n}$. Hence, for a sufficiently large path length $l_n\in\NN$ such that $\frac{\alpha_n^{l_n}}{1-\alpha_n}< \varepsilon$, the impact of the nodes connected to $x_q$ only via such long paths is small.

Consider the bandwidth $\bandwidth_n:= \frac{\varepsilon}{l_n}$ and the window kernel $K_{\bandwidth_n}(x):=\1_{\{x\in B_{\bandwidth_n}(0)\}}$ that induces the $\bandwidth_n$-graph\footnote{usually referred to as $\varepsilon$-graph but in our case $\varepsilon$ already denotes the given error tolerance and $\bandwidth_n$ steers the graph connectivity.} on the sample $\datasample$ with the following affinity matrix $W$:
\begin{equation}
    w_{i,j} = \begin{cases}
        1, & \text{if } \|x_i - x_j\| < {\bandwidth_n} \text{ and } i \neq j \\
        0, & \text{else}
    \end{cases}.
\end{equation}
With this choice of a sufficiently sparse neighborhood graph, we ensure that for every node \mbox{$x_j\notin B_\varepsilon(x_q)$} there is no path shorter than $l_n$ that connects $x_j$ and $x_q$ due to their Euclidean distance. On the other hand, every node \mbox{$x_j\in B_{\bandwidth_n}(x_q)$} is directly connected to $x_q$. As the powers $S^i$ of the matrix $S$ encode the $i$-step transitions in the graph, we can derive the following:

\begin{align}
    \label{eq:small_weight_outside}
    \sum_{j=1}^n \1_{x_j\notin B_\varepsilon(x_q)} \varphi_{q,j}& = \sum_{j=1}^n \1_{x_j\notin B_\varepsilon(x_q)} (1-\alpha_n)\sum_{i=0}^\infty \alpha_n^i S^i_{q,j} \nonumber \\
    & = \sum_{j=1}^n \1_{x_j\notin B_\varepsilon(x_q)} (1-\alpha_n)\sum_{i=l_n}^\infty \alpha_n^i S^i_{q,j} \nonumber \\
    & = (1-\alpha_n)\sum_{i=l_n}^\infty \alpha_n^i \sum_{j=1}^n \1_{x_j\notin B_\varepsilon(x_q)}  S^i_{q,j} \nonumber \\
    & \leq (1-\alpha_n)\sum_{i=l_n}^\infty \alpha_n^i = \alpha_n^{l_n}
\end{align}
We require the following condition:
\begin{equation}
\label{eq:small_weight_outside_condition}
     \sum_{j=1}^n \1_{x_j\notin B_\varepsilon(x_q)} \varphi_{q,j} < \varepsilon
\end{equation}
which, according to \eqref{eq:small_weight_outside}, is fulfilled provided that the sequences $\alpha_n$ and $l_n$ fulfill the following:
\begin{equation}
    \alpha_n^{l_n} < \varepsilon
\end{equation}
Note that a larger $\alpha_n$ requires a larger path length $l_n$ to bound the remainder of the geometric series by $\varepsilon$, which decreases the bandwidth of the window kernel under consideration and in turn, requires more data points to be justified. In particular, based on the sequence $\alpha_n$, we choose the following sequence of path lengths $l_n$ and thereby, kernel bandwidths $\bandwidth_n$
\begin{align}
    &l_n := \lceil\log_{\alpha_n}\varepsilon\rceil \\
    &\bandwidth_n := \frac{\varepsilon}{l_n}
\end{align}

By the choice of the sequences $\alpha_n$ and $l_n$, we have that  the following condition, which commonly appears in consistency proofs of kernel methods is fulfilled:
\begin{equation}
\label{eq: n/l}
    n \bandwidth_n^d = \varepsilon^d \frac{n}{{l_n}^d} \stackrel{n\to\infty}{\to} \infty.
\end{equation}
To show that \ref{eq: n/l} holds, we derive the fact that the sequence $l_n$ grows slower than $\sqrt[d]{n}$:
\begin{align}
    l_n &= \log_{\alpha_n}\varepsilon \nonumber \\
    &= \frac{\log \varepsilon}{\log {\alpha_n}} \nonumber\\
    &\leq \frac{\log \varepsilon}{- (1- \alpha_n)} \nonumber\\
    &= \log \frac{1}{\varepsilon} \sqrt[d+1]{n} = O(n^{\frac{1}{d}})
\end{align}
For the inequality, we applied the fact that $\log(1-x) = -x - x^2/2 - x^3 / 3 - \ldots \leq -x\;\forall\;x\in(0,1)$.

In particular, the divergence in \ref{eq: n/l} implies that even though the kernel bandwidth $\bandwidth_n$ shrinks with growing $n$, the number of data points contained in the surrounding $B_{\bandwidth_n}(x_q)$ of each $x_q$ grows with $n$. We have at least the following growth rate
\begin{align}
    \label{eq:rate of n/l^d}
    \frac{n}{{l_n}^d} &= \left(\frac{n^{\frac{1}{d}}}{l_n}\right)^d \geq \left(\frac{n^{\frac{1}{d}}}{\log \frac{1}{\varepsilon} n^{\frac{1}{d+1}}}\right)^d \nonumber \\
    &=\left(\frac{n^{\frac{1}{d}-\frac{1}{d+1}}}{ \log \frac{1}{\varepsilon} }\right)^d  = \frac{ n^{\frac{1}{d+1} }}{(\log \frac{1}{\varepsilon})^d} 
\end{align}

Defining $\gamma:= \frac{1}{d+1}$, we have $n \bandwidth_n^d = o(n^\gamma)$.

\textbf{Step 3 -- Error decomposition}

We show the consistency of our estimate with the following decomposition: For one, we ensure that the impact of points outside of an $\varepsilon$-surrounding is small with high probability when considering the effective kernel, i.e., only a sample of the data points $\{x_1\ldots, x_n\}$ spreads information. This implies that the estimate is mainly based on observations from within the $\varepsilon$-surrounding, which are drawn from similar soft labels due to the Lipschitz-continuity assumption. Hence, the bias of this estimator will be small and what is left to prove is that the variance vanishes as well. This requires a consideration of the heat kernel weights $\varphi_{q,i_j}$, in particular, we desire an upper bound that decreases with $n$ such that the variance vanishes.

Recall that the estimate of the soft label in $x_q$ is the following:
\begin{equation}
    \hat{p}_q =  \frac{\sum_{j=1}^{\nsamples} \varphi_{q, i_j} y_{i_j}  }{\sum_{j=1}^{\nsamples} \varphi_{q, i_j}}
\end{equation}
Assuming that the nodes $x_1, \ldots x_n$ are given, we define the following quantities:
\begin{align}
    &\Phi_q := \sum_{j=1}^{\nsamples} \varphi_{q, i_j} \\
    &\Phi_q^{B_\varepsilon} := \sum_{j=1}^{\nsamples} \1_{\{\|x_{i_j}-x_q\| \leq \varepsilon\}} \varphi_{q, i_j} \\
    &\hat{p}_q^{B_\varepsilon} := \frac{1}{\Phi_q^{B_\varepsilon}} \sum_{j=1}^{\nsamples} \1_{\{\|x_{i_j}-x_q\| \leq \varepsilon\}} \varphi_{q, i_j} y_{i_j}
\end{align}
These quantities depend on the sample $\sample$ where $\Phi_q$ and $\Phi_q^{B_\varepsilon}$ only depend on the graph structure, i.e., the nodes $x_1, \ldots, x_n$ and the sampled features $x_{i_j}$ while $\hat{p}_q^{B_\varepsilon}$ additionally depends on the drawn annotation feedbacks $y_{i_j}$. 

The error decomposition described above reads as follows:
\begin{align}
\label{eq: error decomposition}
   \|\hat{p}_q - p_q\| &= \left\| \frac{1}{\Phi_q}  \sum_{j=1}^{\nsamples} \varphi_{q, i_j} y_{i_j} - p_q \right\| \nonumber  \\
   &= \left\| \frac{1}{\Phi_q}  \sum_{j=1}^{\nsamples} \varphi_{q, i_j} ( y_{i_j} - p_q )\right\| \nonumber \\
   & \leq \left\| \frac{1}{\Phi_q}  \sum_{j=1}^{\nsamples} \1_{\{\|x_{i_j}-x_q\| \leq \varepsilon\}} \varphi_{q, i_j} ( y_{i_j} - p_q )\right\| + \left\| \frac{1}{\Phi_q}  \sum_{j=1}^{\nsamples} \1_{\{\|x_{i_j}-x_q\| > \varepsilon\}} \varphi_{q, i_j} ( y_{i_j} - p_q )\right\| \nonumber \\
   & = \left\| \frac{\Phi_q^{B_\varepsilon}}{\Phi_q}  (\hat{p}_q^{B_\varepsilon}  - p_q )\right\| + \left\| \frac{1}{\Phi_q}  \sum_{j=1}^{\nsamples} \1_{\{\|x_{i_j}-x_q\| > \varepsilon\}} \varphi_{q, i_j} ( y_{i_j} - p_q )\right\| \nonumber \\
   &\leq \left\| \hat{p}_q^{B_\varepsilon}  - p_q \right\| + \frac{1}{\Phi_q}  \sum_{j=1}^{\nsamples} \1_{\{\|x_{i_j}-x_q\| > \varepsilon\}} \varphi_{q, i_j} \left\| y_{i_j} - p_q \right\| 
    \quad \quad \text{ (since $\frac{\Phi_q^{B_\varepsilon}}{\Phi_q} \leq 1$ )} \nonumber \\
    &\leq \left\| \hat{p}_q^{B_\varepsilon}  - p_q \right\| + \sqrt{2} \frac{1}{\Phi_q}  \sum_{j=1}^{\nsamples} \1_{\{\|x_{i_j}-x_q\| > \varepsilon\}} \varphi_{q, i_j}\nonumber  \\
    & = \left\| \hat{p}_q^{B_\varepsilon}  - p_q \right\| + \sqrt{2} \frac{\Phi_q - \Phi_q^{B_\varepsilon}}{\Phi_q} \nonumber \\
    & \leq \left\| \hat{p}_q^{B_\varepsilon}  - \EE_{Y|X}\left[\hat{p}_q^{B_\varepsilon} \right] \right\| + \left\| \EE_{Y|X}\left[\hat{p}_q^{B_\varepsilon} \right] - p_q \right\| + \sqrt{2}\frac{\Phi_q - \Phi_q^{B_\varepsilon}}{\Phi_q}
\end{align}

We obtain three terms, that we need to bound by $\varepsilon$ individually with high probability. The first term corresponds to the variance in $Y$ of the weighted estimate based solely on observations from within the $\varepsilon$-neighborhood. Here we sample labels $y_{i_j}$ from a categorical distribution that is similar to the one defined by $p_q$, the soft label, we try to estimate. The second term is the bias of this estimate, which we will bound within the $\varepsilon$-neighborhood exploiting the Lipschitz continuity of the target function. The last summand is the impact of nodes that are not located within the $\varepsilon$-neighborhood and thus, potentially carry different soft labels. Above, we applied the universal bound of $\sqrt{2}$ for the distance of two probability vectors and hence, for this term need to prove that the overall impact $\frac{\Phi_q - \Phi_q^{B_\varepsilon}}{\Phi_q}$ on the estimate is small with high probability.

\textbf{Step 4 -- Bounding the bias}

We can bound the bias $\left\| \EE_{Y|X}\left[\hat{p}_q^{B_\varepsilon} \right] - p_q \right\|$ due to the Lipschitz continuity by
\begin{align}
\label{eq:bias}
    \left\| \EE_{Y|X}[\hat{p}_q^{B_\varepsilon}] - p_q \right\| &= \left\| \frac{1}{\Phi_q^{B_\varepsilon}} \sum_{j=1}^{\nsamples} \1_{\{\|x_{i_j}-x_q\| \leq \varepsilon\}}\varphi_{q,i_j} (\EE_{Y|X}[y_{i_j}] -p_q) \right\|\nonumber\\
    &= \left\| \frac{1}{\Phi_q^{B_\varepsilon}} \sum_{j=1}^{\nsamples} \1_{\{\|x_{i_j}-x_q\| \leq \varepsilon\}}\varphi_{q,i_j} (p_{i_j} -p_q) \right\|\nonumber\\
    &\leq \frac{1}{\Phi_q^{B_\varepsilon}} \sum_{j=1}^{\nsamples} \1_{\{\|x_{i_j}-x_q\| \leq \varepsilon\}}\varphi_{q,i_j} \| p_{i_j} - p_q\| \nonumber\\
    & \leq  \frac{1}{\Phi_q^{B_\varepsilon}}L_y \varepsilon \sum_{j=1}^{\nsamples} \1_{\{\|x_{i_j}-x_q\| \leq \varepsilon\}}\varphi_{q,i_j} \nonumber\\
    &= L_y \varepsilon
\end{align}

\textbf{Step 5 -- Deriving an explicit uniform bound $\varphi_{\max}$ for the graph weights}

Here, we ensure that the graph weights $\varphi_{q,j}$ can be bounded by an expression that decreases with $n$. As usual in weighted sum estimates, we need to ensure that no single observation dominates the estimate with a disproportionally large weight. Here, we also exploit that we chose $\alpha_n \to 1$. We perform two steps here. First, we ensure that in each small surrounding $B_{\bandwidth_n}(x_q)$ the number of data points concentrates around the expected value. Next, we give a straightforward estimate for the heat kernel weights based on the fact that there are at least a certain number of points in each surrounding.

\textbf{Step 5.1 -- Ensure that $n$ is sufficiently large such that each surrounding $B_{\bandwidth_n}(x_q)$ contains sufficiently many samples but also not much more than expected.}

In order to bound the maximal weight induced by the heat kernel, we first ensure that in each surrounding $B_{\bandwidth_n}(x_q)$, the number of data points concentrates around the expected value.

Using a Chernoff bound for binomial random variables and applying a union bound, we ensure that the number of data points in each ball $B_{\bandwidth_n}(x_q)$ does not deviate too much from the expected value with high probability. To do so, we define suitable binomial random variables corresponding to the number of sampled datapoints that lie within a given ball $B_{\bandwidth_n}(x_q)$.
By definition \ref{def: valid region}, for $\bandwidth_n < r_0$, the volume of $B_{\bandwidth_n}(x_q)$ in $\mathcal{X}$ can be bounded by $\mathrm{vol}(B_{\bandwidth_n}(x) \cap \mathcal{X}) \geq \rho \, \mathrm{vol}(B_{\bandwidth_n}(x)) = \rho \eta_d h_n^d$, where $\eta_d$ is the volume of the unit ball in $\RR^d$. Thus, the probability mass of $B_{\bandwidth_n}(x_q) \cap \mathcal{X}$ can be bounded by
\mbox{$\mu_{q,-}^{\bandwidth_n} := \max\{p(x_q) - L_x\bandwidth_n, p_{\min}\} \rho \eta_d \bandwidth_n^d$}, where $L_x$ is the Lipschitz constant of the density $p_X$. Similarly, \mbox{$\mu_{q,+}^{\bandwidth_n}:= (p(x_q) + L_x\bandwidth_n)\rho \eta_d \bandwidth_n^d$} denotes an upper bound for the same. To simplify notation, from here on out, we consider all balls to be contained in $\mathcal{X}$, so we mean $B_{\bandwidth_n}(x_q) \cap \mathcal{X}$ when we write $B_{\bandwidth_n}(x_q)$. Also, let \mbox{$N_q^{\bandwidth_n}:=\sum_{i=1}^n \1_{\{\|x_i-x_q\|<\bandwidth_n\}}$} be the number of samples in $B_{\bandwidth_n}(x_q)$. Note that by assumption the density $p_X$ is bounded from below by $p_{\min}>0$ and from above by $p_{\max}<\infty$ and hence, 
\begin{equation}
    \mu_{q,-}^{\bandwidth_n} \geq \rho \eta_d p_{\min} \bandwidth_n^d =:\mu_{\min}^{\bandwidth_n}
\end{equation}
as well as
\begin{equation}
    \mu_{q,+}^{\bandwidth_n} \leq \rho \eta_d p_{\max} \bandwidth_n^d =:\mu_{\max}^{\bandwidth_n}.
\end{equation}

The condition $\bandwidth_n < r_0$ imposes the following condition on $n$ for the probability mass bounds of $B_{\bandwidth_n}(x_q)$. We have
\begin{align}
    \bandwidth_n = \frac{\varepsilon}{\lceil \log_{\alpha_n} \varepsilon \rceil} \leq \frac{\varepsilon \log \alpha_n}{\log \varepsilon} = \frac{\varepsilon}{\log \varepsilon}\log(1-n^{-\frac{1}{d+1}})
\end{align}
Thus,
\begin{align}
\label{eq: data requirement for small radii}
    \bandwidth_n < r_0 &\iff \frac{\varepsilon}{\log \varepsilon}\log(1-n^{-\frac{1}{d+1}}) < r_0 \nonumber\\
    &\iff r_0 \log \frac{\log \varepsilon}{\varepsilon} < \log(1-n^{-\frac{1}{d+1}}) \nonumber\\
    &\iff \exp \left( r_0 \frac{\log \varepsilon}{\varepsilon} \right) < 1 - n^{-\frac{1}{d+1}} \nonumber\\
    &\iff n > \left(1- \exp \left( r_0 \frac{\log \varepsilon}{\varepsilon} \right)\right)^{-(d+1)}
\end{align}

For a binomial random variable $Z$, we have the following Chernoff bounds \cite[Chap.\ 8]{devroyeProbabilisticTheoryPattern1996a}:
\begin{align}
    P(Z \geq (1+\gamma) \EE[Z]) \leq& \exp\left(-\frac{\gamma^2}{3}\EE[Z]\right) \\
    P(Z \leq (1-\gamma) \EE[Z]) \leq& \exp\left(-\frac{\gamma^2}{2}\EE[Z]\right)
\end{align} 
And in particular, with $\gamma = \frac{1}{2}$ and writing $\EE[Z] = \mu n$:
\begin{align}
    \label{eq:upper_chernoff_bound}
    P(Z \geq \frac{3}{2}\mu n) \leq \exp\left(-\frac{1}{12} \mu n \right)
\end{align}
\begin{equation}
    \label{eq:lower_chernoff_bound}
     P\left(Z\leq \frac{1}{2}\mu n \right) \leq \exp\left(-\frac{1}{8}\mu n\right)
\end{equation}

With the definitions from above, we derive the following (inspired by \cite{vonLuxburg2014}, propositions 28 and 30):

\begin{equation}
    P\left(N_q^{\bandwidth_n} \geq \frac{3}{2}\mu_{q,+}^{\bandwidth_n} n \right) \stackrel{\eqref{eq:upper_chernoff_bound}}{\leq} \exp\left(-\frac{1}{12} \mu_{q,+}^{\bandwidth_n} n \right) \leq \exp\left(-\frac{1}{12} \mu_{\min}^{\bandwidth_n}n \right) 
\end{equation}
Thus, with a union bound argument, we obtain:
\begin{align}
    P\left(\exists\; q\in\{1,\ldots, n\}: N_q^{\bandwidth_n}\geq \frac{3}{2}\mu_{q,+}^{\bandwidth_n} n \right) & \leq \sum_{j=1}^n P\left(N_j^{\bandwidth_n} \geq \frac{3}{2}\mu_{q,+}^{\bandwidth_n} n  \right) \nonumber \\
     & \leq n  \exp\left(-\frac{1}{12}\mu_{\min}^{\bandwidth_n} n\right) \nonumber \\
     & = n \exp\left(-\frac{1}{12}\rho \eta_d p_{\min} \bandwidth_n^d n\right)
\end{align}
This probability tends to zero as $n\to\infty$ whenever $\bandwidth_n^d n$ grows faster than logarithmically, which is ensured by equation \ref{eq:rate of n/l^d}. 

Considering the lower tail, with a similar argumentation, we find that
\begin{align}
    P\left(N_q^{\bandwidth_n} \leq \frac{1}{2}\mu_{q,-}^{\bandwidth_n} n \right) \stackrel{\eqref{eq:lower_chernoff_bound}}{\leq} \exp\left(-\frac{1}{8} \mu_{q,-}^{\bandwidth_n} n \right) \leq \exp\left(-\frac{1}{8} \mu_{\min}^{\bandwidth_n}n \right) 
\end{align}
Again, we apply a union bound to derive
\begin{align}
    \label{eq:Chernoff and union bound}
    P\left(\exists\; q\in\{1,\ldots, n\}: N_q^{\bandwidth_n}\leq \frac{1}{2}\mu_{q,-}^{\bandwidth_n} n \right) &\leq \sum_{j=1}^n P\left(N_j^{\bandwidth_n}\leq \frac{1}{2}\mu_j^- n \right) \nonumber \\
     & \leq n  \exp\left(-\frac{1}{8}\mu_{\min}^{\bandwidth_n} n\right)
\end{align}
Under the above growth assumption, we can conclude, that for $n$ being sufficiently large the following holds:
\begin{equation}
    P\left(\exists\; q\in\{1,\ldots, n\}: N_q^{\bandwidth_n}\leq \frac{1}{2}\mu_{q,-}^{\bandwidth_n} n \right) \leq \delta^\prime
\end{equation}
In particular, according to \ref{eq:Chernoff and union bound}, $n$ needs to be chosen such that
\begin{align}
\label{eq:sample complexity for n}
    &\frac{n \mu_{\min}^{\bandwidth_n}}{12} - \log n \geq \log \frac{1}{\delta^\prime} \nonumber\\
    \iff &\frac{\rho \eta_d p_{\min} \bandwidth_n^d n}{12} - \log n \geq \log \frac{1}{\delta^\prime} \nonumber\\
    \iff &\frac{\rho \eta_d p_{\min} \varepsilon^d }{12} \frac{n}{{l_n}^d} - \log n \geq \log \frac{1}{\delta^\prime}.
\end{align}

Hence, we require $n$ to be sufficiently large such that \ref{eq:sample complexity for n} holds.
Then, the respective probability in eq.~\eqref{eq:Chernoff and union bound} is bounded by $\delta^\prime$.
Thus, for $n$ sufficiently large so that the inequality \ref{eq:sample complexity for n} is fulfilled the following holds:
\begin{equation}
    \label{eq: number of samples in each ball}
    P\left(N_q^{\bandwidth_n}\in\left[\frac{1}{2}\mu_{q,-}^{\bandwidth_n} n , \frac{3}{2}\mu_{q,+}^{\bandwidth_n} n\right]\; \forall\; q\in\{1,\ldots, n\}\right) \geq 1-2\delta^\prime.
\end{equation}

\textbf{Step 5.2 -- Bounding the graph weights}

By considering the Neumann-series representation of the heat kernel, we derive a uniform bound $\varphi_{\max}$ for the weights induced by the heat kernel such that $\max_{q,j} \varphi_{q,j} \leq \varphi_{max}$:
\begin{align}
    \varphi_{q,j} &= (1-\alpha_n) \sum_{i=0}^\infty \alpha_n^i e_q^T(D^{-1} W)^i e_j \nonumber \\
    &\leq (1-\alpha_n) + (1-\alpha_n)\sum_{i=1}^\infty \alpha_n^i \| e_q^T\| \|(D^{-1} W)^i e_j\|
\end{align}
Note that the Euclidean norm can be bounded using the spectral norm by
\begin{align}
    \|(D^{-1} W)^i e_j\|_2 &=  \|(D^{-1} W) \ldots (D^{-1} W) e_j\|_2 \nonumber\\
    &\leq \underset{\leq 1}{\underbrace{\|(D^{-1} W)\| \ldots \|(D^{-1} W)\|}} \|(D^{-1} W) e_j\|_2 \nonumber\\
    &\leq \|D^{-1} \| \| W e_j\|_2 \leq \|D^{-1} \| \sqrt{N_j^{\bandwidth_n}}
\end{align}
Since $D^{-1} = diag\left( \left( N_1^{\bandwidth_n} \right)^{-1}, \ldots, \left( N_n^{\bandwidth_n} \right)^{-1}  \right)$, we have
\begin{equation}
    \|D^{-1} \| = \max_{j=1,\ldots, n} \left(N_j^{\bandwidth_n} \right)^{-1} = \frac{1}{\min_{j=1,\ldots, n} N_j^{\bandwidth_n}} 
\end{equation}
Therefore, we find that for $i \geq 1$
\begin{align}
    \|(D^{-1} W)^i e_j\| &\leq \|D^{-1} \| \| W e_j\| \nonumber\\ &\leq \frac{1}{\min_{j=1,\ldots, n} N_j^{\bandwidth_n}} \sqrt{\max_{j=1,\ldots, n} N_j^{\bandwidth_n}} \nonumber\\ 
    &= \sqrt{ \frac{ \max_{j=1,\ldots, n}  N_j^{\bandwidth_n}} { \min_{j=1,\ldots, n}  N_j^{\bandwidth_n}}} \frac{1}{\sqrt{ \min_{j=1,\ldots, n}  N_j^{\bandwidth_n}}} \nonumber\\
    &\stackrel{P \geq 1 - 2\delta^\prime (\ref{eq: number of samples in each ball})}{\leq} \sqrt{ \frac{ \frac{3}{2} \mu_{\max}^{\bandwidth_n}n} { \frac{1}{2} \mu_{\min}^{\bandwidth_n}n }}  \sqrt{\frac{2}{\mu_{\min}^{\bandwidth_n}n}} \nonumber\\
    & = \underset{:= c(p_{\min}, p_{\max})  }{\underbrace{\sqrt{ 6\, \frac{p_{\max}} { p_{\min} }} }}\frac{1}{\sqrt{\mu_{\min}^{\bandwidth_n}n}}\to 0
\end{align}

Here, we exploited the boundedness of the density, by which the fraction of the maximal and minimal number of nodes within a surrounding will be constant with high probability if $n$ is sufficiently large as shown in \ref{eq: number of samples in each ball}. The variation of the density over the feature space is encoded in the constant $c(p_{\min}, p_{\max})$ defined above.

Overall, we obtain 
\begin{align}
\label{eq: bounding phi_qj}
    \varphi_{q,j} &= (1-\alpha_n) \delta_{q,j} + (1-\alpha_n)\sum_{i=1}^\infty \alpha_n^i \|(D^{-1} W)^i e_j\| \nonumber \\
    & \stackrel{\mathclap{P \geq 1 - 2\delta^\prime}}{\leq}(1-\alpha_n) \delta_{q,j} + \frac{c(p_{\min}, p_{\max})}{\sqrt{\mu_{\min}^{\bandwidth_n}n}} (1-\alpha_n)\sum_{i=1}^\infty \alpha_n^i  \nonumber \\
    & =  (1-\alpha_n) \delta_{q,j} + \frac{c(p_{\min}, p_{\max})}{\sqrt{\mu_{\min}^{\bandwidth_n}n}} \alpha_n \nonumber \\
    & \leq(1-\alpha_n) \delta_{q,j} +  \frac{c(p_{\min}, p_{\max})}{\sqrt{\mu_{\min}^{\bandwidth_n}n}}
\end{align}

Recall that by the choice of the sequences $\alpha_n$ and $l_n$, we ensure that $1-\alpha_n \to 0$  and $\mu_{\min}^{\bandwidth_n} n \to \infty$. In particular, $1-\alpha_n = n^{-\frac{1}{d+1}}$ and by equation \eqref{eq:rate of n/l^d}: 
\begin{align}
    n \mu_{\min}^{\bandwidth_n} = n \rho\eta_d p_{\min}\bandwidth_n^d &=\rho\eta_d p_{\min}\frac{\varepsilon^d n}{{l_n}^d} \nonumber \\
    & \geq \rho\eta_d p_{\min} \varepsilon^d \frac{n^\gamma}{(\log \frac{1}{\varepsilon})^d}
\end{align}
Thus, we can bound the second summand with high probability by
\begin{align}
    \frac{c(p_{\min}, p_{\max})}{\sqrt{\mu_{\min}^{\bandwidth_n}n}} \leq \frac{c(p_{\min}, p_{\max}) (\log \frac{1}{\varepsilon})^{\frac{d}{2}}}{\sqrt{\rho\eta_d p_{\min} \varepsilon^d}} \frac{1}{\sqrt{n^\gamma}}
\end{align}
where we define the appearing constant as $c_0 = c_0(p_{\min}, p_{\max},\varepsilon, \mathcal{X}, d):=1+\frac{c(p_{\min}, p_{\max}) (\log \frac{1}{\varepsilon})^{\frac{d}{2}}}{\sqrt{\rho\eta_d p_{\min} \varepsilon^d}} = 1+\frac{\sqrt{ 6\, p_{\max}} (\log \frac{1}{\varepsilon})^{\frac{d}{2}}}{p_{\min}\sqrt{\rho\eta_d \varepsilon^d}}$. Here, the 1 captures the first summand of equation \eqref{eq: bounding phi_qj}, which decays faster.

Overall, we have that for $n$ sufficiently large, with high probability
\begin{equation}
\label{eq: upper bound for phi_max}
    \varphi_{\max} \leq c_0 n^{-\frac{\gamma}{2}}
\end{equation}
Here, $\frac{\gamma}{2} = \frac{1}{2(d+1)}$ by equation \eqref{eq:rate of n/l^d} is a lower bound for the rate with which the maximal weight decays with high probability.

\textbf{Step 6 -- Concentration of $\Phi_q^{B_\varepsilon}$ around its expected value.}

We derived a bound for the graph weights $\varphi_{q,i_j}$ in the previous step. However, in the estimate, we consider the effective kernel, i.e., the graph weights are divided by the sum of the weights of all $x_{i_j}$ that were assigned feedback. Hence, to bound the effective kernel weights, we still require a lower bound for the sum of the weights $\Phi_q^{B_\varepsilon}$, i.e., the cumulative feedback obtained from points within the $\varepsilon$-surrounding.

We use the Bernstein inequality to prove the concentration around the mean of the cumulative feedback $\Phi_q^{B_\varepsilon}$ obtained in $x_q$. In particular, we use the version for bounded random variables as given in Theorem 2.9.5 of~\cite{vershynin_highdimprob_2026}. 
The considerations made here will be useful in the following step to show that with high probability, the variance term (first summand in \eqref{eq: error decomposition}) is small. The same concentration arguments will be applied to show that the cumulative feedback obtained from nodes outside of an $\varepsilon$-surrounding (third summand in \eqref{eq: error decomposition}) is small with high probability in Step 8.

\begin{theorem}[Bernstein inequality for bounded distributions (Theorem 2.9.5 of~\cite{vershynin_highdimprob_2026})]
\label{thm: Bernstein Ungleichung}
Let $X_1, \dots, X_N$ be independent, mean zero random variables satisfying
\[
|X_i| \le K \quad \text{for all } i.
\]
Then, for every $t \ge 0$, we have
\[
\mathbb{P}\!\left( \left| \sum_{i=1}^N X_i \right| \ge t \right)
\le 2 \exp\!\left( - \frac{t^2/2}{\sigma^2 + K t/3} \right),
\]
where
\[
\sigma^2 = \sum_{i=1}^N \mathbb{E}[X_i^2]
\]
is the variance of the sum.
\end{theorem}

Applying theorem \ref{thm: Bernstein Ungleichung} to the random variable $\Phi_q^{B_\varepsilon}$ results in the following lemma, the proof of which is given at the end to improve readability.
\begin{lemma}[Concentration of $\Phi_q^{B_\varepsilon}$]
\label{lemma: concentration of phi_eps}
    $\EE[\Phi_q^{B_\varepsilon}] \geq \frac{\nsamples}{n} (1-\varepsilon)$ holds with high probability if $n$ is sufficiently large by equation \eqref{eq:small_weight_outside_condition}. Given this lower bound, together with the boundedness of the graph weights, i.e., $\varphi_{\max}\leq c_0 n^{-\frac{\gamma}{2}}$ w.h.p., we find that
    \begin{align*}
            P\left( \vert \Phi_q^{B_\varepsilon} - \EE[\Phi_q^{B_\varepsilon}]| \geq \frac{\nsamples}{2n} (1-\varepsilon)\right) \leq  2 \exp\left( -\frac{3(1-\varepsilon)^2}{40 c_0}\nsamples n^{\frac{\gamma}{2} -1} \right)
    \end{align*}
\end{lemma}

Based on lemma \ref{lemma: concentration of phi_eps}, we apply a union bound and require $n, \nsamples$ to be sufficiently large with $\nsamples = O(n^{1-\frac{\gamma}{2}} \log n)$ such that 
\begin{equation}
\label{eq: Bernstein for Phi_eps with union bound}
    P\left(\exists\,q\in\{1,\ldots, n\}: \Phi_q^{B_\varepsilon} \leq \frac{\nsamples}{2n} (1 - \varepsilon) \right) \leq 2 n \exp\left( -\frac{3(1-\varepsilon)^2}{40 c_0}\nsamples n^{\frac{\gamma}{2} -1} \right) < \delta^\prime
\end{equation}
This means that with high probability, we can guarantee that for every data point, the cumulative feedback $\Phi_q^{B_\varepsilon}$ obtained from within its $\varepsilon$-surrounding is at least half of the expected value.

\textbf{Step 7 -- Apply Hoeffding's Inequality to bound the variance term}

Through the considerations of the previous steps, we can now bound the variance term appearing in the error decomposition given in \eqref{eq: error decomposition}. In particular, equation \eqref{eq: upper bound for phi_max} states that with high probability, $\varphi_{\max} \leq c_0 n^{-\frac{\gamma}{2}}$. Also, via equation \eqref{eq: Bernstein for Phi_eps with union bound}, we find that with high probability, $\Phi_q^{B_\varepsilon} \geq \frac{\nsamples}{2n} (1 - \varepsilon)$ for all $q\in\{1,\ldots,n\}$. 

First note that we can relate the deviation in the Euclidean norm with the deviations in the individual components:
\begin{align}
    \label{eq: from norm to components}
    P(\| \hat{p}_q^{B_\varepsilon} - \EE[\hat{p}_q^{B_\varepsilon}]\| > \varepsilon) &\leq P\left( \exists\, c\in\{1,\ldots, C\}: \left\vert (\hat{p}_q^{B_\varepsilon})^c - \EE[(\hat{p}_q^{B_\varepsilon})^c] \right\vert > \frac{\varepsilon}{\sqrt{C}}\right)\nonumber \\
    &\leq \sum_{c=1}^C P\left( \left\vert (\hat{p}_q^{B_\varepsilon})^c - \EE[(\hat{p}_q^{B_\varepsilon})^c] \right\vert > \frac{\varepsilon}{\sqrt{C}}\right)
\end{align}

Hence, we now derive a tail bound for the deviation of the estimate in an individual component. We exploit the boundedness of the random variables involved to derive an exponential tail bound using Hoeffding's inequality for a sum of independent bounded random variables.

\begin{theorem}[Hoeffding's inequality for bounded random variables (Theorem 2.2.6 of~\cite{vershynin_highdimprob_2026})]
\label{theorem: Hoeffding's inequality for bounded random variables}
Let $X_1, \ldots, X_M$ be independent random variables such that 
$X_j \in [a_j, b_j]$ for every $j$.\\
Then, for any $t > 0$, we have
\[
P\left( \sum_{j=1}^M (X_j - \EE[X_j]) \ge t \right)
\le 
\exp\!\left(
    \frac{-2t^2}{\sum_{j=1}^M (b_j - a_j)^2}
\right).
\]
\end{theorem}

In order to apply this theorem, we write
\begin{align}
    (\hat{p}_q^{B_\varepsilon})^c - \EE[(\hat{p}_q^{B_\varepsilon})^c] =\sum_{j=1}^{\nsamples} \1_{\{\|x_{i_j}-x_q\|<\varepsilon\}}\left( \frac{\varphi_{q,i_j}}{\Phi_q^{B_\varepsilon}}{y_{i_j}}^c - \EE\left[ \frac{\varphi_{q,i_j}}{\Phi_q^{B_\varepsilon}}{y_{i_j}}^c\right]\right)
\end{align}

Now, ${y_{i_j}}^c \in \{0,1\}$ and thus, $ \frac{\varphi_{q,i_j}}{\Phi_q^{B_\varepsilon}}{y_{i_j}}^c \in \{0,  \frac{\varphi_{q,i_j}}{\Phi_q^{B_\varepsilon}}\}$ for all $q,i_j\in\{1,\ldots,n\}$. In the context of theorem \ref{theorem: Hoeffding's inequality for bounded random variables}, this corresponds to $a_j = 0$ and $b_j= \frac{\varphi_{q,i_j}}{\Phi_q^{B_\varepsilon}}$ for all $j\in\{1,\ldots,\nsamples\}$. We obtain the following bound that holds with high probability based on the considerations of the previous steps:
\begin{align}
    P\left( \left\vert (\hat{p}_q^{B_\varepsilon})^c - \EE[(\hat{p}_q^{B_\varepsilon})^c] \right\vert > \frac{\varepsilon}{\sqrt{C}}\right) & =  P\left( \left\vert \sum_{j=1}^{\nsamples} \1_{\{\|x_{i_j}-x_q\|<\varepsilon\}}\left( \frac{\varphi_{q,i_j}}{\Phi_q^{B_\varepsilon}}{y_{i_j}}^c - \EE\left[ \frac{\varphi_{q,i_j}}{\Phi_q^{B_\varepsilon}}{y_{i_j}}^c\right]\right) \right\vert > \frac{\varepsilon}{\sqrt{C}}\right) \nonumber \\
    &\leq 2 P\left( \sum_{j=1}^{\nsamples} \1_{\{\|x_{i_j}-x_q\|<\varepsilon\}}\left( \frac{\varphi_{q,i_j}}{\Phi_q^{B_\varepsilon}}{y_{i_j}}^c - \EE\left[ \frac{\varphi_{q,i_j}}{\Phi_q^{B_\varepsilon}}{y_{i_j}}^c\right]\right) > \frac{\varepsilon}{\sqrt{C}}\right) \nonumber \\
    &\stackrel{\mathclap{\text{Hoeffding}}}{\leq} 2 \exp\left( - \frac{2\varepsilon^2}{C \sum_{j=1}^{\nsamples} \1_{\{\|x_{i_j}-x_q\|<\varepsilon\}} \left(\frac{\varphi_{q,i_j}}{\Phi_q^{B_\varepsilon}}\right)^2}\right) \nonumber \\
    &\leq 2 \exp\left( - \frac{2\varepsilon^2}{C \frac{\varphi_{\max}}{\Phi_q^{B_\varepsilon}} \sum_{j=1}^{\nsamples} \1_{\{\|x_{i_j}-x_q\|<\varepsilon\}} \frac{\varphi_{q,i_j}}{\Phi_q^{B_\varepsilon}}}\right) \nonumber \\
    &= 2 \exp\left( - \frac{2\varepsilon^2 \Phi_q^{B_\varepsilon}}{C \varphi_{\max}}\right) \nonumber \\
    &\stackrel{\mathclap{\text{Previous steps}}}{\leq} 2 \exp\left( - \frac{\varepsilon^2 (1-\varepsilon) \nsamples}{C c_0 n^{1-\frac{\gamma}{2}}}\right)
    \, \forall\, q\in\{1,\ldots,n\}, c\in\{1,\ldots, C\}
\end{align}

Inserting this bound into equation \eqref{eq: from norm to components}, gives:
\begin{align}
    P(\| \hat{p}_q^{B_\varepsilon} - \EE[\hat{p}_q^{B_\varepsilon}]\| > \varepsilon) & \leq \sum_{c=1}^C P\left( \left\vert (\hat{p}_q^{B_\varepsilon})^c - \EE[(\hat{p}_q^{B_\varepsilon})^c] \right\vert > \frac{\varepsilon}{\sqrt{C}}\right) \nonumber \\
    &\leq 2 C \exp\left( - \frac{\varepsilon^2 (1-\varepsilon)}{C c_0 } \frac{\nsamples}{n^{1-\frac{\gamma}{2}}}\right) \,\forall\, q\in\{1,\ldots, n\}
\end{align}

Lastly, we apply a union bound to derive a condition for convergence in probability of the estimate restricted to the $\varepsilon$-surrounding to its expected value for every node of the graph:
\begin{align}
\label{eq: uniform probability bound for variance term}
    P(\exists\, q\in\{1,\ldots, n\}: \| \hat{p}_q^{B_\varepsilon} - \EE[\hat{p}_q^{B_\varepsilon}]\| > \varepsilon) & \leq \sum_{q=1}^n     P(\| \hat{p}_q^{B_\varepsilon} - 
    \EE[\hat{p}_q^{B_\varepsilon}]\| > \varepsilon) \nonumber \\
    & \leq 2 n C \exp\left( - \frac{\varepsilon^2 (1-\varepsilon)}{C c_0 } \frac{\nsamples}{n^{1-\frac{\gamma}{2}}}\right) < \delta^\prime
\end{align}

By the growth condition imposed on the annotation budget, $\nsamples=O(n^{1-\frac{\gamma}{2}}\log n )$, we may require $n$ to be sufficiently large such that this probability is smaller than $\delta^\prime$.

\textbf{Step 8 -- Bounding the impact of data points outside of an $\varepsilon$-surrounding w.h.p.}

Similarly to the concentration of $\Phi_q^{B_\varepsilon}$, we require a concentration of the remainder $\frac{\Phi_q - \Phi_q^{B_\varepsilon}}{\Phi_q}$ that encodes the feedback from outside of an $\varepsilon$-surrounding. Again, we use Bernstein's inequality to derive the following two lemmas:
\begin{lemma}[Concentration of $\Phi_q$]
\label{lemma: concentration of phi}
    We have $\EE[\Phi_q] = \frac{\nsamples}{n}$. Given the boundedness of the graph weights, i.e., $\varphi_{\max}\leq c_0 n^{-\frac{\gamma}{2}}$ w.h.p., we find that
    \begin{align*}
            P\left( \Phi_q \leq \frac{\nsamples}{2n}\right) \leq  2 \exp\left( -\frac{3}{40 c_0}\nsamples n^{\frac{\gamma}{2} -1} \right)
    \end{align*}
\end{lemma}
Lemma \ref{lemma: concentration of phi} with a union bound leads to the following requirement on $n$:
\begin{equation}
\label{eq: Bernstein for Phi with union bound}
    P\left(\exists\,q\in\{1,\ldots, n\}: \Phi_q \leq \frac{\nsamples}{2n} \right) \leq 2 n \exp\left( -\frac{3}{40 c_0}\nsamples n^{\frac{\gamma}{2} -1}   \right) < \delta^\prime
\end{equation}

\begin{lemma}[Concentration of $\Phi_q - \Phi_q^{B_\varepsilon}$]
\label{lemma: concentration of remainder}
    $\EE[\Phi_q - \Phi_q^{B_\varepsilon}] \leq \varepsilon\frac{\nsamples}{n}$ holds with high probability if $n$ is sufficiently large by equation \eqref{eq:small_weight_outside_condition}. Given this upper bound, together with the boundedness of the graph weights, i.e., $\varphi_{\max}\leq c_0 n^{-\frac{\gamma}{2}}$ w.h.p., we find that
    \begin{align*}
            P\left( \Phi_q - \Phi_q^{B_\varepsilon} \geq 2\varepsilon \frac{\nsamples}{n}\right) \leq  2 \exp\left( -\frac{2\varepsilon}{(1+\frac{2}{3}\varepsilon) c_0}\nsamples n^{\frac{\gamma}{2} -1} \right)
    \end{align*}
\end{lemma}

Based on lemma \ref{lemma: concentration of remainder}, we apply a union bound and require $n, \nsamples$ to be sufficiently large with $\nsamples = O(n^{1-\frac{\gamma}{2}} \log n)$ such that 
\begin{equation}
\label{eq: Bernstein for remainder with union bound}
    P\left(\exists\,q\in\{1,\ldots, n\}: \Phi_q - \Phi_q^{B_\varepsilon} \geq 2\varepsilon\frac{\nsamples}{n} \right) \leq 2 n \exp\left( -\frac{2\varepsilon}{(1+\frac{2}{3}\varepsilon) c_0}\nsamples n^{\frac{\gamma}{2} -1}  \right) < \delta^\prime
\end{equation}

Equations \eqref{eq: Bernstein for Phi with union bound} and \eqref{eq: Bernstein for remainder with union bound} together yield:
\begin{equation}
\label{eq: bounding the impact from outside whp}
    P\left(\exists\,q\in\{1,\ldots, n\}: \frac{\Phi_q - \Phi_q^{B_\varepsilon}}{\Phi_q} \geq 4\varepsilon \right) < 2\delta^\prime
\end{equation}

This means that with high probability, for every data point, we can bound the cumulative feedback obtained from outside of an $\varepsilon$-surrounding by $4\varepsilon$.

\textbf{Step 9 -- Conclusion}

We derived an error decomposition in equation \eqref{eq: error decomposition}
\begin{equation*}
    \|\hat{p}_q - p_q\|\leq \left\| \hat{p}_q^{B_\varepsilon}  - \EE_{Y|X}\left[\hat{p}_q^{B_\varepsilon} \right] \right\| + \left\| \EE_{Y|X}\left[\hat{p}_q^{B_\varepsilon} \right] - p_q \right\| + \sqrt{2}\frac{\Phi_q - \Phi_q^{B_\varepsilon}}{\Phi_q}
\end{equation*}

By equation \eqref{eq:bias}, we can bound the bias in the second summand deterministically by $L_y \varepsilon$. By equation \eqref{eq: uniform probability bound for variance term}, we concluded that the first summand, which captures the variance of drawing feedbacks, can be bounded by $\varepsilon$ with high probability for $\nsamples$ growing faster than $n^{1- \frac{1}{2(d+1)}} \log n $ and $n$ being sufficiently large. Lastly, under the same growth condition on $\nsamples$ and $n$ being sufficiently large, by equation \eqref{eq: bounding the impact from outside whp}, we find that the feedback from within the $\varepsilon$-surrounding dominates the estimate with high probability, ensuring that the last summand is bounded by $\sqrt{2}\,4\varepsilon$.

To derive these statements, we applied several probabilistic arguments (6 in total), that all need to hold simultaneously. As they each hold with probability of at least $1-\delta^\prime$, by a union bound argument, the probability of all arguments holding up is at least $1-6\delta^\prime = 1-\delta$. We repeat the probabilistic arguments given and derive a final sample complexity bound by scaling with constants to achieve an overall error below $\varepsilon$.

\textbf{Probabilistic Arguments}
\begin{itemize}
    \item By two Chernoff Bound arguments, we concluded via eq. \eqref{eq: number of samples in each ball} that w.h.p. the number of data points within each ball does not deviate much from the expected value, which bounds the maximal kernel weight:
\begin{align*}
    P\left(N_q^{\bandwidth_n}\in\left[\frac{1}{2}\mu_{q,-}^{\bandwidth_n} n , \frac{3}{2}\mu_{q,+}^{\bandwidth_n} n\right]\; \forall\; q\in\{1,\ldots, n\}\right)
    &\geq 1 - 2n\exp(-\frac{1}{12}\rho \eta_d p_{\min} \bandwidth_n^d n) \\
    & \geq 1-2n\exp(-c_1 n^\gamma)
    \geq 1-2\delta^\prime.
\end{align*}
with some constant $c_1 = c_1(p_{\min}, \varepsilon, \mathcal{X}, d) = \frac{1}{12}\rho \eta_d p_{\min} \varepsilon^d (\log \varepsilon)^{-d}$. This rate induces a weaker condition on $n$ than the rates that follow.

\item Using Bernstein's inequality in lemma \ref{lemma: concentration of phi_eps}, we concluded that $\Phi_q^{B_\varepsilon}$ is not much smaller than its expected value in \eqref{eq: Bernstein for Phi_eps with union bound}:
\begin{align*}
\label{eq: Bernstein for Phi_eps with union bound}
    P\left(\exists\,q\in\{1,\ldots, n\}: \Phi_q^{B_\varepsilon} \leq \frac{\nsamples}{2n} (1 - \varepsilon) \right) &\leq 2 n \exp\left( -\frac{3(1-\varepsilon)^2}{c_0}\nsamples n^{\frac{\gamma}{2} -1} \right)\\ &= 2 n \exp\left( -c_2 \nsamples n^{\frac{\gamma}{2} -1} \right) < \delta^\prime
\end{align*}
with $c_2 = c_2(p_{\min}, p_{\max},\varepsilon, \mathcal{X}, d) = -\frac{3(1-\varepsilon)^2}{40 c_0}$ where $c_0=1+\frac{\sqrt{ 6\, p_{\max}} (\log \frac{1}{\varepsilon})^{\frac{d}{2}}}{p_{\min}\sqrt{\rho\eta_d \varepsilon^d}}$.

\item We bounded the estimate's deviation from the expectation using Hoeffding's inequality in \eqref{eq: uniform probability bound for variance term}:
\begin{align*}
    P(\exists\, q\in\{1,\ldots, n\}: \| \hat{p}_q^{B_\varepsilon} - \EE[\hat{p}_q^{B_\varepsilon}]\| > \varepsilon) 
    & \leq 2 n C \exp\left( - \frac{\varepsilon^2 (1-\varepsilon)}{C c_0 } \frac{\nsamples}{n^{1-\frac{\gamma}{2}}}\right) \\
    &= 2 n C \exp\left( - c_3 \frac{\nsamples}{n^{1-\frac{\gamma}{2}}}\right) < \delta^\prime
\end{align*}
with $c_3 = c_3(p_{\min}, p_{\max},\varepsilon, \mathcal{X}, d, C)  = \frac{\varepsilon^2 (1-\varepsilon)}{C c_0 }$.

\item 
 Finally, using lemmas \ref{lemma: concentration of phi} and \ref{lemma: concentration of remainder}, we bounded the cumulative feedback from outside of an $\varepsilon$-surrounding in equation \eqref{eq: bounding the impact from outside whp} by combining equations \eqref{eq: Bernstein for Phi with union bound} and \eqref{eq: Bernstein for remainder with union bound}. This way, we showed that w.h.p.\ the nodes within a surrounding dominate the estimate and bias is small.
\begin{equation*}
    P\left(\exists\,q\in\{1,\ldots, n\}: \Phi_q \leq \frac{\nsamples}{2n} \right) \leq 2 n \exp\left( -c_4\nsamples n^{\frac{\gamma}{2} -1}   \right) < \delta^\prime
\end{equation*}
\begin{equation*}
    P\left(\exists\,q\in\{1,\ldots, n\}: \Phi_q - \Phi_q^{B_\varepsilon} \geq 2\varepsilon\frac{\nsamples}{n} \right) \leq 2 n \exp\left( -c_5\nsamples n^{\frac{\gamma}{2} -1}  \right) < \delta^\prime
\end{equation*}
with $c_4 = c_4(p_{\min}, p_{\max},\varepsilon, \mathcal{X}, d) = -\frac{3}{40 c_0}$ and $c_5 = c_5(p_{\min}, p_{\max},\varepsilon, \mathcal{X}, d) = \frac{2\varepsilon}{(1+\frac{2}{3}\varepsilon) c_0}$.
\item Apart from the probabilistic arguments, we derived another requirement on $n$ in eq. \eqref{eq: data requirement for small radii}:
\begin{equation*}
    n > \left(1- \exp \left( r_0 \frac{\log \varepsilon}{\varepsilon} \right)\right)^{-(d+1)}
\end{equation*}
\end{itemize}

\textbf{Scaling with Constants}

To adjust for the scaling with constants and obtain an overall error of less than $\varepsilon$, we make the following adjustments and obtain the final sample complexity bounds. 
\begin{itemize}[left = 0pt]
    \item Adjust bandwidth: $\bandwidth_n \to \bandwidth_n / (3 L_y)$ gives bias $< \varepsilon /3$ and changes $c_1$ to $c_1^\prime = (3L_y)^{-d} c_1$.
    \item Adjusting the Hoeffding tail bounds in \eqref{eq: uniform probability bound for variance term} such that $\varepsilon \to \varepsilon / 3$ is straightforward and gives that the variance term is $< \varepsilon /3$ w.h.p. It updates  $c_3$ to $c_3^\prime = \frac{1}{9} c_3$.
    \item For the third summand to be $< \varepsilon / 3$, it is sufficient to require that impact of the nodes outside of the surrounding can be bounded by $\varepsilon / (\sqrt{2}\,12)$, i.e., $\EE[\Phi_q - \Phi_q^{B_\varepsilon}] \leq (\nsamples \varepsilon) / (\sqrt{2}\,12n)$. In that case, if w.h.p.\ $\Phi_q - \Phi_q^{B_\varepsilon}\leq 2\EE[\Phi_q - \Phi_q^{B_\varepsilon}]$ and $\Phi_q\geq \frac{1}{2}\EE[\Phi_q] = \nsamples / (2n)$, the fraction $\Phi_q - \Phi_q^{B_\varepsilon} / \Phi_q$ is $< \varepsilon/(\sqrt{2}\,3)$ and this error term is $< \varepsilon / 3$.  
    
    To achieve this, we define an adjusted path length sequence $\tilde{l}_n$ such that $\alpha_n^{\tilde{l}_n} < \varepsilon /  (\sqrt{2}\,12)$, i.e., $ \tilde{l}_n = \lceil \log_{\alpha_n}\frac{\varepsilon}{\sqrt{2}\,12}\rceil$. This induces an adjusted bandwidth $h_n^\prime = \frac{\varepsilon}{\tilde{l}_n}$ and affects all constants. Let us define $\tilde{h}_n = \frac{\varepsilon}{3L_y \tilde{l}_n}$ to also account for the bias scaling. It results that, $c_1 \to \tilde{c}_1 = \frac{1}{12 (3L_y)^d}\rho \eta_d p_{\min} \varepsilon^d (\log \frac{\varepsilon}{\sqrt{2}\,12})^{-d} $ and $c_0 \to \tilde{c_0}= 1+\frac{(3L_y)^{\frac{d}{2}}\sqrt{ 6\, p_{\max}} (\log \frac{\sqrt{2}\,12}{\varepsilon})^{\frac{d}{2}}}{p_{\min}\sqrt{\rho\eta_d \varepsilon^d}}$.
    
    The remaining constants now additionally depend on the Lipschitz constant and are affected as follows:
    \begin{itemize}
        \item $c_2 \to \tilde{c}_2 = -\frac{3(1-\frac{\varepsilon}{\sqrt{2}\,12})^2}{40 \tilde{c}_0}$
        \item $c_3 \to \tilde{c}_3 = \frac{\varepsilon^2 (1-\frac{\varepsilon}{\sqrt{2}\,12})}{9 C \tilde{c}_0 }$ (including variance term scaling and thus replacing $c_3^\prime$)
        \item $c_4 \to \tilde{c}_4 = - \frac{3}{40 \tilde{c}_0}$
        \item $c_5 \to \tilde{c}_5 = \frac{\varepsilon}{(\sqrt{2}\,6+\frac{1}{3}\varepsilon) \tilde{c}_0}$
    \end{itemize}
    \item Adjusting the bandwidth $\bandwidth_n \to \tilde{h}_n$ affects the data requirement on $n$ in eq. \eqref{eq: data requirement for small radii} as follows:
    \begin{equation*}
        n > \left(1- \exp\left(r_0 \frac{\log \frac{\varepsilon}{\sqrt{2}\,12}}{\varepsilon}3L_y\right)\right)^{-(d+1)}
    \end{equation*}
\end{itemize}

As a final sample complexity bound that implies that all probabilistic arguments hold simultaneously, we obtain
\begin{equation}
    P(\exists\, q\in\{1,\ldots, n\}: \| \hat{p}_q - p_q\|_2 > \varepsilon) \nonumber \\  \leq 12 n C \exp\left( - \zeta(p_{\min}, p_{\max},\varepsilon, \mathcal{X}, d, C,  L_y)  \frac{\nsamples}{n^{1-\frac{1}{2(d+1)}}}\right)
\end{equation}
with the constant $\zeta(p_{\min}, p_{\max},\varepsilon, \mathcal{X}, d, C, L_y) = \min\{\tilde{c}_2, \tilde{c_3}, \tilde{c}_4, \tilde{c}_5\}$.

\hfill $\square$
\newpage

\textbf{Step 10 -- Proving the Lemmas}

We conclude the proof by giving the calculations needed to infer lemmas \ref{lemma: concentration of phi_eps}, \ref{lemma: concentration of phi}, and \ref{lemma: concentration of remainder}. The reasoning is the same for all three lemmas that rely on Bernstein's inequality for sums of independent, bounded random variables. We start with the proof of lemma \ref{lemma: concentration of phi} as it is the most intuitive random variable and bootstrap the results.

\textbf{Proof of Lemma \ref{lemma: concentration of phi}: Concentration of $\Phi_q$}

Given a fixed number of vertices $x_1, \ldots, x_n$ out of which $\nsamples$ samples $x_{i_j}$ are drawn uniformly, i.e. $x_{i_j}\sim \mathcal{U}(x_1, \ldots, x_n)$ and $\varphi_{q, i_j}\sim \mathcal{U}(\varphi_{q, 1}, \ldots, \varphi_{q, n})$, the expectation of $\Phi_q$ with respect to this sampling procedure is given as
\begin{align}
\label{eq: expectation of phi}
    \EE\left[\Phi_q\right] &= \EE\left[\sum_{j=1}^{\nsamples} \varphi_{q,i_j}\right]
    = \sum_{j=1}^{\nsamples} \EE\left[\varphi_{q, i_j}\right]
    = \sum_{j=1}^{\nsamples} \frac{1}{n} \sum_{k=1}^n \varphi_{q, k} \nonumber\\
    &= \nsamples\frac{1}{n} \sum_{k=1}^n \varphi_{q, k} = \frac{\nsamples}{n}
\end{align}

Consequently, we can rewrite 
\begin{align}
    \Phi_q - \frac{\nsamples}{n} = \sum_{j=1}^{\nsamples} \left(\varphi_{q,i_j} - \frac{1}{n}\right)
\end{align}
where the individual summands are centered i.i.d.\ random variables $X_j := \varphi_{q,i_j} - \frac{1}{n}$ that are bounded by $\varphi_{\max}\leq c_0 n^{-\frac{\gamma}{2}}$ with high probability by equation \eqref{eq: upper bound for phi_max}. Hence, assuming the boundedness $\vert X_j \vert \leq c_0 n^{-\frac{\gamma}{2}}$, we can apply Bernsteins inequality. The variance of the sum can be bounded by the following:
\begin{align}
\label{eq: variance of the sum for phi}
    \VV[X_j] &= \EE\left[\left(\varphi_{q,i_j} -  \frac{1}{n}\right)^2\right] = \EE\left[\varphi_{q,i_j}^2\right]   -  \frac{1}{n^2} \nonumber\\
    & \leq \frac{1}{n} \sum_{k=1}^n \varphi_{q,k}^2 \nonumber\\
    & \leq  \frac{\varphi_{\max}}{n} \sum_{k=1}^n \varphi_{q,k} = \frac{\varphi_{\max}}{n}
\end{align}
A direct consequence is that
\begin{align}
\label{eq: variance of the sum}
    \sigma^2 = \VV\left[\Phi_q - \frac{\nsamples}{n}\right] &=\nsamples \VV[X_j] \leq \nsamples \frac{\varphi_{\max}}{n} 
\end{align}
Therefore, with $K = \varphi_{\max} \leq c_0 n^{-\frac{\gamma}{2}}$ and $\sigma^2 \leq \nsamples \frac{\varphi_{\max}}{n} $, the Bernstein inequality reads
\begin{equation}
    P\left( \vert \Phi_q - \frac{\nsamples}{n}\vert \geq t\right) \leq 2 \exp\left( \frac{-t^2}{\frac{2\nsamples\varphi_{\max}}{n} + \frac{2\varphi_{\max} t }{3}} \right)
\end{equation}
In particular, for $t = \frac{\nsamples}{2n}$:
\begin{align}
    \label{eq: applying bernstein for phi}
    P\left( \vert \Phi_q - \frac{\nsamples}{n}\vert \geq \frac{\nsamples}{2n}\right) & \leq 2 \exp\left( -\frac{1}{8\varphi_{\max} + \frac{16}{3} \varphi_{\max}} \frac{\nsamples}{n} \right) \nonumber\\
    & = 2 \exp\left( -\frac{3}{40} \frac{\nsamples}{n \varphi_{\max}} \right) \nonumber\\
    & \leq 2 \exp\left( -\frac{3}{40 c_0}\nsamples n^{\frac{\gamma}{2} -1} \right)
\end{align}

So provided that $\nsamples$ grows faster than $O(n^{1-\frac{\gamma}{2}})$, this probability tends to zero and $\Phi_q$ concentrates around its expected value. We found that
\begin{equation}
    P\left(\Phi_q \leq \frac{\nsamples}{2n}\right) \leq P\left( \vert \Phi_q - \frac{\nsamples}{n}\vert \geq \frac{\nsamples}{2n}\right) \leq 2 \exp\left( -\frac{3}{40 c_0}\nsamples n^{\frac{\gamma}{2} -1} \right)
\end{equation}

\newpage
\textbf{Proof of Lemma \ref{lemma: concentration of phi_eps}: Concentration of $\Phi_q^{B_\varepsilon}$}

Here, we apply the same probabilistic argument based on the Bernstein inequality to derive a similar concentration result for $\Phi_q^{B_\varepsilon}$. First note that
\begin{align}
    \EE\left[\Phi_q^{B_\varepsilon} \right] &= \EE\left[\sum_{j=1}^{\nsamples} \1_{\{\|x_{i_j}-x_q\| \leq \varepsilon\}} \varphi_{q, i_j} \right]  = \EE\left[\sum_{k=1}^n  \sum_{j=1}^{\nsamples} \1_{\{x_k = x_{i_j}\}}  \1_{\{\|x_{i_j}-x_q\| \leq \varepsilon\}} \varphi_{q,i_j}\right] \nonumber\\
    & = \EE\left[\sum_{k=1}^n \varphi_{q,k} \1_{\{\|x_{k}-x_q\| \leq \varepsilon\}} \sum_{j=1}^{\nsamples} \1_{\{x_k = x_{i_j}\}}\right] \nonumber\\
    & = \sum_{k=1}^n \varphi_{q,k} \1_{\{\|x_{k}-x_q\| \leq \varepsilon\}} \EE\left[ \sum_{j=1}^{\nsamples} \1_{\{x_k = x_{i_j}\}}\right] \nonumber\\
    & = \frac{\nsamples}{n} \sum_{k=1}^n \varphi_{q,k} \1_{\{\|x_{k}-x_q\| \leq \varepsilon\}} 
\end{align}

Recall that by equation \eqref{eq:small_weight_outside_condition}, for $n$ sufficiently large, we have that with high probability
\begin{align}
\label{eq: expectation of phi_eps}
     \EE\left[\Phi_q^{B_\varepsilon} \right] & = \frac{\nsamples}{n} \sum_{k=1}^n \varphi_{q,k} \1_{\{\|x_{k}-x_q\| \leq \varepsilon\}} \nonumber\\
     & = \frac{\nsamples}{n}\left(1 - \sum_{k=1}^n \varphi_{q,k} \1_{x_j\notin B_\varepsilon(x_q)}\right) \nonumber\\
     &\geq \frac{\nsamples}{n} (1-\varepsilon)
\end{align}

Again, we can rewrite $\Phi_q^{B_\varepsilon} - \EE[\Phi_q^{B_\varepsilon}]$ as a sum of centered random variables that are bounded by $\varphi_{\max}$:
\begin{align}
    \Phi_q^{B_\varepsilon} - \EE[\Phi_q^{B_\varepsilon}] = \sum_{j=1}^{\nsamples} \1_{\{\|x_{i_j}-x_q\| \leq \varepsilon\}} (\varphi_{q, i_j} - \frac{1}{n})
\end{align}
In particular, we obtain
\begin{align}
    \sigma^2 = \VV\left[\Phi_q^{B_\varepsilon} - \EE[\Phi_q^{B_\varepsilon}]\right] \leq \frac{\varphi_{\max}}{n} \sum_{j=1}^{\nsamples} \1_{\{\|x_{i_j}-x_q\| \leq \varepsilon\}}  \leq \frac{\varphi_{\max}\nsamples}{n}
\end{align}
So, the same bounds for $\varphi_{\max}$ and $\sigma^2$ can be used in the Bernstein inequality and we infer similarly to \eqref{eq: applying bernstein for phi} that the following holds:
\begin{align}
\label{eq: Bernstein for Phi_eps}
    P\left(\Phi_q^{B_\varepsilon} \leq \frac{\nsamples}{2n} (1 - \varepsilon) \right) &\leq P\left(\Phi_q^{B_\varepsilon} \leq \frac{1}{2}  \EE[\Phi_q^{B_\varepsilon}] \right) \nonumber\\ 
    &\leq P\left( \vert \Phi_q^{B_\varepsilon} - \EE[\Phi_q^{B_\varepsilon}]| \geq \frac{1}{2}  \EE[\Phi_q^{B_\varepsilon}]\right)  \nonumber\\
    &\leq P\left( \vert \Phi_q^{B_\varepsilon} - \EE[\Phi_q^{B_\varepsilon}]| \geq \frac{\nsamples}{2n} (1-\varepsilon)\right) \nonumber\\
    &\leq 2 \exp\left( -\frac{3 (1-\varepsilon)^2}{40 c_0}\nsamples n^{\frac{\gamma}{2} -1} \right) 
\end{align}
In the last step, we applied Bernstein's inequality with $t = \frac{\nsamples}{2n} (1-\varepsilon)$, which reads:
\begin{align}
    P\left( \vert \Phi_q^{B_\varepsilon} - \EE[\Phi_q^{B_\varepsilon}]| \geq \frac{\nsamples}{2n} (1-\varepsilon)\right) & \leq 2 \exp\left( -\frac{ (1-\varepsilon)^2}{8\varphi_{\max} + \frac{16}{3} (1-\varepsilon) \varphi_{\max}} \frac{\nsamples}{n} \right)\nonumber\\
    & \leq 2 \exp\left( -\frac{ (1-\varepsilon)^2}{8\varphi_{\max} + \frac{16}{3}  \varphi_{\max}} \frac{\nsamples}{n} \right)\nonumber\\
    & = 2 \exp\left( -\frac{3(1-\varepsilon)^2}{40} \frac{\nsamples}{n \varphi_{\max}} \right) \nonumber\\
    & \leq 2 \exp\left( -\frac{3(1-\varepsilon)^2}{40 c_0}\nsamples n^{\frac{\gamma}{2} -1} \right)
\end{align}

\textbf{Proof of Lemma \ref{lemma: concentration of remainder}: Concentration of $\Phi_q - \Phi_q^{B_\varepsilon}$}

Using the linearity of the expectation and the already computed expected values in equations \eqref{eq: expectation of phi} and \eqref{eq: expectation of phi_eps}, we have w.h.p.\ that
\begin{equation}
    \EE[\Phi_q - \Phi_q^{B_\varepsilon}] = \frac{\nsamples}{n} - \EE[\Phi_q^{B_\varepsilon}] \leq \varepsilon\frac{\nsamples}{n}
\end{equation}
As before, we rewrite $\Phi_q - \Phi_q^{B_\varepsilon} - \EE[\Phi_q - \Phi_q^{B_\varepsilon}]$ as a sum of centered random variables:
\begin{align}
    \Phi_q - \Phi_q^{B_\varepsilon} - \EE[\Phi_q - \Phi_q^{B_\varepsilon}] = \sum_{j=1}^{\nsamples} \1_{\{\|x_{i_j}-x_q\| > \varepsilon\}} (\varphi_{q, i_j} - \frac{1}{n})
\end{align}
In this case, the random variables are even bounded by $\varepsilon\varphi_{\max}$ due to the path length argument involved. More precisely, with the kernel bandwidth being $\varepsilon / l_n$, we know that it takes a path of at least length $l_n$ to connect $x_q$ with a node outside of $B_\varepsilon(x_q)$. Thus, the corresponding entries of $S^i$ are zero for all $i < l_n$. For $x_j \notin B_\varepsilon(x_q)$, we have
\begin{align}
    \varphi_{q,j} &=  (1-\alpha_n)\sum_{i=1}^\infty \alpha_n^i \|(D^{-1} W)^i e_j\| \nonumber \\
    &=  (1-\alpha_n)\sum_{i=l_n}^\infty \alpha_n^i \|(D^{-1} W)^i e_j\| \nonumber \\
    & \leq \varphi_{\max} (1-\alpha_n)\sum_{i=l_n}^\infty \alpha_n^i \nonumber \\
    &=  \varphi_{\max} \alpha_n^{l_n}\nonumber \\
    &<  \varphi_{\max} \varepsilon
\end{align}

Thus, we have with high probability that $K = \varepsilon \varphi_{\max}$ serves as a bound for the Bernstein inequality, and we can bound the variance of the sum as in \eqref{eq: variance of the sum for phi} by
\begin{align}
    \sigma^2 = \VV[ \Phi_q - \Phi_q^{B_\varepsilon} - \EE[\Phi_q - \Phi_q^{B_\varepsilon}] ] \leq \frac{\varepsilon\varphi_{\max}\nsamples}{n}
\end{align}

Applying Bernstein's inequality gives
\begin{align}
    P(\Phi_q - \Phi_q^{B_\varepsilon} \geq 2 \varepsilon \frac{\nsamples}{n}) &\leq P(\Phi_q - \Phi_q^{B_\varepsilon} \geq 2 \EE[\Phi_q - \Phi_q^{B_\varepsilon}]) \nonumber \\
    & \stackrel{\text{Bernstein}}{\leq} 2 \exp\left(- \frac{2\varepsilon^2 }{\varepsilon\varphi_{\max} + \frac{2}{3}\varepsilon^2\varphi_{\max}} \frac{\nsamples}{n} \right) \nonumber \\
    & = 2 \exp\left(- \frac{2\varepsilon}{1 + \frac{2}{3}\varepsilon} \frac{\nsamples}{n\varphi_{\max}} \right) \nonumber \\
    & \stackrel{\text{Weight bound}}{\leq}  2 \exp\left(- \frac{2\varepsilon}{(1 + \frac{2}{3}\varepsilon)c_0} \nsamples n^{\frac{\gamma}{2}-1} \right)
\end{align}


\newpage
\textbf{Confidence Intervals}

Here, we derive confidence intervals for the estimate in a single $x_q$. Since we require them to hold only in a single point and not for all point simultaneously, we do not have to apply a union bound argument or insert other uniform bounds as done in the proof. Further, since we are interested in the uncertainty of the estimate given the data, i.e., we consider a fixed dataset in practice, we are dealing with fixed kernel weights here. 

We use the bias variance decomposition and apply the considerations made in the proof to derive confidence intervals for the individual components of the point estimate $\hat{p}_q \in \Delta_C$. 
\begin{equation*}
       \|\hat{p}_q - p_q\|  \leq \left\| \hat{p}_q  - \EE_{Y|X}\left[\hat{p}_q \right] \right\| + \left\| \EE_{Y|X}\left[\hat{p}_q \right] - p_q \right\| 
\end{equation*}
For a single component $(\hat{p}_q)^c$, the bias variance decomposition reads as
\begin{equation}
\label{eq: error decomposition single component}
       \vert(\hat{p}_q)^c - (p_q)^c\vert  \leq \left\vert (\hat{p}_q)^c  - \EE_{Y|X}\left[(\hat{p}_q)^c \right] \right\vert + \left\vert \EE_{Y|X}\left[(\hat{p}_q)^c \right] - (p_q)^c \right\vert 
\end{equation}

Recall that the bias of the estimator, can be bounded through the exploitation of the Lipschitz continuity. More precisely, by considering the distances of the individual nodes, we derive:
\begin{align}
    \label{eq: bias bound for confidence intervals}
    \left\vert \EE_{Y|X}\left[(\hat{p}_q)^c \right] - (p_q)^c \right\vert &= 
    \left\vert \sum_{j=1}^{\nsamples}  \frac{\varphi_{q,i_j}}{\Phi_q} \left( (p_{i_j})^c - (p_q)^c\right) \right\vert \nonumber \\
    &\leq \sum_{j=1}^{\nsamples} \frac{\varphi_{q,i_j}}{\Phi_q} \left\vert (p_{i_j})^c - (p_q)^c \right\vert \nonumber \\
    &\leq \sum_{j=1}^{\nsamples}\frac{\varphi_{q,i_j}}{\Phi_q} \min\{\,1,\,  L_y \| x_{i_j} - x_q\|\,\} =:\mathcal{B}
\end{align}

Since the graph structure is considered as fixed here, the kernel weights are not random variables that depend on the graph structure anymore but simply scalars. In particular, we can simply compute $\mathcal{B}$ defined in \eqref{eq: bias bound for confidence intervals} assuming that we have knowledge of the Lipschitz-constant $L_y$. The first summand in \eqref{eq: error decomposition single component} represents the variance associated with the sampling of feedbacks. As in the proof, we apply Hoeffding's inequality to obtain
\begin{equation}
\label{eq: Hoeffding for conf int}
    P\left( \vert (\hat{p}_q)^c  - \EE_{Y|X}\left[(\hat{p}_q)^c \right]\vert > \varepsilon \right) \leq 2 \exp\left(-\frac{2\varepsilon^2}{\sum_{j=1}^{\nsamples}  \left(\frac{\varphi_{q,i_j}}{\Phi_q}\right)^2}\right)
\end{equation}

Inserting the bound $\mathcal{B}$ for the bias \eqref{eq: bias bound for confidence intervals} and Hoeffding's inequality \eqref{eq: Hoeffding for conf int} into the bias variance decomposition \eqref{eq: error decomposition single component}, we obtain the following confidence interval for a single component estimate $(\hat{p}_q)^c$:
\begin{align}
\label{eq: Confidence Interval}
    &P\left(\vert (\hat{p}_q)^c  - (p_q)^c \vert > \varepsilon+ \mathcal{B}  \right) \nonumber \\
    &\leq 2 \exp\left(-\frac{2\varepsilon^2}{\sum_{j=1}^{\nsamples}  \left(\frac{\varphi_{q,i_j}}{\Phi_q}\right)^2}\right)
\end{align}

For any given confidence level $\delta\in(0,1)$, we can use Hoeffding's inequality to determine the smallest $\varepsilon>0$ possible such that the Hoeffding bound yields a bound for the failure probability smaller than $\delta$. This $\varepsilon$ then reflects the variance part of the error and, together with the bias bound $\mathcal{B}$, leads to the confidence interval bounds in \eqref{eq: Confidence Interval}.

We can construct confidence intervals for all components that hold simultaneously by introduction of the union bound:
\begin{align}
    &P\left(\exists\, c \in\{1,\ldots, C\}: \vert (\hat{p}_q)^c  - (p_q)^c \vert > \varepsilon + \mathcal{B} \right) \nonumber \\
    &\leq 2 C\exp\left(-\frac{2\varepsilon^2}{\sum_{j=1}^{\nsamples} \left(\frac{\varphi_{q,i_j}}{\Phi_q}\right)^2}\right)
\end{align}

\end{document}